%% file: iclr2025_conference.tex
\documentclass{article} % For LaTeX2e
\usepackage{iclr2025_conference,times}

\input{math_commands.tex}

\usepackage{pifont}
\usepackage{hyperref}
\usepackage{cleveref}
\usepackage{amssymb}
\usepackage{url}
\usepackage{graphicx} 
\usepackage{fontawesome5}
\usepackage{xspace}
\usepackage{wrapfig}
\usepackage{amsthm}
\usepackage{xcolor}
\usepackage[table]{xcolor}
\definecolor{darkgreen}{RGB}{0,128,0}
\definecolor{darkred}{RGB}{139,0,0}
\definecolor{mygray}{gray}{0.35}
\definecolor{lightgray}{gray}{0.9}

\newtheorem{remark}{Remark}

\newtheorem{theorem}{Theorem}

\usepackage{multirow}
\usepackage{graphicx}
\usepackage[normalem]{ulem}
\usepackage{booktabs} 
\Crefname{section}{section}{sections}
\Crefname{section}{Section}{Sections}
\newcommand{\ie}{\textit{i.e.}\xspace}
\newcommand{\eg}{e.g.\xspace}
\newcommand{\iconfire}{\raisebox{-0.2ex}{\includegraphics[height=1.7ex]{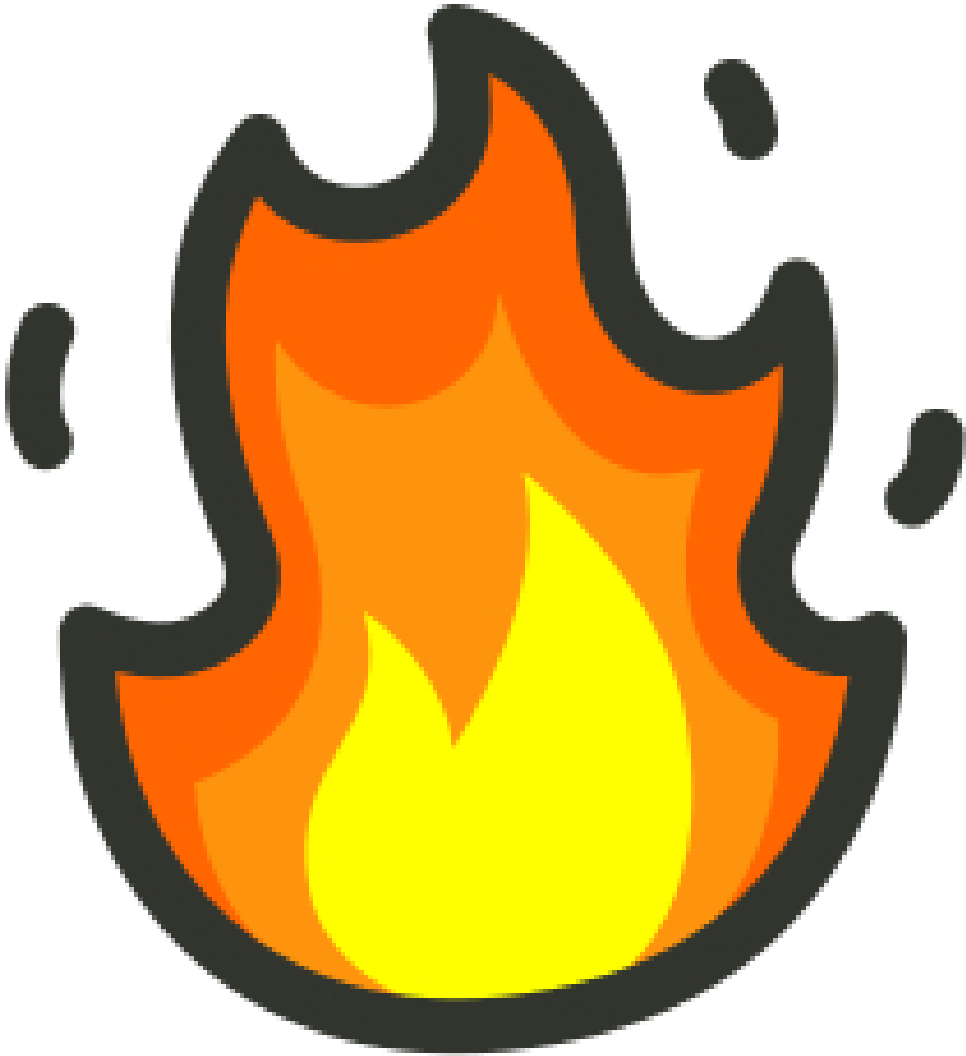}}}
\newcommand{\iconsnow}{\raisebox{-0.28ex}{\includegraphics[height=1.7ex]{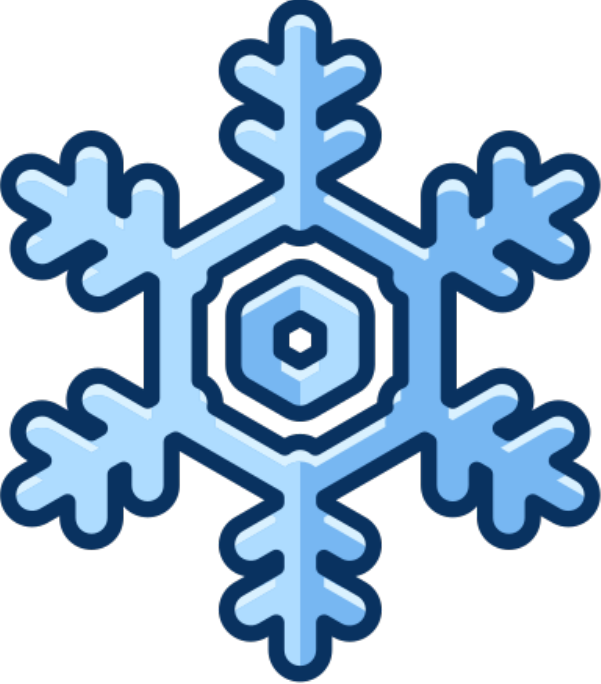}}}
\title{ASMIL: Attention-Stabilized Multiple Instance Learning for Whole Slide Imaging}

\author{%
\textbf{Linfeng Ye}\textsuperscript{1},
Shayan Mohajer Hamidi\textsuperscript{2},
Zhixiang Chi\textsuperscript{1},
Guang Li\textsuperscript{3},\\
\textbf{Mert Pilanci\textsuperscript{2},
Takahiro Ogawa\textsuperscript{3},
Miki Haseyama\textsuperscript{3},
Konstantinos N.~Plataniotis\textsuperscript{1}} \\
\textsuperscript{1}University of Toronto, \textsuperscript{2}Stanford University, \textsuperscript{3}Hokkaido University\\
\textsuperscript{1}\texttt{\{linfeng.ye,zhixiang.chi\}@mail.utoronto.ca}\\
\textsuperscript{1}\texttt{kostas@ece.utoronto.ca}\\
\textsuperscript{2}\texttt{\{smohajer,pilanci\}@stanford.edu} \\
\textsuperscript{3}\texttt{\{guang,ogawa,mhaseyama\}@lmd.ist.hokudai.ac.jp} 
}
\iclrfinalcopy % Uncomment for camera-ready version, but NOT for submission.
\begin{document}

\maketitle

\begin{abstract}

Attention-based multiple instance learning (MIL) has emerged as a powerful framework for whole slide image (WSI) diagnosis,  leveraging attention to aggregate instance-level features into bag-level predictions. Despite this success, we find that such methods exhibit a new failure mode: unstable attention dynamics.  Across four representative attention-based MIL methods and two public WSI datasets, we observe that attention distributions oscillate across epochs rather than converging to a consistent pattern, degrading performance. This instability adds to two previously reported challenges: overfitting and over-concentrated attention distribution. To simultaneously overcome these three limitations, we introduce attention-stabilized multiple instance learning (ASMIL), a novel unified framework. ASMIL uses an anchor model to stabilize attention, replaces softmax with a normalized sigmoid function in the anchor to prevent over-concentration, and applies token random dropping to mitigate overfitting. Extensive experiments demonstrate that ASMIL achieves up to a 6.49\% F1 score improvement over state-of-the-art methods. Moreover, integrating the anchor model and normalized sigmoid into existing attention-based MIL methods consistently boosts their performance, with F1 score gains up to 10.73\%. All code and data are publicly available at \url{https://github.com/Linfeng-Ye/ASMIL}.
\end{abstract}

\section{Introduction} \label{sec:intro}
Computational pathology, at the intersection of digital imaging, machine learning, and clinical diagnostics, has transformed modern workflows \citep{verghese2023computational}. Advances in whole slide imaging (WSI) now allow glass slides to be digitized into gigapixel images \citep{evans2001method}, which are central to cancer diagnosis and treatment planning. WSIs preserve rich spatial context and enable large-scale sharing, but their extreme size and sparsity create major challenges: diagnostically relevant regions often occupy only a tiny fraction of the slide, and exhaustive pixel- or tile-level annotations are infeasible in practice. As a result, most datasets provide only weak slide-level labels, making it critical to design methods that learn effectively under weak supervision.

This weakly supervised setting naturally motivates multiple instance learning (MIL) \citep{NIPS1990_e46de7e1,10.1016/S0004-3702(96)00034-3,10.5555/302528.302753}. In MIL, a bag of instances is mapped to a single bag-level label. For WSIs, the image is divided into tiles, each treated as an instance, while only the slide-level label is required. This dramatically reduces annotation costs and makes large-scale WSI datasets more practical for research and clinical use.

Early approaches to MIL-based WSI analysis focused on simple aggregation strategies, such as clustering instance features \citep{xu2014weakly} or applying global pooling layers \citep{kraus2016classifying}. A major breakthrough came with the introduction of attention-based MIL (ABMIL) \citep{ilse2018attention}, which provided theoretical guidance for neural network-based MIL algorithms and introduced a permutation-invariant attention mechanism to aggregate instance information into bag-level representations. ABMIL established a strong baseline for WSI analysis \citep{shao2025do} and, importantly, enhanced interpretability through visualized attention scores, which is an essential property for clinical adoption. Building on this foundation, subsequent works have refined ABMIL to further improve performance, scalability, and robustness \citep{xiong2021nystromformer, shao2021transmil, zhang2022dtfd, Tang_2023_ICCV, zhang2024attention}. In particular, TransMIL replaces independent instance weighting with a transformer encoder that explicitly models inter-instance relations within a bag \citep{shao2021transmil}. As a result, attention-based MIL has become the de facto choice for WSI subtyping not only because it aggregates instance features but also because its attention maps are used as clinical evidence of model interpretability.

Despite its success, attention-based MIL still suffers from three major problems, which we denote as \textbf{(PI)}, \textbf{(PII)}, and \textbf{(PIII)}, and elaborate on in the sequel.

A critical yet underexplored aspect of MIL-based WSI analysis is the convergence behavior of attention mechanisms during training. The gigapixel scale of WSIs, coupled with weak supervision, high variability, and sparsity, makes it difficult for models to consistently identify informative tiles among thousands of candidates. Our investigation reveals that existing MIL algorithms often fail to converge stably on WSI datasets. To the best of our knowledge, we are the first to identify and systematically analyze \textbf{(PI)} \emph{unstable attention dynamics}, where attention distributions for individual WSIs oscillate substantially across epochs instead of converging into consistent patterns. To quantify this phenomenon, we measure the Jensen-Shannon divergence \citep{cover1999elements} between consecutive attention distributions of the same WSI, as illustrated for TransMIL \citep{shao2021transmil} in \Cref{fig:ABMIL_fluctuation}. Additional experiments across methods and datasets are provided in \Cref{Appendix:AttentionDynamics}. This persistent oscillation results in unstable training and degraded performance, reflected in higher cross-entropy values compared to our proposed method.

Beyond this new limitation identified in our study, prior work has highlighted two additional challenges. One is \textbf{(PII)} \emph{over-concentrated attention distribution} \citep{zhang2024attention,lu2021data}, where models allocate excessive importance to only a few tiles, thereby harming generalization and interpretability. The other is \textbf{(PIII)} \emph{overfitting} \citep{zhang2022dtfd,lin2023interventional}, a common issue in histopathology WSI classification caused by the limited number of available training samples.

In this paper, we aim to simultaneously address the challenges \textbf{(PI)}–\textbf{(PIII)}. To stabilize attention distribution and the training process, we introduce an \emph{anchor model}, which has the same architecture as the online model's attention module and receives the same input, but is updated via an exponential moving average (EMA) instead of by backpropagation. Acting as a stable reference, the anchor provides smoother and more consistent attention distributions. To transfer this stability, we encourage the online model to mimic the anchor by minimizing the Kullback–Leibler (KL) divergence between their attention distributions. To mitigate over-concentration, which we attribute to the exponential sensitivity of the softmax function, we replace softmax in the anchor branch with a normalized sigmoid function (NSF), as defined in \Cref{Eq:NSF}. Finally, we propose a simple yet effective token dropout strategy that regularizes the model and reduces overfitting. Together with the anchor model, these components form a unified framework called attention-stabilized multiple instance learning (ASMIL), which improves both the stability and generalization of MIL-based WSI analysis.

In summary, this paper’s contributions are as follows:

\noindent $\bullet$ We are the first to identify and systematically analyze the problem of \textit{unstable attention dynamics} in attention-based MIL for WSI analysis. This overlooked issue not only limits predictive performance but also undermines interpretability, since fluctuating attention distributions prevent consistent identification of the tissue regions that drive the model’s decisions.

\noindent $\bullet$ To overcome this instability, we introduce an anchor model that stabilizes attention distribution throughout training. The anchor model is updated using an exponential moving average of the online model, which ensures stable training dynamics and improves both performance and interpretability.

\noindent $\bullet$ We show \textit{mathematically} that replacing softmax with an NSF alleviates attention over-concentration. Since applying the NSF to the online model causes vanishing gradients, we apply it to the anchor model instead, ensuring stable and well-distributed attention.

\noindent $\bullet$ To mitigate overfitting, we introduce token dropout, which randomly discards a portion of feature tokens during training while retaining all tokens during inference.

\noindent $\bullet$ By integrating these innovations, we present attention-stabilized MIL (ASMIL), a novel MIL-based WSI analysis algorithm. Through comprehensive experiments on multiple public WSI datasets, we demonstrate that ASMIL achieves state-of-the-art performance in subtyping and localization tasks.

\textit{Paper Organization.} The remainder of this paper is structured as follows: \Cref{Sec:relatedWork} reviews related work on MIL and attention mechanisms in WSI analysis; \Cref{Sec:NotationandPreliminaries} presents the preliminaries and motivation of our approach; \Cref{section:METHODOLOGY} details the ASMIL framework; \Cref{Sec:Experiments} presents the experimental setup and results; and finally \Cref{Sec:Conclusion} concludes the paper with future research directions.

\begin{figure}[ht]
\begin{center}
%\framebox[4.0in]{$\;$}
\includegraphics[width=\linewidth]{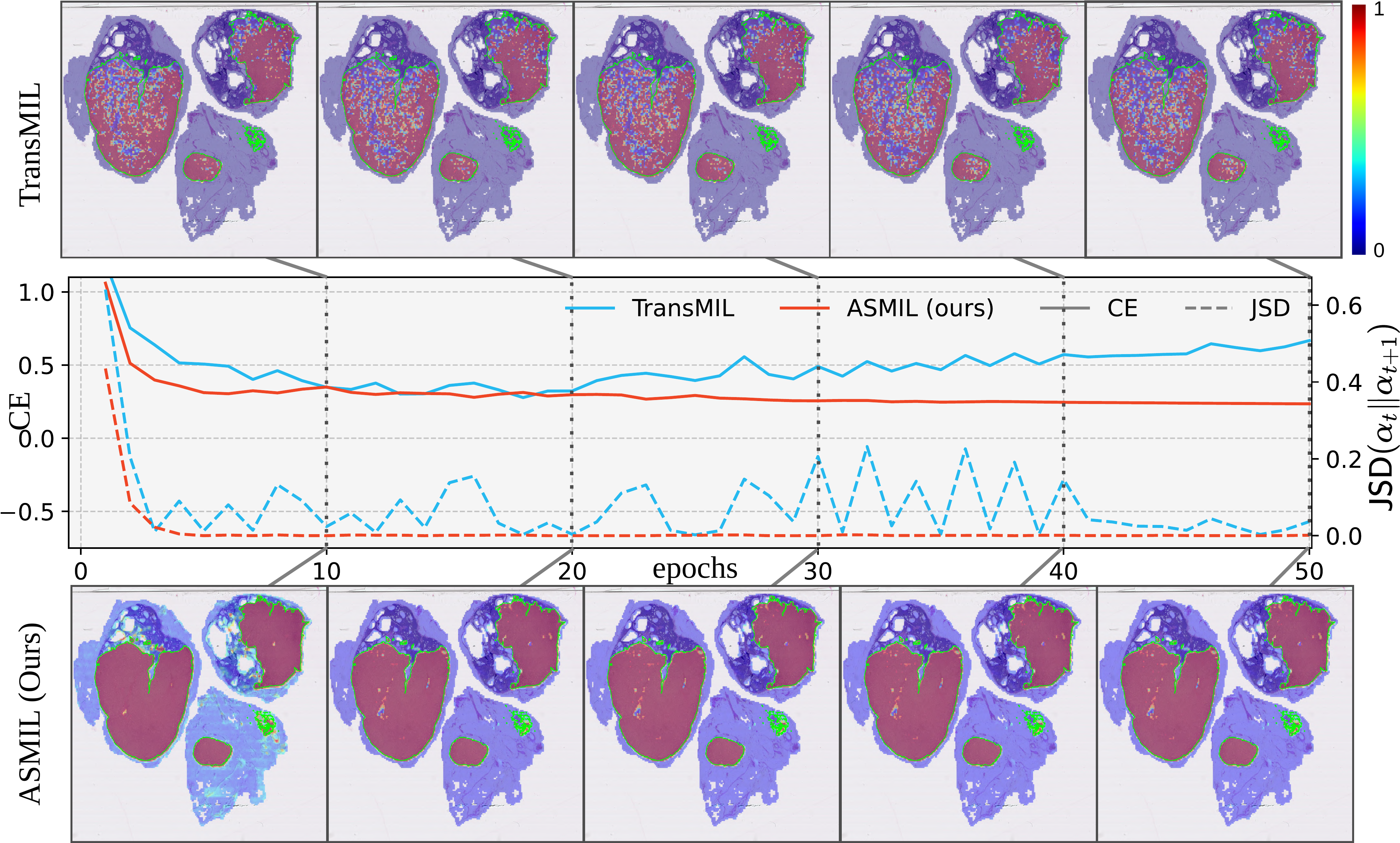}
\end{center}
\vspace{-0.4cm}
\caption{ Visualization of attention dynamics on a tumor WSI for TransMIL \citep{shao2021transmil} vs. ASMIL (our method). The green contours in the figures indicate the annotated tumor regions.
\textbf{Top}: TransMIL attention distribution at selected training iterations.
\textbf{Middle}: Jensen-Shannon divergence (JSD) between attention distributions at successive steps and the cross entropy loss (CE), comparing TransMIL (blue) and ASMIL (red).
\textbf{Bottom}: Attention distribution from ASMIL over different training iterations. Due to the weakly supervised nature of WSI subtyping datasets, TransMIL’s attention patterns never converge during training, further, it focuses on only a subset of cancerous regions. In contrast, our method ($i$) produces stable attention distributions throughout training and ($ii$) consistently highlights cancerous regions. }
\vspace{-0.5cm}
\label{fig:ABMIL_fluctuation}
\end{figure}
\section{Related Work}
\label{Sec:relatedWork}
Early weakly supervised approaches in computational pathology leveraged multi-view convolutional neural network ensembles and basic MIL pooling to transition from patch-level labels to slide-level predictions \citep{7950690,8363642}. As datasets scaled and slide-level supervision became the norm, methods shifted from fixed pooling to attention mechanisms that make aggregation learnable. Building on this trend, attention-based MIL \citep{ilse2018attention} introduced learnable instance weights and generated heatmaps from slide-level labels, achieving breast and colon cancer classification on par with fully supervised methods at scale. Complementary to weighting instances, subsequent work reduced morphological redundancy in tile representations, \citet{song2024morphological} used a Gaussian mixture model, and sped up inference by skipping irrelevant patches \citep{dong2025fast}. {\citet{li2021dual} propose DSMIL, a dual-stream MIL framework that selects a critical instance via max-pooling and then applies a trainable non-local, distance-based attention from this instance to all others to form bag embeddings for WSI classification. Subsequent works extend this line of research by leveraging multi-scale fusion to aggregate information across resolutions \citep{9522980, guo2023higt, 10849962, buzzard2024paths, li2019attention}.}

{Several works further refine training strategies for attention-based MIL. To prevent the attention distribution from collapsing onto a few input patches and to obtain more faithful attention maps, \citet{zhang2024attention} stochastically masks the top-$K$ instances, while \citet{zhang2025AEM} adds an entropy regularization term that explicitly flattens the attention distribution. In a complementary direction, \citet{fourkioti2024camil} introduces neighbor-constrained attention to suppress noise in the feature maps. Because WSI datasets usually contain only a few hundred training samples, many methods focus on mitigating overfitting, for example, by introducing bag splitting to create pseudo-bags \citep{zhang2022dtfd}, designing efficient instance-based classifiers \citep{10530149}, and performing hard-negative mining with EMA teachers \citep{Tang_2023_ICCV}. {\citet{lu2021data} introduce clustering-constrained attention multiple-instance learning (CLAM), which replaces max-pooling with class-specific attention pooling and adds instance-level clustering supervision so that weakly supervised slide-level MIL can be both data-efficient and interpretable on WSIs}, or using contrastive critical-instance branches \citep{9578683}. Recently, \citet{zhu2025how, zhu2023pdl} systematically studied the effect of random dropping in MIL and proposed to randomly remove the top-$K$ instances with the highest attention weights together with $G \times k$ similar tokens during training, which mitigates overfitting and encourages convergence to flatter regions of the loss landscape, thereby improving generalization. Since our anchor leverages an EMA update, we relate it to EMA/teacher models and provide additional details in \Cref{Appendix:EMAliter}.}

% {In contrast to these prior works, we trace over-concentrated attention distributions to the exponential nature of the softmax function and identify a new failure mode, which we term unstable attention dynamics, that has not been explicitly addressed in the existing MIL literature. We propose a principled solution by introducing an EMA-updated anchor model to stabilize the attention dynamics, and we further mitigate overfitting through token random dropping.}

% In contrast to the prior works, we trace \textbf{(PII)} to softmax and address it with an NSF in the anchor, while \textbf{(PIII)} is mitigated by random token dropping as a regularizer.Most importantly, \textbf{(PI)} \emph{unstable attention dynamics} has not been previously identified or addressed in the literature. We are the first to diagnose this phenomenon and propose a principled solution. 

\section{Preliminaries and Motivation}
\label{Sec:NotationandPreliminaries}
\subsection{Notation}
\label{section:Notation}
Scalars are denoted by non-bold letters (\eg, $a, \beta$), vectors by bold lowercase letters (\eg, $\boldsymbol{a}$), and matrices by bold uppercase letters (\eg, $\boldsymbol{A}$). The $i$-th entry of a vector $\boldsymbol{a}$ is written as $\boldsymbol{a}_i$. A $C$-dimensional probability simplex is denoted by $\Delta^C$. For two distributions $P_1, P_2 \in \Delta^C$, the Kullback–Leibler divergence (KL divergence) is defined as $\mathsf{KL} (P_1 \| P_2 ) = \sum_{c=1}^C P_1 [c]\log\!{\frac{P_1 [c]}{P_2 [c]} }$.

% The attention mechanism in the transformer is based on the scaled dot-product attention. Given query matrix $Q$, key matrix $K$, and value matrix $V$, the attention output is computed as: $Attention(Q,K,V) = \psi(\frac{QK^T}{\sqrt{d_k}})V$, where the $Q, K, V\in \mathbb{R}^{n\times d_k}$ are the query, key and value matrices, respectively. We refer $\frac{QK^T}{\sqrt{d_k}}$ as the attention score, while $\psi(\frac{QK^T}{\sqrt{d_k}})$ is the attention distribution. 

% \subsection{Transformer Attention Mechanism}
% A Transformer computes pairwise attention scores between queries $Q$, keys $K$, and values $V$. The attention output is $
% \text{Attention}(Q,K,V) = \psi(S)V, \quad S = \tfrac{QK^\top}{\sqrt{d_k}}$, where $S$ denotes the attention scores and $\psi$ is the softmax that normalizes them into attention distribution, representing the relative importance of each token. During training, learnable tokens are appended to the input and act as queries. Updated by gradient descent, these tokens learn to aggregate globally relevant features across layers, yielding compact, task-specific summaries of the input.

\subsection{Multiple Instance Learning with Attention}
\label{sec:mil_attention}

In MIL, supervision is provided only at the bag level. A slide is represented as a bag
\(X=\{\boldsymbol{x}_i\}_{i=1}^{N}\) with unknown instance labels. After a pretrained encoder, we obtain instance embeddings
\(\{\boldsymbol{h}_i\}_{i=1}^{N}\).

Attention-based MIL assigns a scalar \emph{attention score} to each embedding via a learnable scorer \(f_{\boldsymbol{\theta}}\):
\begin{align}
z_i \;=\; f_{\boldsymbol{\theta}}(\boldsymbol{h}_i), \qquad \boldsymbol{z}=(z_1,\dots,z_N)\in\mathbb{R}^{N}.
\end{align}
Scores are normalized into an \emph{attention distribution} on the probability simplex \(\Delta^{N}\) using a softmax:
\begin{align}
\alpha_i \;=\; \frac{\exp(z_i)}{\sum_{j=1}^{N}\exp(z_j)}, \qquad
\sum_{i=1}^{N}\alpha_i = 1, \qquad \boldsymbol{\alpha}=(\alpha_1,\dots,\alpha_N)\in\Delta^{N}.
\end{align}
The slide-level representation, $\boldsymbol{h}_{\mathrm{bag}} \;=\; \sum_{i=1}^{N}\alpha_i\,\boldsymbol{h}_i$, is a convex combination of instance features weighted by the attention distribution and is passed to a classifier to produce the bag-level prediction.
\subsection{Motivation}
\label{section:Motivation}

MIL is effective for WSI analysis, but its weak supervision and small WSI dataset sizes introduce three failure modes: unstable attention dynamics, over-concentrated attention, and overfitting. 

\noindent $\bullet$ \textbf{(PI) \textit{Unstable attention dynamics}.} {Under bag-level supervision,} we empirically observe that attention distribution oscillates across epochs rather than converging to a consistent pattern. {To the best of our knowledge, this phenomenon has not been previously identified or explicitly addressed in the literature.} To quantify stability, we measure the Jensen-Shannon divergence (JSD) between consecutive attention distributions for the same WSI. Let $\boldsymbol{\alpha}_t\in\Delta^N$ denote the attention over $N$ tiles at epoch $t$. With $\mathsf{KL}(\cdot \| \cdot)$ denoting the KL divergence and $\bar{\boldsymbol{\alpha}}=\tfrac12(\boldsymbol{\alpha}_t+\boldsymbol{\alpha}_{t+1})$, {we define}
\begin{align}
\mathsf{JSD}(\boldsymbol{\alpha}_t\|\boldsymbol{\alpha}_{t+1})=\tfrac12 \mathsf{KL}(\boldsymbol{\alpha}_t\|\bar{\boldsymbol{\alpha}})+\tfrac12 \mathsf{KL}(\boldsymbol{\alpha}_{t+1}\|\bar{\boldsymbol{\alpha}}).
\end{align}
As shown in \Cref{fig:ABMIL_fluctuation}, TransMIL \citep{shao2021transmil} exhibits large JSD fluctuations, indicating a lack of stable convergence. Similar behavior appears in other attention-based MIL models; additional results are provided in \Cref{Appendix:AttentionDynamics}.

\noindent $\bullet$ \textbf{(PII) \textit{Over-concentration of attention}.} Complementary to instability,
prior works report that ABMIL often assigns most mass to a few tiles, which harms generalization and interpretability \citep{zhang2024attention, zhang2025AEM}. {Distinct from previous approaches, we attribute these over-concentrated attention distributions to the exponential nature of the softmax function.}

\noindent $\bullet$ \textbf{(PIII) \textit{Overfitting}.} 
WSI datasets typically contain only a few slides per class and highly redundant tiles \citep{zhang2022dtfd}. High-capacity neural-network-based MIL models can memorize spurious tile-level patterns, leading to poor out-of-distribution performance. {To alleviate this, we introduce a random token drop mechanism specialized for our method.}

In the next section, we present our proposed methodology, which simultaneously addresses the three problems \textbf{(PI)}, \textbf{(PII)}, and \textbf{(PIII)}.
\section{METHODOLOGY}
\label{section:METHODOLOGY}
To address the limitations of attention-based MIL, we propose a framework illustrated in \Cref{fig:MyMILStr}. Our methodology addresses \textbf{(PI)} by stabilizing attention through an anchor model, tackles \textbf{(PII)} by replacing softmax with an NSF in the anchor, and mitigates \textbf{(PIII)} by token random dropping to regularize training. The next subsections detail each component and the overall objective.

\subsection{Stabilizing Attention distributions via an Anchor Model}
\begin{wrapfigure}[33]{r}{0.62\textwidth}
\vskip -0.2in
\includegraphics[width=0.61\textwidth]{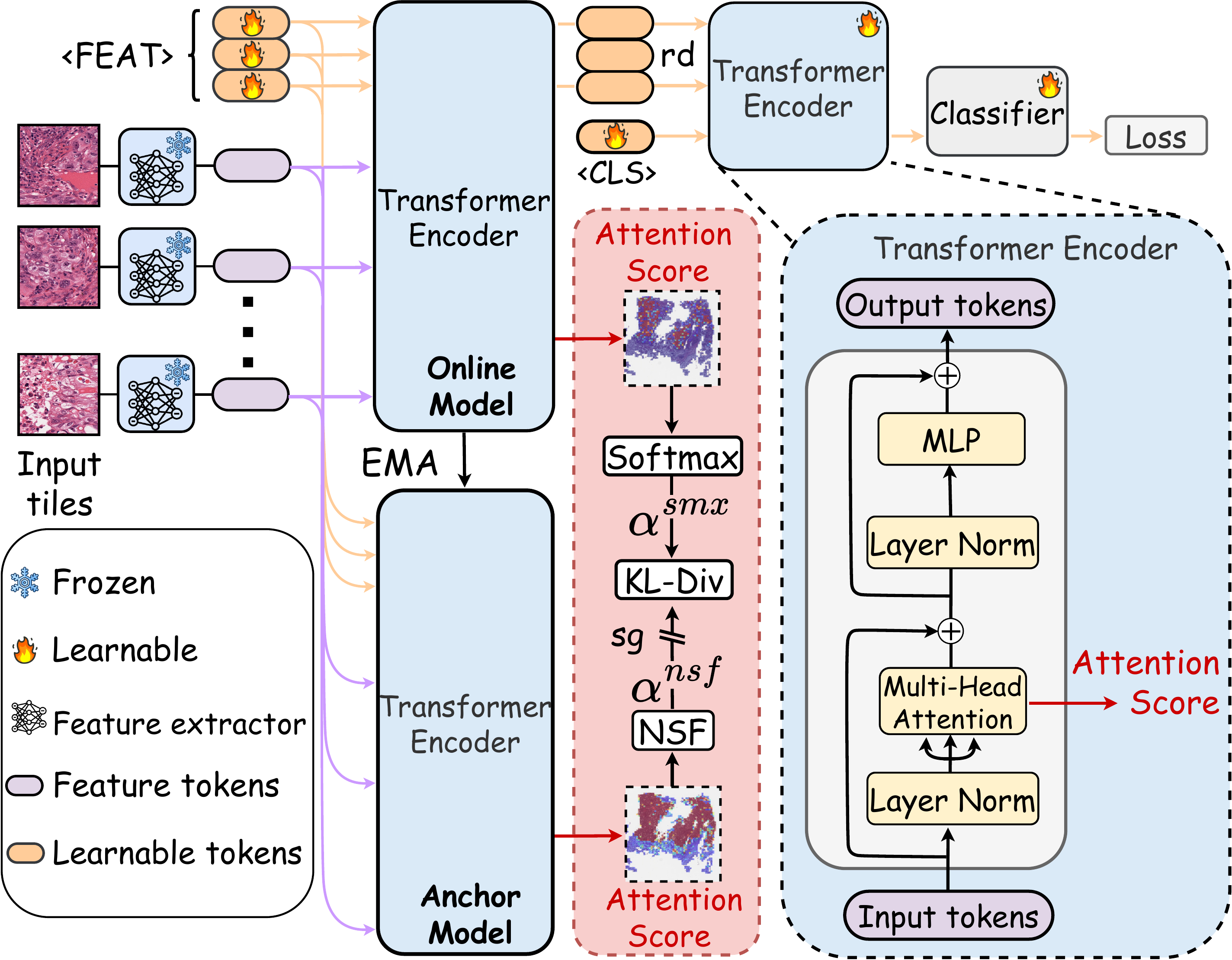}
\vskip -0.05in
\caption{Overview of ASMIL. Each WSI is divided into tiles and embedded into vision tokens using a pretrained encoder. These tokens, along with trainable FEAT tokens, feed into both online and anchor encoders. The anchor encoder’s attention scores over the FEAT tokens are transformed into a probability vector using an NSF, while the online encoder applies a softmax. To stabilize training and prevent the online model’s attention from becoming overly concentrated, we compute the KL divergence between the two distributions. Gradients are blocked to the anchor encoder using a stop-gradient (sg) operator, and its parameters are updated via EMA from the online encoder. During training, we randomly drop (rd) $N$ FEAT tokens, feed the remaining tokens into a second transformer with a trainable [CLS] token, and train a classifier on its output. \iconfire\ and \iconsnow\ indicate learnable and frozen components, respectively.} 
\label{fig:MyMILStr}
\end{wrapfigure}

As discussed in \Cref{section:Motivation}, weak supervision in MIL often leads to unstable attention distributions that fluctuate across epochs, preventing convergence. To mitigate this, we introduce an \textit{anchor model} that mirrors the attention block of the online model. The anchor serves as a stable reference by being updated through an EMA of the online model’s parameters. Specifically, at training step $t$, the anchor parameters $\boldsymbol{\theta}'_t$ are updated as
\begin{align} \label{eq:emaeq}
\boldsymbol{\theta}'_t \leftarrow m \boldsymbol{\theta}'_{t-1} + (1-m)\boldsymbol{\theta}_t,
\end{align}
where $\boldsymbol{\theta}_t$ are the online model’s parameters and $m \in [0,1)$ is the EMA factor. Both the anchor and online models receive the same inputs, but \textbf{\underline{only}} the online model is updated by backpropagation, the anchor is updated via EMA. The goal is to align the online attention distribution to the anchor distribution, which yields a stabilization loss. 

In \Cref{Sec:TSNE_bag_feature}, we show that standard attention-based MIL yields poorly separated bag-level feature clusters during training because attention distributions do not converge reliably. Introducing the anchor model stabilizes attention, improves convergence, and produces clearly separated bag-level clusters.
\begin{remark}
    \textbf{Why an anchor model instead of a single regularizer.} Scalar penalties on attention, such as entropy, $\ell_2$, or temperature, are content-agnostic and act only on the current batch. They cannot encode relational structure among instances. An EMA \emph{anchor model} yields a data-dependent attention distribution conditioned on the bag. Encouraging the online attention to stay close to this target performs functional regularization that captures inter-instance relations and stabilizes training, which a scalar regularizer cannot do.
\end{remark}
The anchor is discarded at inference, adding no extra FLOPs or latency. In the next subsection, we describe how we further improve the anchor’s attention using an NSF, which alleviates over-concentration before applying this stabilization loss.

\subsection{Preventing Attention Concentration with NSF in the Anchor Model}

% In conventional transformer architectures, the softmax function maps self-attention scores \(\boldsymbol{z}\in\mathbb{R}^{N}\) to a probability vector. However, softmax often produces over-concentrated attention, where few tokens dominate and gradients for the remaining tokens vanish. Temperature scaling is an incomplete remedy: small temperatures preserve concentration, while large temperatures flatten the distribution so aggressively that weak tokens receive undue weight. We therefore seek a mechanism that equalizes attention among genuinely informative tokens while suppressing weak ones.

In conventional transformer architectures, the softmax function maps self-attention scores \(\boldsymbol{z}\in\mathbb{R}^{N}\) to a probability vector. However, softmax often produces over-concentrated attention, in which a few tokens dominate while the weights of the remaining tokens vanish. Temperature scaling is an incomplete remedy: small temperatures preserve concentration, while large temperatures flatten the distribution so aggressively that weak tokens receive undue weight. We therefore seek a mechanism that equalizes attention among genuinely informative tokens while suppressing weak ones.

We compare softmax with normalized sigmoid function (NSF).\footnote{We discuss alternatives to NSF, including entmax and softmax with temperature scaling in \Cref{appendix:NSFAlter}.} For \(\boldsymbol{z}=(z_1,\dots,z_N)\), define
\begin{align}
\alpha_i^{\mathrm{smx}}(\boldsymbol{z};T) \;=\; \frac{e^{z_i/T}}{\sum_{j=1}^{N} e^{z_j/T}}, 
\qquad
\alpha_i^{\mathrm{nsf}}(\boldsymbol{z}) \;=\; \frac{\sigma(z_i)}{\sum_{j=1}^{N} \sigma(z_j)},
\qquad
\sigma(t)=\frac{1}{1+e^{-t}}. \label{Eq:NSF}
\end{align}
For thresholds $\tau>0$ and bandwidth \(\gamma\ge 0\), let \(\mathcal{S}(\tau,\gamma,\mathcal{H},\mathcal{L})\) be the set of score vectors with ``high'' indices \(\mathcal{H}\) satisfying \(z_i\in[\tau,\tau+\gamma]\) for \(i\in\mathcal{H}\) and ``low'' indices \(\mathcal{L}\) satisfying \(z_j\le -\tau\) for \(j\in\mathcal{L}\). Denote \(h \triangleq |\mathcal{H}|\) and \(\ell \triangleq |\mathcal{L}|\). The following theorem (proof deferred to \Cref{app:NSF_proof}) formalizes the selective flattening property of NSF and shows that softmax cannot match it with a single temperature.

\begin{theorem}[NSF achieves selective flattening; softmax cannot with a single \(T\)]\label{thm:nsf_vs_smx}
 Fix \(\tau>0\), $\gamma\ge 0$, and index sets \(\mathcal{H},\mathcal{L}\) with \(h\ge 1\), \(\ell\ge 1\). For any \(\boldsymbol{z}\in \mathcal{S}(\tau,\gamma,\mathcal{H},\mathcal{L})\):

\textbf{(A) NSF bounds.} For any \(i,h'\in\mathcal{H}\) and any \(j\in\mathcal{L}\),
\begin{align}
\frac{\alpha_i^{\mathrm{nsf}}(\boldsymbol{z})}{\alpha_{h'}^{\mathrm{nsf}}(\boldsymbol{z})}
= \frac{\sigma(z_i)}{\sigma(z_{h'})}
\;\le\; \frac{\sigma(\tau+\gamma)}{\sigma(\tau)}
= \frac{1+e^{-\tau}}{1+e^{-(\tau+\gamma)}}
\;\le\; 1 + e^{-\tau},~~~~
\alpha_j^{\mathrm{nsf}}(\boldsymbol{z})
\le\; \frac{\sigma(-\tau)}{h\,\sigma(\tau)}
\;=\; \frac{e^{-\tau}}{h}.
\end{align}
Hence, NSF equalizes the high tokens up to a factor \(1+e^{-\tau}\) and suppresses 
lows % each low token
to at most \(e^{-\tau}/h\). As \(\tau\to\infty\) with fixed \(\gamma\), ratios among high tokens approach \(1\) and low-token weights vanish.

\textbf{(B) Softmax incompatibility with one temperature.} Suppose we desire suppression and equalization targets \((\varepsilon,\kappa)\) on \(\mathcal{S}(\tau,\gamma,\mathcal{H},\mathcal{L})\):
\begin{align*}
\text{(Suppression)}\quad \alpha_j^{\mathrm{smx}}(\boldsymbol{z};T) \le \varepsilon \;\; \forall j\in\mathcal{L},
\qquad
\text{(Equalization)}\quad 
\frac{\max_{i\in \mathcal{H}}\alpha_i^{\mathrm{smx}}(\boldsymbol{z};T)}{\min_{h'\in \mathcal{H}}\alpha_{h'}^{\mathrm{smx}}(\boldsymbol{z};T)} \le \kappa.
\end{align*}
Then \(T\) must satisfy \(T \le \dfrac{2\tau}{\log\!\bigl(\frac{h}{\varepsilon}\bigr)}\) and \(T \ge \dfrac{\gamma}{\log \kappa}\) simultaneously, which is impossible whenever \(\dfrac{\gamma}{\log \kappa} > \dfrac{2\tau}{\log\!\bigl(\frac{h}{\varepsilon}\bigr)}\). Thus, no single temperature achieves both targets for all \(\boldsymbol{z}\in\mathcal{S}(\tau,\gamma,\mathcal{H},\mathcal{L})\).
\end{theorem}

We further illustrate this effect in \Cref{fig:sigmax_motivation} by comparing attention maps with softmax and NSF using ABMIL \citep{ilse2018attention} on a cancer slide from the CAMELYON-16 dataset \citep{CAMELYON16}. Softmax yields a highly concentrated map that obscures broader context, whereas NSF produces a less concentrated attention map that highlights most cancerous regions.

A naive option is to apply NSF directly in the online model. 
In practice, this induces vanishing gradients and degrades performance; see \Cref{appendix:onlineNormalizedSigmoid}.
We therefore place NSF in the \emph{anchor} model as a stable prior, guiding the online model without hindering its learning dynamics.
\begin{wrapfigure}[18]{r}{0.6\textwidth}
\vspace{-0.5cm}
\includegraphics[width=\linewidth]{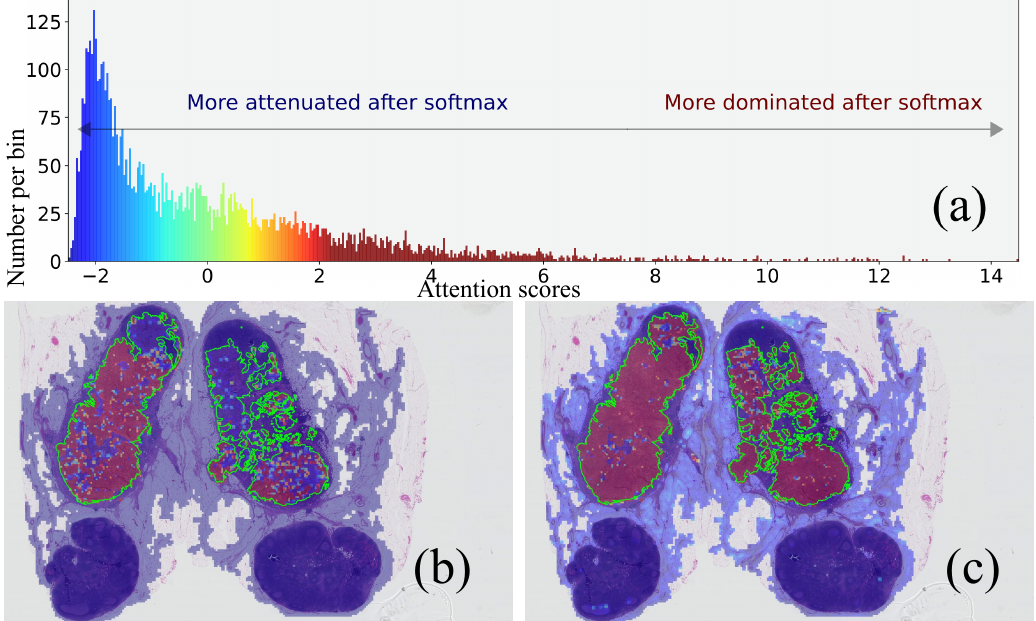}
\vspace{-0.8cm}
\caption{(a) Distribution of attention scores in ABMIL, which exhibits a long-tailed pattern. (b) Attention distribution obtained with the softmax function and (c) with the NSF. Unlike softmax, the normalized sigmoid suppresses large values in the long tail, yielding a less sparse and more interpretable attention distribution.}
\label{fig:sigmax_motivation}
\end{wrapfigure}
As attention distributions lie on the probability simplex, we use the KL divergence to align the online attention distribution with the NSF-based anchor distribution:
\begin{align} \label{eq:KLloss}
    \mathcal{L}_{\mathrm{AS}} \;=\; \mathsf{KL}\!\bigl(\boldsymbol{\alpha}^{\mathrm{nsf}} \,\|\, \boldsymbol{\alpha} \bigr),
\end{align}
where \(\boldsymbol{\alpha}\) is the online attention (softmax over \(\boldsymbol{z}\)) and \(\boldsymbol{\alpha}^{\mathrm{nsf}}\) is the anchor attention (NSF over the anchor scores). Using \(\tfrac{\partial \alpha_j}{\partial z_i}=\alpha_j(\delta_{ij}-\alpha_i)\) and treating \(\boldsymbol{\alpha}^{\mathrm{nsf}}\) as fixed, the gradient with respect to the online attention score \(z_i\) is
\begin{align}\label{eq:grad_LAS}
    \frac{\partial \mathsf{KL}(\boldsymbol{\alpha}^{\mathrm{nsf}}\,\|\,\boldsymbol{\alpha})}{\partial z_i}
    \;&=\; \sum_{j=1}^{N}\alpha^{\mathrm{nsf}}_j(\delta_{ij}-\alpha_i) \nonumber\\&=\; \alpha_i - \alpha^{\mathrm{nsf}}_i.
\end{align}
Thus, gradient descent moves the online attention toward the anchor distribution, promoting stability and discouraging over-concentration.
\begin{remark}
The anchor in ASMIL superficially resembles the teacher in MHIM-MIL \citep{Tang_2023_ICCV}: both are EMA-updated copies of the online model. Their roles, however, differ in two important ways. \textbf{(i)} MHIM-MIL uses the teacher to mine hard instances, whereas ASMIL uses the anchor to stabilize attention and prevent over-concentration. \textbf{(ii)} MHIM-MIL matches softmax bag-level features, while ASMIL directly matches attention distributions. \Cref{appendix:MHIMNotGood} discusses why softmax bag-level matching fails to stabilize attention maps.
\end{remark}
\subsection{Mitigating Overfitting with Token Random Dropping}
\label{Sec:TokenRandomDropping}
To reduce overfitting, {we designed a token-level regularizer, specialized for ASMIL, that operates on the trainable tokens used by the online model.} Let a WSI \(\boldsymbol{x}\) be partitioned into \(M\) tiles and embedded by a pretrained encoder into tile tokens \(\mathcal{T}=\{\boldsymbol{t}_1,\dots,\boldsymbol{t}_M\}\). We augment these with $N$ trainable FEAT tokens \(\mathcal{P}=\{\boldsymbol{p}_1,\dots,\boldsymbol{p}_N\}\) and feed the concatenation \([\mathcal{T};\mathcal{P}]\) into the online encoder. After the online encoder, only the FEAT tokens are retained. Since the number of FEAT tokens is much smaller than the tile tokens (\ie, $N\ll M$), this design acts as information aggregation via token reduction.

During training, we sample an independent Bernoulli mask over FEAT tokens and drop a fraction \(B\in[0,1)\) of them. Denote the kept set by \(\mathcal{P}_{\mathrm{keep}}\) with \(|\mathcal{P}_{\mathrm{keep}}|=\tilde{N}\sim \mathrm{Binomial}(N,1-B)\) and \(\mathbb{E}[\tilde{N}]=(1-B)N\). The remaining tokens, together with a trainable \(\mathrm{[CLS]}\) token, are passed to a second transformer to produce a bag representation \(\boldsymbol{h}_{\mathrm{bag}}\), which is then classified to obtain \(\hat{y}\).
At inference time, no tokens are dropped \((B=0)\). {Since ASMIL stabilize the attention via aligning the anchor model, which assumes a one-to-one correspondence, as thus general instance dropout method, such as MIL-Dropout \cite{zhu2025how}, could not be integrated easily. }

% This stochastic removal prevents co-adaptation among FEAT tokens and discourages the model from over-relying on a small subset of tokens, while preserving image content by keeping all FEAT tokens at inference. Empirically, this acts as an effective regularizer that improves generalization. In \Cref{Sec:RandomDropRate} we study the effect of \(B\) and observe a consistent peak in performance around \(B\approx 0.5\).

This stochastic removal prevents co-adaptation among FEAT tokens and discourages the model from over-relying on a subset of tokens, while preserving image content by keeping all FEAT tokens at inference. Empirically, this acts as an effective regularizer that improves generalization. In \Cref{Sec:RandomDropRate} we study the effect of \(B\) {and observe a consistent peak in performance around  \(B\approx 0.5\).}

\subsection{Overall Training Objective}
Based on the discussion thus far, we train with a joint objective that couples standard bag-level classification with attention stabilization:
\begin{align}
\mathcal{L} \;=\; \mathcal{L}_{\mathrm{CE}} \;+\; \beta\,\mathcal{L}_{\mathrm{AS}}, \label{eq:totalloss}
\end{align}
where the coefficient \(\beta>0\) balances the stabilization and classification objectives. In practice, to calculate  $\mathcal{L}_{\mathrm{AS}}$, \(\boldsymbol{\alpha}\) is computed by a softmax over the online scores, \(\boldsymbol{\alpha}^{\mathrm{nsf}}\) is computed by applying the NSF to the anchor scores, and the anchor model is treated as \textit{stop-gradient} while its parameters are updated via EMA. The KL divergence is taken over the attention distributions on the FEAT token set used for aggregation. This objective discourages attention concentration through \(\mathcal{L}_{\mathrm{AS}}\) and preserves task performance through \(\mathcal{L}_{\mathrm{CE}}\). {ASMIL can be easily applied to other tasks, including survival prediction by replacing the objective function and the classification head accordingly. During training, the online model is updated by gradient descent
\begin{align}
    \theta_{t+1} = \theta_t -\eta \nabla_\theta \mathcal{L},
\end{align}
where $\eta$ is the learning rate. $\mathcal{L}$ is computed as in \Cref{eq:totalloss}. The anchor model is then updated according to \Cref{eq:emaeq}. The gradient is only used to update the online model, while the anchor model influences learning through \Cref{eq:KLloss}. At inference time, ASMIL uses only the online model and discards the anchor model; therefore, the anchor does not increase the computational budget at inference.}
\begin{table}[]
\centering
\vspace{-0.3cm}
\caption{The F1 score and AUC of different MIL approaches across three WSI datasets. \textbf{Bold} and \underline{underlined} values denote the best and second-best results, respectively.}
\vspace{+0.05cm}
\label{Tab:subtyping}
\resizebox{\textwidth}{!}{
\begin{tabular}{clcccccc}
\toprule
\rowcolor{lightgray} \cellcolor{white}&  Dataset              & \multicolumn{2}{c}{CAMELYON-16}             & \multicolumn{2}{c}{CAMELYON-17} & \multicolumn{2}{c}{BRACS}   \\ \cline{2-8} 
Backbone   & Method & F1 score $\uparrow$ & AUC $\uparrow$ & F1 score $\uparrow$ & AUC $\uparrow$ & F1 score  $\uparrow$ & AUC $\uparrow$ \\ \bottomrule
\bottomrule
\multirow{12}{*}{\rotatebox{90}{\begin{tabular}[c]{@{}c@{}}ResNet-18\\ ImageNet Pretrained\end{tabular}}}
& ABMIL \textcolor{mygray}{\tiny ICML  \citeyear{ilse2018attention}} & $0.757\scriptscriptstyle\pm 0.020$          & $0.790\scriptscriptstyle\pm 0.027$          & $0.508\scriptscriptstyle\pm 0.032$   & $0.779\scriptscriptstyle\pm 0.021$   & $0.523\scriptscriptstyle\pm 0.028$          & $0.723\scriptscriptstyle\pm 0.035$      \\
& Clam-SB \textcolor{mygray}{\tiny Nature \citeyear{lu2021data}} & $0.742\scriptscriptstyle\pm 0.024$          & $0.763\scriptscriptstyle\pm 0.049$          &  $0.504\scriptscriptstyle\pm 0.012$          & $0.778\scriptscriptstyle\pm 0.024$   & $0.521\scriptscriptstyle\pm 0.046$          & $0.750\scriptscriptstyle\pm 0.039$      \\
& TransMIL \textcolor{mygray}{\tiny NeurIPS \citeyear{shao2021transmil}} & $0.643\scriptscriptstyle\pm 0.088$          & $0.706\scriptscriptstyle\pm 0.076$          &  $0.499\scriptscriptstyle\pm 0.082$          & $0.794\scriptscriptstyle\pm 0.053$   & $0.444\scriptscriptstyle\pm 0.040$          & $0.732\scriptscriptstyle\pm 0.043$       \\
& DSMIL \textcolor{mygray}{\tiny CVPR \citeyear{li2021dual}} & $0.736\scriptscriptstyle\pm 0.025$          & $0.773\scriptscriptstyle\pm 0.034$          &  $0.473\scriptscriptstyle\pm 0.052$          & $0.705\scriptscriptstyle\pm 0.022$   & $0.511\scriptscriptstyle\pm 0.052$          & $0.751\scriptscriptstyle\pm 0.028$      \\
& DTFD-MIL \textcolor{mygray}{\tiny CVPR \citeyear{zhang2022dtfd}} & $0.758\scriptscriptstyle\pm 0.051$          & $0.815\scriptscriptstyle\pm 0.063$ & $0.546\scriptscriptstyle\pm 0.010$          & $0.735\scriptscriptstyle\pm 0.011$   & $0.469\scriptscriptstyle\pm 0.016$          & $0.717\scriptscriptstyle\pm 0.032$        \\
& IBMIL \textcolor{mygray}{\tiny CVPR \citeyear{lin2023interventional}} & $0.777\scriptscriptstyle\pm 0.009$          & $0.799\scriptscriptstyle\pm 0.050$          & $0.533\scriptscriptstyle\pm 0.015$   & $0.813\scriptscriptstyle\pm 0.092$   & $0.510\scriptscriptstyle\pm 0.043$          & $0.726\scriptscriptstyle\pm 0.034$       \\
& MHIM-MIL \textcolor{mygray}{\tiny ICCV \citeyear{Tang_2023_ICCV}}& $0.752\scriptscriptstyle\pm 0.034$          & $0.772\scriptscriptstyle\pm 0.026$          & $0.56\scriptscriptstyle\pm 0.029$          & $0.815\scriptscriptstyle\pm 0.019$   & $0.511\scriptscriptstyle\pm 0.022$          & $0.775\scriptscriptstyle\pm 0.021$       \\
& ACMIL \textcolor{mygray}{\tiny ECCV \citeyear{zhang2024attention}}& $0.798\scriptscriptstyle\pm 0.029$          & $0.841\scriptscriptstyle\pm 0.030$          & $0.528\scriptscriptstyle\pm 0.053$   & $0.789\scriptscriptstyle\pm 0.046$     & $0.552\scriptscriptstyle\pm 0.048$          & $0.754\scriptscriptstyle\pm 0.008$        \\
& CAMIL \textcolor{mygray}{\tiny ICLR \citeyear{fourkioti2024camil}}& $0.778\scriptscriptstyle\pm 0.011$ & $0.812\scriptscriptstyle\pm 0.017$ & $0.503\scriptscriptstyle\pm 0.007$ & $0.806\scriptscriptstyle\pm 0.006$ & $0.569\scriptscriptstyle\pm 0.007$  &   $\underline{0.787}\scriptscriptstyle\pm 0.011$   \\
& AEM \textcolor{mygray}{\tiny MICCAI \citeyear{zhang2025AEM}}& $\underline{0.804}\scriptscriptstyle\pm 0.022$          & $\underline{0.859}\scriptscriptstyle\pm 0.031$         & $0.525\scriptscriptstyle\pm 0.043$   &   $0.828\scriptscriptstyle\pm 0.054$   & $0.554\scriptscriptstyle\pm 0.004$          & $0.764\scriptscriptstyle\pm 0.008$       \\
& HDMIL \textcolor{mygray}{\tiny CVPR \citeyear{dong2025fast}} & $0.790\scriptscriptstyle\pm 0.023$          & $0.856\scriptscriptstyle\pm 0.027 $         & $\underline{0.557}\scriptscriptstyle\pm 0.007 $   & $\mathbf{0.853}\scriptscriptstyle\pm 0.013$   & $\underline{0.578}\scriptscriptstyle\pm 0.012$          & $0.761\scriptscriptstyle\pm 0.011$      \\
& \texttt{ASMIL} (Ours)         & $\mathbf{0.814}\scriptscriptstyle\pm 0.052$          & $\mathbf{0.870}\scriptscriptstyle\pm 0.064$          & $\mathbf{0.564}\scriptscriptstyle\pm 0.020$   & ${\underline{0.851}\scriptscriptstyle\pm 0.061}$   & $\mathbf{0.601}\scriptscriptstyle\pm 0.072$          & $\mathbf{0.810}\scriptscriptstyle\pm 0.054$       \\ \hline
\multirow{12}{*}{\rotatebox{90}{\begin{tabular}[c]{@{}c@{}}VIT-S \\ SSL pretrained\end{tabular}}}
& ABMIL \textcolor{mygray}{\tiny ICML \citeyear{ilse2018attention}}& $0.914\scriptscriptstyle\pm 0.031$          & $0.945\scriptscriptstyle\pm 0.027$          & $0.522\scriptscriptstyle\pm 0.050$   & $0.853\scriptscriptstyle\pm 0.016$   & $0.680\scriptscriptstyle\pm 0.051$          & $0.866\scriptscriptstyle\pm 0.029$     \\
& Clam-SB \textcolor{mygray}{\tiny Nature \citeyear{lu2021data}}& $0.925\scriptscriptstyle\pm 0.085$          & $0.969\scriptscriptstyle\pm 0.024$          & $0.523\scriptscriptstyle\pm 0.020$   & $0.846\scriptscriptstyle\pm 0.020$   & $0.631\scriptscriptstyle\pm 0.034$          & $0.863\scriptscriptstyle\pm 0.005$     \\
& TransMIL \textcolor{mygray}{\tiny NeurIPS \citeyear{shao2021transmil}}& $0.922\scriptscriptstyle\pm 0.019$          & $0.943\scriptscriptstyle\pm 0.009$          & $0.554\scriptscriptstyle\pm 0.048$    & $0.792\scriptscriptstyle\pm 0.029$   & $0.631\scriptscriptstyle\pm 0.030$          & $0.841\scriptscriptstyle\pm 0.006$  \\
& DSMIL \textcolor{mygray}{\tiny CVPR \citeyear{li2021dual}}& $0.943\scriptscriptstyle\pm 0.007$          & $0.966\scriptscriptstyle\pm 0.009$          & $0.532\scriptscriptstyle\pm 0.064$   & $0.804\scriptscriptstyle\pm 0.032$   & $0.577\scriptscriptstyle\pm 0.028$          & $0.816\scriptscriptstyle\pm 0.028$   \\
& DTFD-MIL \textcolor{mygray}{\tiny CVPR \citeyear{zhang2022dtfd}}& $0.948\scriptscriptstyle\pm 0.007$          & $\underline{0.980}\scriptscriptstyle\pm 0.011$          & $0.627\scriptscriptstyle\pm 0.015 $  & $0.866\scriptscriptstyle\pm 0.012$   & $0.612\scriptscriptstyle\pm 0.080$          & $0.870\scriptscriptstyle\pm 0.022$     \\
& IBMIL \textcolor{mygray}{\tiny CVPR \citeyear{lin2023interventional}} & $0.912\scriptscriptstyle\pm 0.034$          & $0.954\scriptscriptstyle\pm 0.022$          & $0.557\scriptscriptstyle\pm 0.064$   & $0.850\scriptscriptstyle\pm 0.024$   & $0.645\scriptscriptstyle\pm 0.041$          & $0.871\scriptscriptstyle\pm 0.014$     \\
& MHIM-MIL \textcolor{mygray}{\tiny ICCV \citeyear{Tang_2023_ICCV}} & $0.932\scriptscriptstyle\pm 0.024$          & $0.970\scriptscriptstyle\pm 0.037$          & $0.541\scriptscriptstyle\pm 0.022$   & $0.845\scriptscriptstyle\pm 0.026$   & $0.625\scriptscriptstyle\pm 0.060$          & $0.865\scriptscriptstyle\pm 0.017$   \\
& ACMIL \textcolor{mygray}{\tiny ECCV \citeyear{zhang2024attention}} & $0.954\scriptscriptstyle\pm 0.012$          & $0.974\scriptscriptstyle\pm 0.012 $       & $0.562\scriptscriptstyle\pm 0.050$   & $0.863\scriptscriptstyle\pm 0.004$   & $0.722\scriptscriptstyle\pm 0.030$          & $0.888\scriptscriptstyle\pm 0.010$   \\
& CAMIL \textcolor{mygray}{\tiny ICLR \citeyear{fourkioti2024camil}} & $0.930\scriptscriptstyle\pm 0.009$ &  $0.963\scriptscriptstyle\pm 0.011$ & $0.633\scriptscriptstyle\pm 0.022$ &  $0.886\scriptscriptstyle\pm 0.034$ & $0.709\scriptscriptstyle\pm 0.011$  & $0.836\scriptscriptstyle\pm 0.014$ \\
& AEM \textcolor{mygray}{\tiny MICCAI \citeyear{zhang2025AEM}}& $0.947\scriptscriptstyle\pm 0.003$          & $0.974\scriptscriptstyle\pm 0.007 $         & $\underline{0.647}\scriptscriptstyle\pm 0.007 $   & $\underline{0.887}\scriptscriptstyle\pm 0.013$   & $\underline{0.742}\scriptscriptstyle\pm 0.030$          & $\underline{0.905}\scriptscriptstyle\pm 0.010$      \\
& HDMIL \textcolor{mygray}{\tiny CVPR \citeyear{dong2025fast}} & $\underline{0.958}\scriptscriptstyle\pm 0.013$          & $0.976\scriptscriptstyle\pm 0.017 $         & $0.571\scriptscriptstyle\pm 0.012 $   & $0.796\scriptscriptstyle\pm 0.022$   & $0.717\scriptscriptstyle\pm 0.033$          & $0.874\scriptscriptstyle\pm 0.010$      \\
& \texttt{ASMIL} (Ours)        & $\mathbf{0.965}\scriptscriptstyle\pm 0.020$          & $\mathbf{0.985}\scriptscriptstyle\pm 0.017 $         & $\mathbf{0.689}\scriptscriptstyle\pm 0.005$   & $\mathbf{0.898}\scriptscriptstyle\pm 0.010$   & $\mathbf{0.781}\scriptscriptstyle\pm 0.042 $         & $\mathbf{0.914}\scriptscriptstyle\pm 0.014$    \\ \hline
\end{tabular}}
\vspace{-0.4cm}
\end{table}
\section{Experiments}
\label{Sec:Experiments}
To demonstrate the effectiveness of ASMIL, we evaluate it on three well-known public WSI subtyping datasets: ($i$) CAMELYON-16 \citep{CAMELYON16}, ($ii$) CAMELYON-17 \citep{8447230}, and ($iii$) BRACS \citep{brancati2022bracs}. Details of the data splits, preprocessing, training setup, and baselines are provided in the \Cref{Appendix:ExperimentalDetails}. {We further evaluate ASMIL on survival prediction and non-WSI datasets in \Cref{Appendix:SP} and \Cref{Appendix:ASMILOtherDataset}, respectively.}

\subsection{Subtyping Performance}
\label{Sec:SubtypingPerformance}
{We compare ASMIL against eleven attention-based MIL baselines that are designed for WSIs}: CLAM-SB \citep{lu2021data}, TransMIL \citep{shao2021transmil}, DSMIL \citep{li2021dual}, DTFD-MIL \citep{zhang2022dtfd}, IBMIL \citep{lin2023interventional}, MHIM-MIL \citep{Tang_2023_ICCV}, ABMIL \citep{ilse2018attention}, ACMIL \citep{zhang2024attention}, CAMIL \citep{fourkioti2024camil}, AEM \citep{zhang2025AEM} and HDMIL \citep{dong2025fast}. Because WSI datasets are class-imbalanced, we report the F1 score and area under the ROC curve (AUC) for each dataset in \Cref{Tab:subtyping}\footnote{See \Cref{appendix:F1andAUC} for details on metric computation and interpretation.}.

Overall, ASMIL demonstrates superior performance, {achieving state-of-the-art performance on all datasets when paired with an in-domain ViT-SSL backbone, and remains competitive with the best baseline on ImageNet-pretrained ResNet-18 features.} 
% achieving the highest F1 score and AUC across all datasets. The only exception is CAMELYON-17 with features extracted by an ImageNet-pretrained ResNet-18, where the AUC lags the state-of-the-art by 0.002. 
On the BRACS dataset, our method attains an F1 score of 0.781 and an AUC of 0.914, exceeding the previous best results by 3.9 and 0.9 percentage points, respectively. This shows its effectiveness in capturing subtle histopathological features in heterogeneous subtyping tasks.

\begin{table}[]
\centering
% \vspace{-0.5cm}
\caption{Applying anchor model and NSF to other attention-based MIL methods.}
\label{Tab:ablationappltanchorNS}
\resizebox{\textwidth}{!}{\begin{tabular}{lcc|cccccc}
\toprule 
\rowcolor{lightgray}
Dataset                   & \multicolumn{2}{c|}{} & \multicolumn{2}{c}{CAMELYON-16} & \multicolumn{2}{c}{CAMELYON-17}  & \multicolumn{2}{c}{BRACS} \\ \hline
Method                    & Anchor      & NSF      & F1 score $\uparrow$& AUC $\uparrow$& F1 score $\uparrow$& AUC $\uparrow$& F1 score $\uparrow$& AUC $\uparrow$\\ \hline
\multirow{3}{*}{ABMIL \textcolor{mygray}{\tiny ICML \citeyear{ilse2018attention}}}    & \ding{55}  & \ding{55}& $0.914\scriptscriptstyle\pm0.031$   & $0.945\scriptscriptstyle\pm0.027$  &  $0.522\scriptscriptstyle\pm 0.050$   & $0.853\scriptscriptstyle\pm 0.016$     & $0.680\scriptscriptstyle\pm0.051$ & $0.866\scriptscriptstyle\pm0.029$    \\
                          & \ding{51}  & \ding{55}& $0.951\scriptscriptstyle\pm0.015$ & $0.963\scriptscriptstyle\pm0.008$  & $0.573\scriptscriptstyle\pm0.011$  & $0.871\scriptscriptstyle\pm0.010$   &$0.751\scriptscriptstyle\pm0.013$& $0.877\scriptscriptstyle\pm0.007$ \\[-0.4em]
                          & \tiny & \tiny& \tiny{\color{darkgreen}{$+0.037$}\quad\quad\quad\quad} & \tiny{\color{darkgreen}{$+0.018$}\quad\quad\quad\quad}  & \tiny{\color{darkgreen}{$+0.051$}\quad\quad\quad\quad} & \tiny{\color{darkgreen}{$+0.018$}\quad\quad\quad\quad} & \tiny{\color{darkgreen}{$+0.071$}\quad\quad\quad\quad}  & \tiny{\color{darkgreen}{$+0.011$}\quad\quad\quad\quad}  \\%[-0.3em]
                          & \ding{51}  & \ding{51}&$0.953\scriptscriptstyle\pm0.009$&$0.967\scriptscriptstyle\pm0.006$   & $0.574\scriptscriptstyle\pm0.010$  & $0.883\scriptscriptstyle\pm0.014$  &$0.753\scriptscriptstyle\pm0.009$&$0.887\scriptscriptstyle\pm0.014$ \\ [-0.4em]
                          & \tiny & \tiny& \tiny{\color{darkgreen}{$+0.039$}\quad\quad\quad\quad} & \tiny{\color{darkgreen}{$+0.022$}\quad\quad\quad\quad}  & \tiny{\color{darkgreen}{$+0.052$}\quad\quad\quad\quad} &  \tiny{\color{darkgreen}{$+0.030$}\quad\quad\quad\quad} & \tiny{\color{darkgreen}{$+0.073$}\quad\quad\quad\quad}  & \tiny{\color{darkgreen}{$+0.021$}\quad\quad\quad\quad} \\%[-0.3em]
                          \hline
\multirow{3}{*}{CLAM-SB  \textcolor{mygray}{\tiny Nature \citeyear{lu2021data}}}  & \ding{55}  & \ding{55}& $0.925\scriptscriptstyle\pm0.085$   & $0.969\scriptscriptstyle\pm0.024$ & $0.523\scriptscriptstyle\pm 0.020$   & $0.846\scriptscriptstyle\pm 0.020$ & $0.631\scriptscriptstyle\pm0.034$ & $0.863\scriptscriptstyle\pm0.005$  \\
                          & \ding{51}  & \ding{55}& $0.937\scriptscriptstyle\pm0.004$& $0.979\scriptscriptstyle\pm0.015$  & $0.547\scriptscriptstyle\pm0.006$  & $0.887\scriptscriptstyle\pm0.0014$  & $0.678\scriptscriptstyle\pm0.018$ & $0.866\scriptscriptstyle\pm0.007$ \\[-0.4em]
                          &\tiny  &\tiny & \tiny{\color{darkgreen}{$+0.012$}\quad\quad\quad\quad}  & \tiny{\color{darkgreen}{$+0.010$}\quad\quad\quad\quad} & \tiny{\color{darkgreen}{$+0.024$}\quad\quad\quad\quad} & \tiny{\color{darkgreen}{$+0.041$}\quad\quad\quad\quad} &   \tiny{\color{darkgreen}{$+0.047$}\quad\quad\quad\quad} & \tiny{\color{darkgreen}{$+0.003$}\quad\quad\quad\quad~}\\%[-0.3em]
                          & \ding{51}  & \ding{51}&$0.948\scriptscriptstyle\pm0.014$&$0.981\scriptscriptstyle\pm0.021$ &  $0.550\scriptscriptstyle\pm0.006$  & $0.886\scriptscriptstyle\pm0.0015$    &$0.679\scriptscriptstyle\pm0.013$ & $0.887\scriptscriptstyle\pm0.002$ \\ [-0.4em]
                          &\tiny  &\tiny & \tiny{\color{darkgreen}{$+0.023$}\quad\quad\quad\quad}  & \tiny{\color{darkgreen}{$+0.012$}\quad\quad\quad\quad} & \tiny{\color{darkgreen}{$+0.027$}\quad\quad\quad\quad} & \tiny{\color{darkgreen}{$+0.040$}\quad\quad\quad\quad} &   \tiny{\color{darkgreen}{$+0.048$}\quad\quad\quad\quad} & \tiny{\color{darkgreen}{$+0.024$}\quad\quad\quad\quad~} \\\hline%[-0.3em]
\multirow{3}{*}{TransMIL \textcolor{mygray}{\tiny NeurIPS \citeyear{shao2021transmil}}} & \ding{55}  & \ding{55}& $0.922\scriptscriptstyle\pm0.019$   & $0.943\scriptscriptstyle\pm0.009$   &  $0.554\scriptscriptstyle\pm 0.048$    & $0.792\scriptscriptstyle\pm 0.029$    & $0.631\scriptscriptstyle\pm0.030$ & $0.841\scriptscriptstyle\pm0.006$   \\
                          & \ding{51}  & \ding{55}& $0.931\scriptscriptstyle\pm0.001$   & $0.947\scriptscriptstyle\pm0.008$   &  $0.577\scriptscriptstyle\pm0.006$  & $0.824\scriptscriptstyle\pm0.012$    & $0.647\scriptscriptstyle\pm0.024$ & $0.853\scriptscriptstyle\pm0.021$ \\ [-0.4em]
                          & \tiny & \tiny& \tiny{\color{darkgreen}{$+0.009$}\quad\quad\quad\quad} & \tiny{\color{darkgreen}{$+0.004$}\quad\quad\quad\quad}   & \tiny{\color{darkgreen}{$+0.023$}\quad\quad\quad\quad}  & \tiny{\color{darkgreen}{$+0.032$}\quad\quad\quad\quad}    & \tiny{\color{darkgreen}{$+0.016$}\quad\quad\quad\quad}  & \tiny{\color{darkgreen}{$+0.012$}\quad\quad\quad\quad}  \\%[-0.3em]
                          & \ding{51}  & \ding{51}& $0.933\scriptscriptstyle\pm0.023$&$0.954\scriptscriptstyle\pm0.021$   &  $0.580\scriptscriptstyle\pm0.008$  & $0.829\scriptscriptstyle\pm0.010$  &$0.672\scriptscriptstyle\pm0.024$&$0.883\scriptscriptstyle\pm0.041$  \\ [-0.4em]
                          & \tiny & \tiny& \tiny{\color{darkgreen}{$+0.011$}\quad\quad\quad\quad} & \tiny{\color{darkgreen}{$+0.011$}\quad\quad\quad\quad}    & \tiny{\color{darkgreen}{$+0.026$}\quad\quad\quad\quad}  & \tiny{\color{darkgreen}{$+0.037$}\quad\quad\quad\quad}   & \tiny{\color{darkgreen}{$+0.041$}\quad\quad\quad\quad}  & \tiny{\color{darkgreen}{$+0.045$}\quad\quad\quad\quad}  \\%[-0.3em]
                          \hline
\multirow{3}{*}{DSMIL \textcolor{mygray}{\tiny CVPR \citeyear{li2021dual}}}    & \ding{55}  & \ding{55}& $0.943\scriptscriptstyle\pm0.007$   & $0.966\scriptscriptstyle\pm0.009$   &   $0.532\scriptscriptstyle\pm 0.064$   & $0.804\scriptscriptstyle\pm 0.032$    & $0.577\scriptscriptstyle\pm0.028$ & $0.816\scriptscriptstyle\pm0.028$ \\
                          & \ding{51}  & \ding{55}& $0.943\scriptscriptstyle\pm0.001$   & $0.974\scriptscriptstyle\pm0.007$  &  $0.544\scriptscriptstyle\pm0.038$  & $0.819\scriptscriptstyle\pm0.031$     & $0.609\scriptscriptstyle\pm0.012$ & $0.837\scriptscriptstyle\pm0.013$  \\[-0.4em]
                          & \tiny & \tiny& \tiny{$\pm0.000$\quad\quad\quad\quad} & \tiny{\color{darkgreen}{$+0.008$}\quad\quad\quad\quad}   & \tiny{\color{darkgreen}{$+0.012$}\quad\quad\quad\quad} & \tiny{\color{darkgreen}{$+0.015$}\quad\quad\quad\quad}   & \tiny{\color{darkgreen}{$+0.032$}\quad\quad\quad\quad}  & \tiny{\color{darkgreen}{$+0.021$}\quad\quad\quad\quad} \\%[-0.3em]
                          & \ding{51}  & \ding{51}&$0.942\scriptscriptstyle\pm0.026$&$0.985\scriptscriptstyle\pm0.022$  & $0.559\scriptscriptstyle\pm0.028$  & $0.823\scriptscriptstyle\pm0.019$   &$0.612\scriptscriptstyle\pm0.031$&$0.849\scriptscriptstyle\pm0.042$  \\ [-0.4em]
                          & \tiny & \tiny& \tiny{\color{darkred}{$-0.001$}\quad\quad\quad\quad} & \tiny{\color{darkgreen}{$+0.019$}\quad\quad\quad\quad}   & \tiny{\color{darkgreen}{$+0.027$}\quad\quad\quad\quad} &  \tiny{\color{darkgreen}{$+0.019$}\quad\quad\quad\quad}  & \tiny{\color{darkgreen}{$+0.035$}\quad\quad\quad\quad}  & \tiny{\color{darkgreen}{$+0.033$}\quad\quad\quad\quad} \\\hline%[-0.3em]
\end{tabular}}
\vspace{-0.4cm}
\end{table}

% We further evaluate performance on CAMELYON-16 and CAMELYON-17, datasets characterized by sparse tumor regions, where malignant tissue can constitute as little as 5\% of the slides \citep{cheng2021computational}. Here, ASMIL’s advantages become even more pronounced: on CAMELYON-16, we observe a 1.1\% increase in F1 score and a 0.5\% uplift in AUC compared to the strongest baseline; similarly, on CAMELYON-17, ASMIL improves the F1 score by 3.5\%, which highlights ASMIL’s efficacy under an ill-posed, weakly supervised task. We compare the computational cost of ASMIL with that of other benchmarks in \Cref{Appendix:CcostASMILVsOthers}. 

For CAMELYON-16 and CAMELYON-17 datasets with sparse tumor regions, where malignant tissue may occupy as little as $5\%$ of a slide \citep{cheng2021computational},the advantages are even more pronounced. on CAMELYON-16, we observe a 3.3\% increase in F1 score and a 1.6\% uplift in AUC compared to the strongest baseline; similarly, on CAMELYON-17, ASMIL improves the F1 score by 6.49\%, which highlights ASMIL’s efficacy under an ill-posed, weakly supervised task. We compare the computational cost of ASMIL with that of other benchmarks in \Cref{Appendix:CcostASMILVsOthers}.

\subsection{Integrating the Anchor Model and NSF with Other MIL Methods}
\label{appendix:CombineAnchorWithMILs} 
We regard the anchor model as a general plug-in module for attention-based MIL in WSI analysis. Accordingly, for each baseline we evaluate two variants while keeping all other components and hyperparameters fixed: ($i$) \textbf{+Anchor} (EMA-updated anchor with attention matching), and ($ii$) \textbf{+Anchor+NSF} (anchor updated by EMA and using NSF). The results are summarized in \Cref{Tab:ablationappltanchorNS}. As shown, adding the anchor model and the NSF consistently improves performance, with F1 score gains up to $10.73\%$ (for ABMIL on BRACS), except when adding the anchor to DSMIL on the CAMELYON-16  dataset, where the F1 score decreases by $0.001$ relative to the original model. The additional computational cost introduced by the anchor model is reported in \Cref{Appendix:AdditionalComputationalCostIntroducedbyAnchorModel}.

\begin{wraptable}[7]{r}{0.42\textwidth}
\vspace{-2cm}
\centering
\caption{Component-wise ablation of ASMIL on BRACS. We evaluate the contribution of the anchor model, NSF, and random drop (rd).}
\label{tab:component_ablation}
\resizebox{0.43\textwidth}{!}{%
\begin{tabular}{c|c|c|cc}
\toprule
\rowcolor{lightgray} Anchor & NSF & rd & F1 score $\uparrow$ & AUC $\uparrow$ \\ \hline
\ding{51} & \ding{51} & \ding{51}  & $\mathbf{0.781\scriptscriptstyle\pm 0.042}$ & $\mathbf{0.914\scriptscriptstyle\pm 0.014}$ \\
\ding{51} & \ding{51} & \ding{55}  & $0.765\scriptscriptstyle\pm 0.030$ & $0.903\scriptscriptstyle\pm 0.018$ \\
\ding{51} & \ding{55}  & \ding{51} & $0.759\scriptscriptstyle\pm 0.028$ & $0.895\scriptscriptstyle\pm 0.012$ \\
\ding{51} & \ding{55}  & \ding{55}  & $0.747\scriptscriptstyle\pm 0.026$ & $0.887\scriptscriptstyle\pm 0.015$ \\
% \ding{55}  & \ding{51} & \ding{51} & $0.741\scriptscriptstyle\pm 0.025$ & $0.880\scriptscriptstyle\pm 0.016$ \\
% \ding{55}  & \ding{51} & \ding{55}  & $0.736\scriptscriptstyle\pm 0.021$ & $0.873\scriptscriptstyle\pm 0.013$ \\
\ding{55}  & \ding{55}  & \ding{51} & $0.728\scriptscriptstyle\pm 0.019$ & $0.868\scriptscriptstyle\pm 0.010$ \\
\ding{55}  & \ding{55}  & \ding{55}  & $0.712\scriptscriptstyle\pm 0.020$ & $0.860\scriptscriptstyle\pm 0.012$ \\ \hline
\end{tabular}%
}
% \vspace{-0.5cm}
\end{wraptable}

\subsection{Localization}
We evaluate tumor localization on CAMELYON-16 both qualitatively and quantitatively. Qualitative heatmaps are shown in \Cref{fig:Localization}. Compared with baseline methods, ASMIL consistently highlights all cancerous regions. We attribute these gains to reduced over-concentration by the NSF in the anchor model, which yields more faithful attention distributions.

Following the official CAMELYON-16 and \citet{fourkioti2024camil}, we report lesion-level Free-Response ROC (FROC) \citep{miller1969froc, bunch1978free} the Dice coefficient on cancerous slides, and tile-level specificity on normal slides. To obtain the predicted masks, we use scaled attention distributions for CLAM \citep{lu2021data}, TransMIL \citep{shao2021transmil}, DSMIL \citep{li2021dual}, and CAMIL \citep{fourkioti2024camil}; tile-level logits for DTFD-MIL \citep{zhang2022dtfd}; and for ASMIL, the per-tile average of FEAT-token attentions. Quantitative results for FROC, Dice, and specificity, as well as additional attention-map visualizations, are provided in \Cref{appendix:local_vis}.
\label{sec:localization}
\begin{figure}
\begin{center}
\includegraphics[width=\linewidth]{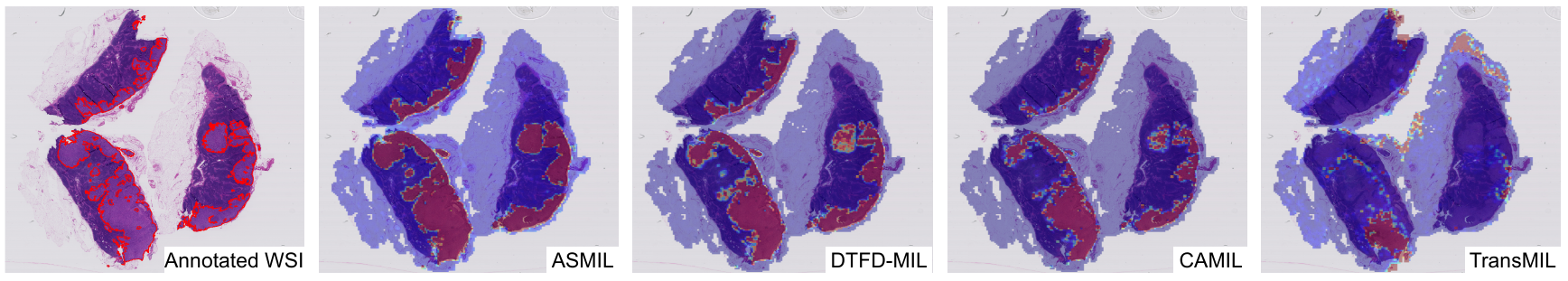}
\end{center}
\vspace{-0.5cm}
\caption{Visual comparison of attention maps on the CAMELYON-16 dataset. The left column shows the original WSI with ground-truth tumor annotations outlined in red; the remaining columns present attention maps for ASMIL (ours), DTFD-MIL, CAMIL, and TransMIL (left to right).}
\vspace{-0.7cm}
\label{fig:Localization}
\end{figure}

\subsection{Ablation Study}
\label{Sec:Ablation}
\label{Sec:AnchorModelUpdate}

Lastly, we evaluate the effect of the anchor model, NSF, and random drop (rd) by enabling or disabling them in all combinations. As shown in \Cref{tab:component_ablation}, the full model (all three enabled) achieves the best F1 score and AUC. Removing any component degrades performance, with the anchor model having the largest impact. Without all three, the model drops to the lowest scores, confirming that each component contributes to the overall effectiveness of ASMIL.  Additional ablations on the loss weight $\beta$, the number of trainable FEAT tokens, the EMA factor $m$, the anchor update frequency, and the random drop rate are reported in \Cref{Appendix:ablation}.

\section{Conclusion}

\label{Sec:Conclusion}
In this work, we identified a previously overlooked failure mode in attention-based MIL for WSI: unstable attention dynamics that hinder convergence. We proposed ASMIL, which stabilizes training via an anchor model, prevents over-concentration by using a normalized sigmoid in the anchor, and mitigates overfitting with token dropout. Across multiple WSI benchmarks, ASMIL improves classification performance and state-of-the-art localization performance. These results underscore the importance of jointly controlling attention stability, concentration, and overfitting in weakly supervised WSI analysis. We anticipate that the proposed anchor model and normalized sigmoid function will serve as building blocks for future MIL-based WSI analysis algorithms, ultimately facilitating more accurate and interpretable analysis of gigapixel pathology images. Due to space constraints, we defer the discussion of future work and limitations to \Cref{Appendix:LimitationsFutureWorks}.

\section*{Ethics Statement}
All WSI datasets used in this work are publicly available and were obtained from open-access websites. The usage of these datasets strictly follows the terms and conditions set by the dataset providers and adheres to established academic and research community standards. No personally identifiable information or sensitive patient data is involved.

\section*{Reproducibility Statement}
We have taken steps to ensure our results are reproducible. All model and algorithmic details, training procedures, hyperparameters, evaluation protocols, and metrics are specified in the main text. The appendix provides complete proofs, implementation notes, ablations, and additional qualitative results. An anonymized GitHub repository contains the source code and configuration files, and pre-trained checkpoints. All datasets used in our experiments are publicly available; download links, data splits, and preprocessing steps are documented in the repository and referenced in the appendix.

% \clearpage

\bibliography{iclr2025_conference}
\bibliographystyle{iclr2025_conference}
% \clearpage
\appendix

%\section{Related work1}
%EMA-based target networks are widely used during training. Mean Teacher \citet{tarvainen2017mean} maintains an exponential moving average (EMA) of the student parameters to enforce prediction consistency and effectively exploit limited labels. Building on this idea, BYOL \cite{chen2020exploring,grill2020bootstrap} adopts a non-contrastive Siamese framework in which a target branch is updated as an EMA “teacher,” and representation collapse is prevented through architectural asymmetry and a dedicated predictor head. DINO \citep{caron2021emerging,oquab2023dinov2} extends EMA-based self-distillation to Vision Transformers, while DINOv3 \cite{simeoni2025dinov3} further stabilizes dense features by introducing a Gram-anchoring objective that constrains patch–patch similarity structures over long training schedules. Collectively, these works establish EMA teachers as central to semi-supervised consistency and to stable, non-contrastive representation learning.

%Our approach differs in intent and mechanism. While ASMIL’s anchor model superficially resembles a moving-average teacher, it is introduced to stabilize attention distributions during supervised MIL training and to mitigate attention over-concentration, rather than to provide representation targets, enlarge the effective training set, or prevent collapse. The anchor serves as a moving reference for attention maps only, whereas patch embeddings and slide-level predictions are still learned directly from the supervised MIL objective.

\section{Related Work on EMA Models and Anchoring Strategies}
\label{Appendix:EMAliter}
EMA-based target networks are central in self-supervised representation learning. Mean Teacher \citet{tarvainen2017mean} maintains an EMA of student parameters and enforces prediction consistency with this temporal ensemble under limited supervision. BYOL \citep{grill2020bootstrap,wu2023metagcd} uses an EMA-updated target network to provide representation targets for an online network with an additional predictor, and avoids collapse through architectural asymmetry and the EMA update instead of negatives. DINO-style methods \citep{caron2021emerging,oquab2023dinov2} adapt EMA self-distillation to Vision Transformers, where an EMA teacher produces soft probability targets on multi-crop views; centering, sharpening, and a momentum schedule control the stability–adaptation trade-off of these targets.

{DINOv3 \citep{simeoni2025dinov3} revisits EMA teachers for dense prediction and studies how dense features drift or collapse under long training. It introduces \emph{Gram anchoring}, which aligns Gram matrices of patch--patch similarities between a student and its EMA teacher so that dense features remain close to a temporally smoothed reference. The EMA momentum and the strength of this anchoring loss jointly determine how strongly dense features are tied to the teacher versus how quickly they adapt.}

{ASMIL also maintains an EMA-updated copy of the model, but uses it in a different regime and on a different target. Training is fully supervised at the bag level, and the EMA branch does not supply pseudo-labels or representation targets. Instead, it defines a temporally smoothed \emph{attention distribution} over tiles, and the anchor enters the loss only through the KL term in Eq.~\eqref{eq:KLloss}, while the bag-level cross-entropy in Eq.~\eqref{eq:totalloss} provides all semantic supervision. The shared encoder and classifier parameters are optimized by standard backpropagation; the EMA update acts purely as a temporal regularizer on attention, in contrast to BYOL/DINO, which anchors global embeddings, and DINOv3, which anchors patch--patch similarity structure.}

{The same EMA hyperparameters induce an analogous stability–adaptation trade-off but at the level of attention rather than features. The EMA momentum in ASMIL sets how rapidly the anchor follows the online model, and the weight $\beta$ on the KL term controls how strongly attention is pulled toward the temporally smoothed reference. Unlike BYOL and DINO/DINOv3, where EMA model is designed to avoid global representation collapse, ASMIL uses EMA anchoring to reduce unstable and over-concentrated attention patterns observed under purely online MIL training, while the supervised objective already discourages trivial constant-attention solutions.}

% \section{Related Works on EMA Model in Representation Learning}
% \label{Appendix:EMAliter}

% EMA-based target networks are widely used during training. Mean Teacher \citet{tarvainen2017mean} maintains an exponential moving average (EMA) of the student parameters to enforce prediction consistency and effectively leverage limited labels. Building on this idea, BYOL \cite{chen2020exploring,grill2020bootstrap} employs a non-contrastive Siamese framework in which the target branch is an EMA “teacher,” and representation collapse is avoided via architectural asymmetry. DINO \citep{caron2021emerging,oquab2023dinov2} adapts self-distillation to Vision Transformers with an EMA teacher. {One contemporaneous work that is conceptually related to ASMIL is DINOv3 \cite{simeoni2025dinov3}, which introduces Gram anchoring on patch–patch similarities to prevent dense features from drifting during long training and thereby improve performance on dense vision tasks.}
% Collectively, these works establish EMA teachers as central to semi-supervised consistency and to stable, non-contrastive representation learning. 

% Our approach differs in intent and mechanism. While ASMIL’s anchor model superficially resembles a moving-average teacher, it is introduced to stabilize attention distributions during training and to mitigate attention over-concentration, rather than to provide supervisory targets or expand the effective training set. The anchor serves as a moving reference for attention distributions.

\section{Experimental Details}
\label{Appendix:ExperimentalDetails}

We train all models for 50 epochs with a batch size of 1, using Adam (weight decay $10^{-4}$) and a cosine learning rate schedule with an initial learning rate of $10^{-4}$. All reported results are averaged over five random seeds.

\subsection{WSI Pre-processing}

For all datasets, we used the publicly available CLAM WSI preprocessing toolbox \citep{lu2021data} to segment tissue regions and divide each slide into non-overlapping $256\times256$ patches at $20\times$ magnification. Tissue segmentation was performed automatically using Otsu’s thresholding. To reduce computational overhead and leverage previously learned representations, we adopted a ResNet-18 model \citep{resnet} pretrained on ImageNet \citep{ILSVRC15} and an open-source self-supervised ViT-small model \citep{kang2022benchmarking} as feature extractors\footnote{The checkpoint is available at \url{https://github.com/lunit-io/benchmark-ssl-pathology}.}. The ViT-small model was pretrained on 36,666 whole slide images from The Cancer Genome Atlas (TCGA) and the internally collected TULIP dataset. For consistency and fairness in the subtyping task, we used the same feature extractors across all baseline methods.

For the localization experiments, following \citet{TOURNIAIRE2023102763}, we used a ResNet-18 backbone pretrained with SimCLR \citep{chen2020simple}\footnote{The checkpoint is available at \url{https://github.com/binli123/dsmil-wsi}.}. This feature extractor maps each tile to a 1024-dimensional feature vector.

\subsection{Datasets}
 
CAMELYON-16 \citep{CAMELYON16} is a widely used publicly available WSI dataset designed for lymph node metastasis detection. It contains 270 training and 129 test slides collected from two medical centers, with detailed pixel-level annotations provided by expert pathologists. Notably, some slides include only partial annotations, making the dataset particularly challenging due to the presence of small or sparse metastatic regions. CAMELYON-16 has become a standard benchmark for evaluating weakly supervised and fully supervised algorithms in computational pathology.

CAMELYON-17 \citep{8447230} extends the scope of CAMELYON-16 by including a total of 1,000 WSIs from five medical centers, making it a more diverse and clinically representative dataset. Among these, 500 slides are publicly available and come with slide-level labels, while the remaining 500 are held out for challenge-based evaluations. The inclusion of data from multiple institutions introduces significant variability in staining and scanning conditions, making CAMELYON-17 a suitable benchmark for testing the generalization performance of WSI-based models.

The BRACS dataset \citep{brancati2022bracs} is a large-scale WSI dataset curated for the task of breast cancer subtype classification. It comprises 547 WSIs collected from several medical institutions and annotated by expert pathologists into clinically relevant categories: benign tumors, atypical tumors, and malignant tumors. These labels reflect the progression of breast lesions and are critical for diagnostic decision-making and treatment planning. BRACS captures a wide range of histological appearances and staining variations, making it a valuable resource for developing and benchmarking MIL and weakly supervised classification models in real-world clinical settings.

\subsection{Data Splits}

Following \citet{zhang2025AEM, zhang2024attention}, we partition the datasets as follows. For CAMELYON-16, the WSIs are divided into training, validation, and test sets. The 270 WSIs from Hospital 1 are split, five times, into training (90\%) and validation (10\%) subsets; the 130 WSIs from Hospital 2 are used as a test set. The official test set of 129 WSIs is used for final evaluation. For CAMELYON-17, we use 500 WSIs in total: 300 WSIs from three hospitals for training/validation (90\%, 10\%) and 200 WSIs from two other hospitals for testing to assess out-of-distribution (OOD) performance. For BRACS, we follow the official split: 395 slides for training, 65 for validation, and 87 for testing. The task is a three-class WSI classification—benign tumor, atypical tumor, and malignant tumor. All results are averaged over five random seeds, and we report the mean performance on the official competition test set.

\begin{figure}[h]
\centering
\includegraphics[width=\linewidth]{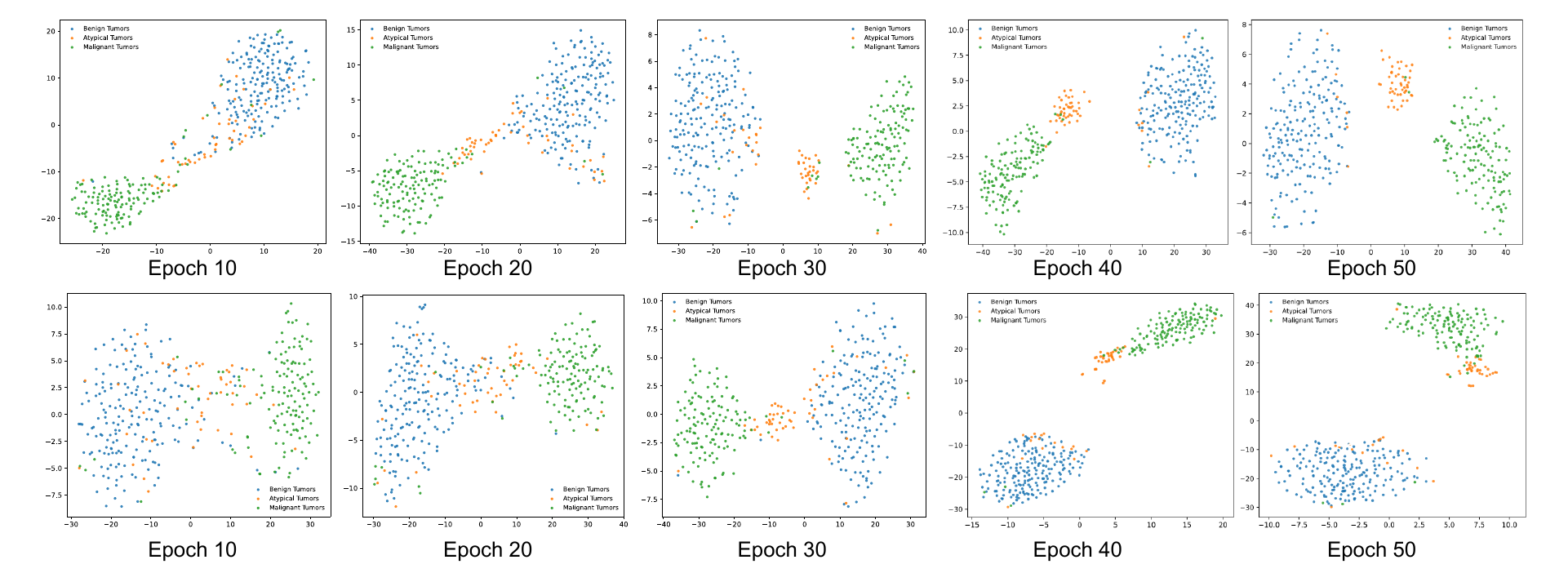}
\vspace{-0.8cm}
\caption{T-SNE embeddings of ASMIL bag-level features on the BRACS training set across training epochs. \textbf{Top:} with the anchor model; \textbf{Bottom:} without the anchor model.}
\label{fig:TsneASMIL_Bracs}
\end{figure}
\begin{figure}[h]
\centering
\includegraphics[width=\linewidth]{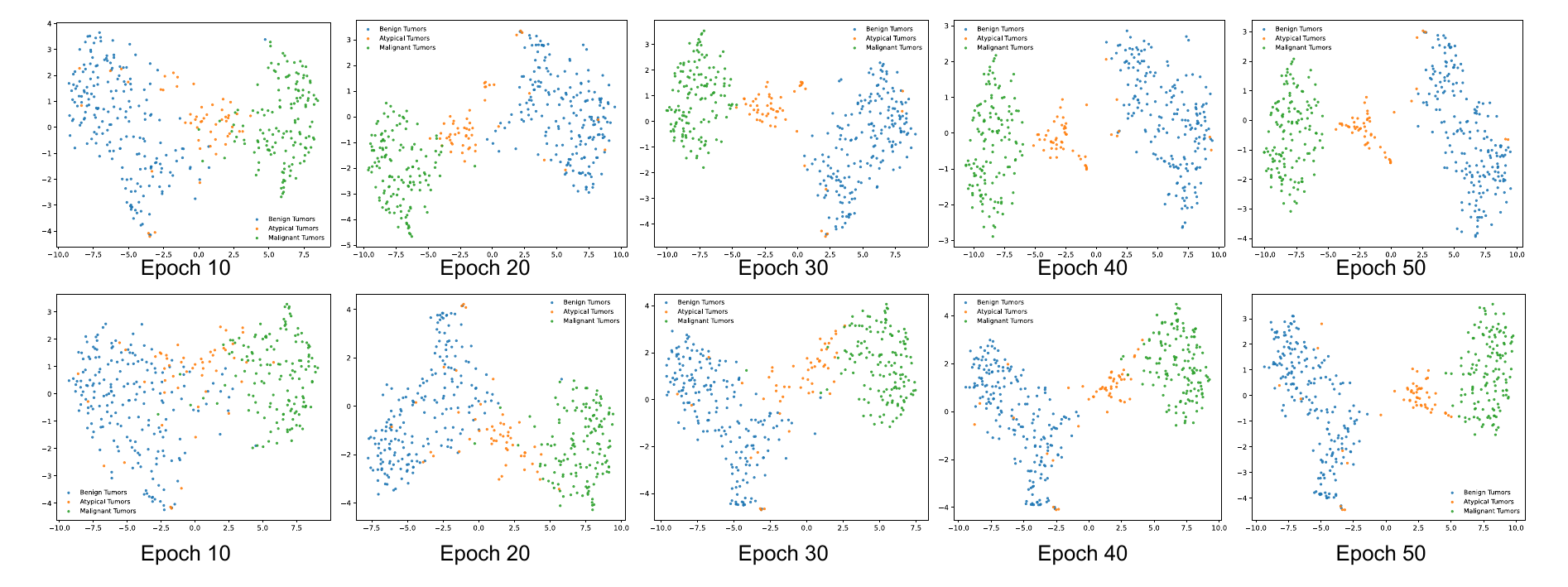}
\vspace{-0.8cm}
\caption{T-SNE embeddings of TransMIL bag-level features on the BRACS training set across training epochs. \textbf{Top:} with the anchor model; \textbf{Bottom:} without the anchor model.}
\label{fig:TsneTranMIL_Bracs}
\end{figure}
\section{T-SNE Visualization of Bag-Level Features}
\label{Sec:TSNE_bag_feature}
To assess how the anchor model stabilizes attention during training, we visualize the bag-level representations learned by ASMIL using t-SNE~\cite{maaten2008visualizing}; see \Cref{fig:TsneASMIL_Bracs}. Compared to ASMIL without the anchor, the model with an anchor forms more distinct clusters and exhibits clearer inter-class boundaries across training epochs, indicating faster convergence and more discriminative features. We observe a similar trend for TransMIL in \Cref{fig:TsneTranMIL_Bracs}.

\section{Macro AUC and Macro F1 score under Class Imbalance}
\label{appendix:F1andAUC}

Since all datasets considered in this work are class-imbalanced, we report \emph{macro-averaged} variants of the area under the ROC curve (AUC) and the F1 score as our primary summary metrics. Macro-averaging assigns equal weight to each class and therefore prevents majority classes from dominating the overall score.

\paragraph{Setup.}
Let $\mathcal{Y}=\{1,\dots,K\}$ denote the set of classes. For a sample $x$ with true label $y \in \mathcal{Y}$, let $s_k(x)\in\mathbb{R}$ be the model score for class $k$. Define one-vs-rest binary indicators $y_k=\mathbb{1}[y=k]$ for each class $k$, and the corresponding confusion-matrix counts $(\mathrm{TP}_k,\mathrm{FP}_k,\mathrm{FN}_k,\mathrm{TN}_k)$ computed by treating class $k$ as “positive’’ and all others as “negative’’.

\subsection[Macro-F1]{Macro-$F1$}
For class $k$, precision and recall are
\begin{align}
\mathrm{Precision}_k=\frac{\mathrm{TP}_k}{\mathrm{TP}_k+\mathrm{FP}_k},\qquad
\mathrm{Recall}_k=\frac{\mathrm{TP}_k}{\mathrm{TP}_k+\mathrm{FN}_k}.
\end{align}
The per-class $F1$ is the harmonic mean of precision and recall:
\begin{align}
F1_k \;=\; \frac{2\,\mathrm{Precision}_k\,\mathrm{Recall}_k}{\mathrm{Precision}_k+\mathrm{Recall}_k}.
\end{align}
The \emph{macro-$F1$} averages the per-class values uniformly:
\begin{align}
\mathrm{Macro}\text{-}F1 \;=\; \frac{1}{K}\sum_{k=1}^{K} F_{1,k}.
\end{align}
As a thresholded, decision-level metric, $F1_{k}$ (and thus macro-$F1$) depends on the classification threshold applied to scores $s_k(x)$. We use a threshold of $0.5$ for all experiments. The same definition applies to multilabel settings by averaging over labels.

\subsection{Macro-AUC (ROC)}
For class $k$, the ROC curve plots the true positive rate against the false positive rate as the threshold on $s_k(x)$ varies:
\begin{align}
\mathrm{TPR}_k=\frac{\mathrm{TP}_k}{\mathrm{TP}_k+\mathrm{FN}_k}, \qquad
\mathrm{FPR}_k=\frac{\mathrm{FP}_k}{\mathrm{FP}_k+\mathrm{TN}_k}.
\end{align}
The per-class AUC, $\mathrm{AUC}_k\in[0,1]$, is the area under this curve; equivalently, it is the probability that a randomly chosen positive example (for class $k$) receives a higher score than a randomly chosen negative example. The \emph{macro-AUC} is the uniform average across classes:
\begin{align}
\mathrm{Macro}\text{-}\mathrm{AUC} \;=\; \frac{1}{K}\sum_{k=1}^{K} \mathrm{AUC}_k.
\end{align}
Unlike $F1$, AUC is threshold-agnostic and measures the ranking quality of scores.

% \paragraph{Why macro-averaging?}
% Under class imbalance, micro-averaging (which aggregates all $\mathrm{TP}$, $\mathrm{FP}$, $\mathrm{FN}$, $\mathrm{TN}$ before computing a single score) is dominated by majority classes. Macro-averaging treats each class equally, yielding an evaluation that is more sensitive to minority-class performance, which is critical in our setting.

\section[Proof of Theorem~1]{Proof of \Cref{thm:nsf_vs_smx}} \label{app:NSF_proof}
\begin{proof}
We proceed in two parts.

\textbf{Part A: NSF bounds.}
Let $s_i=\sigma(z_i)$ and $S=\sum_{j=1}^{N}\sigma(z_j)$, so $\alpha_i^{\mathrm{nsf}}=s_i/S$.

\emph{Equalization among highs.}
For $i,h'\in \mathcal{H}$,
\begin{align}
\frac{\alpha_i^{\mathrm{nsf}}}{\alpha_{h'}^{\mathrm{nsf}}}
= \frac{s_i}{s_{h'}}
= \frac{\sigma(z_i)}{\sigma(z_{h'})}.
\end{align}
Since $\sigma$ is strictly increasing and $z_i,z_{h'}\in[\tau,\tau+\gamma]$,
\begin{align}
\frac{\sigma(z_i)}{\sigma(z_{h'})}
\ \le\ \frac{\sigma(\tau+\gamma)}{\sigma(\tau)}
\ =\ \frac{1+e^{-\tau}}{1+e^{-(\tau+\gamma)}}
\ \le\ 1+e^{-\tau}.
\end{align}
\emph{Suppression of lows.}
For any $j\in \mathcal{L}$ we have $z_j\le -\tau$. Using monotonicity and the identity
\begin{align}
\sigma(-t)=e^{-t}\,\sigma(t)\qquad\text{for all }t\in\mathbb{R},
\label{eq:sig_identity}
\end{align}
we get $\sigma(z_j)\le \sigma(-\tau)=e^{-\tau}\sigma(\tau)$. Meanwhile
\begin{align}
S=\sum_{i=1}^{N}\sigma(z_i)\ \ge\ \sum_{i\in \mathcal{H}}\sigma(z_i)\ \ge\ h\,\sigma(\tau),
\end{align}
since $z_i\ge \tau$ for $i\in \mathcal{H}$. Hence
\begin{align}
\alpha_j^{\mathrm{nsf}}=\frac{\sigma(z_j)}{S}
\ \le\ \frac{e^{-\tau}\sigma(\tau)}{h\,\sigma(\tau)}
\ =\ \frac{e^{-\tau}}{h}.
\end{align}
For completeness, \eqref{eq:sig_identity} follows from
\(
\sigma(-t)=\frac{1}{1+e^{t}}
=\frac{e^{-t}}{1+e^{-t}}
=e^{-t}\sigma(t).
\)

\textbf{Part B: Softmax temperature constraints.}
Fix $T>0$ and $\boldsymbol{z}\in\mathcal{S}(\tau,\gamma,\mathcal{H},\mathcal{L})$.

\emph{Equalization among highs.}
For any $i,h'\in\mathcal{H}$,
\begin{align}
\frac{\alpha_i^{\mathrm{smx}}}{\alpha_{h'}^{\mathrm{smx}}}
=\frac{e^{z_i/T}}{e^{z_{h'}/T}}
=e^{(z_i-z_{h'})/T}.
\end{align}
Over $\mathcal{S}(\tau,\gamma,\mathcal{H},\mathcal{L})$, the worst high to high ratio occurs at $z_i=\tau+\gamma$ and $z_{h'}=\tau$, so
\begin{align}
\frac{\max_{i\in \mathcal{H}}\alpha_i^{\mathrm{smx}}}{\min_{h'\in \mathcal{H}}\alpha_{h'}^{\mathrm{smx}}}
\ \ge\ e^{\gamma/T}.
\end{align}
Therefore, the uniform bound
\(
\frac{\max_{i\in \mathcal{H}}\alpha_i^{\mathrm{smx}}}{\min_{h'\in \mathcal{H}}\alpha_{h'}^{\mathrm{smx}}}\ \le\ \kappa
\)
for all $\boldsymbol{z}\in\mathcal{S}(\tau,\gamma,\mathcal{H},\mathcal{L})$ implies
\begin{align}
T\ \ge\ \frac{\gamma}{\log \kappa}.
\label{eq:T_lower_eq}
\end{align}
\emph{Suppression of lows.}
Fix $j\in \mathcal{L}$. For a given $T$, the quantity $\alpha_j^{\mathrm{smx}}(\boldsymbol{z};T)$ is maximized over $\mathcal{S}(\tau,\gamma,\mathcal{H},\mathcal{L})$ by taking
\(
z_j=-\tau,\ z_i=\tau\ \forall i\in \mathcal{H},\ z_k\to -\infty\ \text{for }k\notin \mathcal{H}\cup\{j\},
\)
which minimizes the denominator subject to the constraints. Thus
\begin{align}
\sup_{\boldsymbol{z}\in\mathcal{S}(\tau,\gamma,\mathcal{H},\mathcal{L})} \alpha_j^{\mathrm{smx}}(\boldsymbol{z};T)
=\frac{e^{-\tau/T}}{h\,e^{\tau/T} + e^{-\tau/T}}
=\frac{1}{h\,e^{2\tau/T}+1}.
\label{eq:low_sup_full}
\end{align}
Consequently, the uniform suppression requirement $\alpha_j^{\mathrm{smx}}(\boldsymbol{z};T)\le \varepsilon$ for all $\boldsymbol{z}\in\mathcal{S}(\tau,\gamma,\mathcal{H},\mathcal{L})$ forces
\begin{align}
\frac{1}{h\,e^{2\tau/T}+1}\ \le\ \varepsilon
\quad\Longleftrightarrow\quad
h\,e^{2\tau/T}\ \ge\ \frac{1}{\varepsilon}-1
\quad\Longleftrightarrow\quad
T\ \le\ \frac{2\tau}{\log\!\bigl(\tfrac{1}{\varepsilon}-1\bigr)-\log h}.
\label{eq:T_upper_supp}
\end{align}
Combining \eqref{eq:T_lower_eq} and \eqref{eq:T_upper_supp} yields the simultaneous constraints
\(
T\ \le\ \frac{2\tau}{\log\!\bigl(\tfrac{1}{\varepsilon}-1\bigr)-\log h},
\;
T\ \ge\ \frac{\gamma}{\log \kappa}.
\)
If
\begin{align}
\frac{\gamma}{\log \kappa}\ >\ \frac{2\tau}{\log\!\bigl(\tfrac{1}{\varepsilon}-1\bigr)-\log h},
\end{align}
no $T$ can satisfy both.

\emph{Instantiating NSF targets.}
Set $\varepsilon=\varepsilon_{\mathrm{nsf}}=e^{-\tau}/h$ and $\kappa=\kappa_{\mathrm{nsf}}=\frac{1+e^{-\tau}}{1+e^{-(\tau+\gamma)}}$. Then
\begin{align}
\log\!\Bigl(\tfrac{1}{\varepsilon_{\mathrm{nsf}}}-1\Bigr)-\log h
&= \log\!\Bigl(\frac{1}{e^{-\tau}/h}-1\Bigr)-\log h
= \log\!\bigl(h e^{\tau}-1\bigr)-\log h
= \log\!\bigl(e^{\tau}-h^{-1}\bigr),
\end{align}
so the right side of the incompatibility condition equals
\begin{align}
\frac{2\tau}{\log\!\bigl(e^{\tau}-h^{-1}\bigr)}\ \xrightarrow[\tau\to\infty]{}\ 2.
\end{align}
Meanwhile,
\begin{align}
\log \kappa_{\mathrm{nsf}}
&= \log(1+e^{-\tau})-\log\bigl(1+e^{-(\tau+\gamma)}\bigr) \\
&= \log\!\Bigl(1+\frac{e^{-\tau}\,(1-e^{-\gamma})}{1+e^{-(\tau+\gamma)}}\Bigr)
\ \sim\ e^{-\tau}\,(1-e^{-\gamma})\quad(\tau\to\infty),
\end{align}
hence
\begin{align}
\frac{\gamma}{\log \kappa_{\mathrm{nsf}}}\ \xrightarrow[\tau\to\infty]{}\ \infty.
\end{align}
Therefore, for any fixed $\gamma>0$, the incompatibility condition holds for all sufficiently large $\tau$, so no single softmax temperature can match NSF uniformly on $\mathcal{S}(\tau,\gamma,\mathcal{H},\mathcal{L})$.
\end{proof}

\begin{remark}[Middle scores]
Allowing additional scores in $(-\tau,\tau)$ only strengthens the NSF suppression bound because the denominator $S$ increases, and it does not weaken the softmax lower bound \eqref{eq:T_lower_eq} on the high to high ratio since that ratio is independent of other coordinates. The softmax low suppression supremum \eqref{eq:low_sup_full} is still attained by driving all non-high and non-$j$ scores to $-\infty$, so the temperature constraints remain necessary.
\end{remark}

\begin{figure}[ht]
\centering
\includegraphics[width=\linewidth]{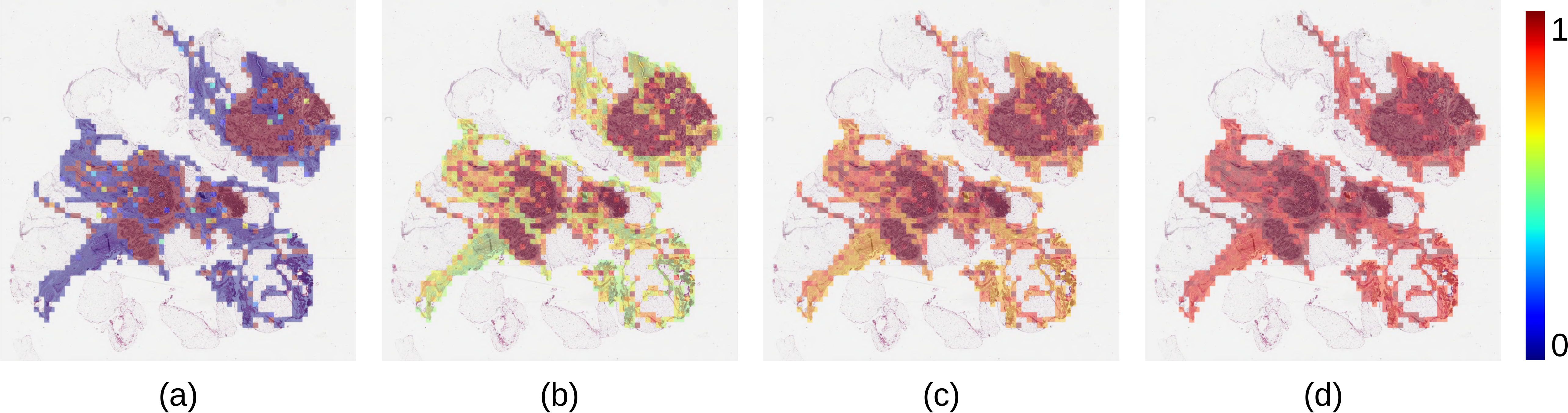}
\vspace{-0.8cm}
\caption{Ablation study of applying the softmax function with temperature scaling to the attention scores: (a) attention distribution of the proposed ASMIL, (b) softmax with  $T=2$ applied to the anchor model, (c)  softmax with $T=4$ applied to the anchor model, (d)  softmax with $T=8$ applied to the anchor model. }
\label{fig:AblationSoftMaXTemp}
\end{figure}

\section{Alternative to NSF in Anchor Model}
\label{appendix:NSFAlter} 
\subsection{Softmax With Temperature Scaling}
\label{appendix:anchorTemperature} 
A straightforward approach to mitigating over-concentration is to apply softmax with temperature scaling \citep{hinton2015distilling, ye2024bayesconditionaldistributionestimation, yang2309conditional, 10619241, 10900607}. This can indeed yield less concentrated attention distribution; however, as we observe in this section, a large temperature produces an overly smooth distribution, approaching a uniform distribution. This makes all tiles nearly indistinguishable, effectively reducing the operation to mean pooling and compromising interpretability.
To illustrate this, we conduct experiments on the BRACS dataset using the same training protocol as in \Cref{Sec:Experiments}, summarize the results in \Cref{Tab:softmax_anchor_results}, and visualize the attention maps in \Cref{fig:AblationSoftMaXTemp}. 
\begin{figure}[ht]
\includegraphics[width=\linewidth]{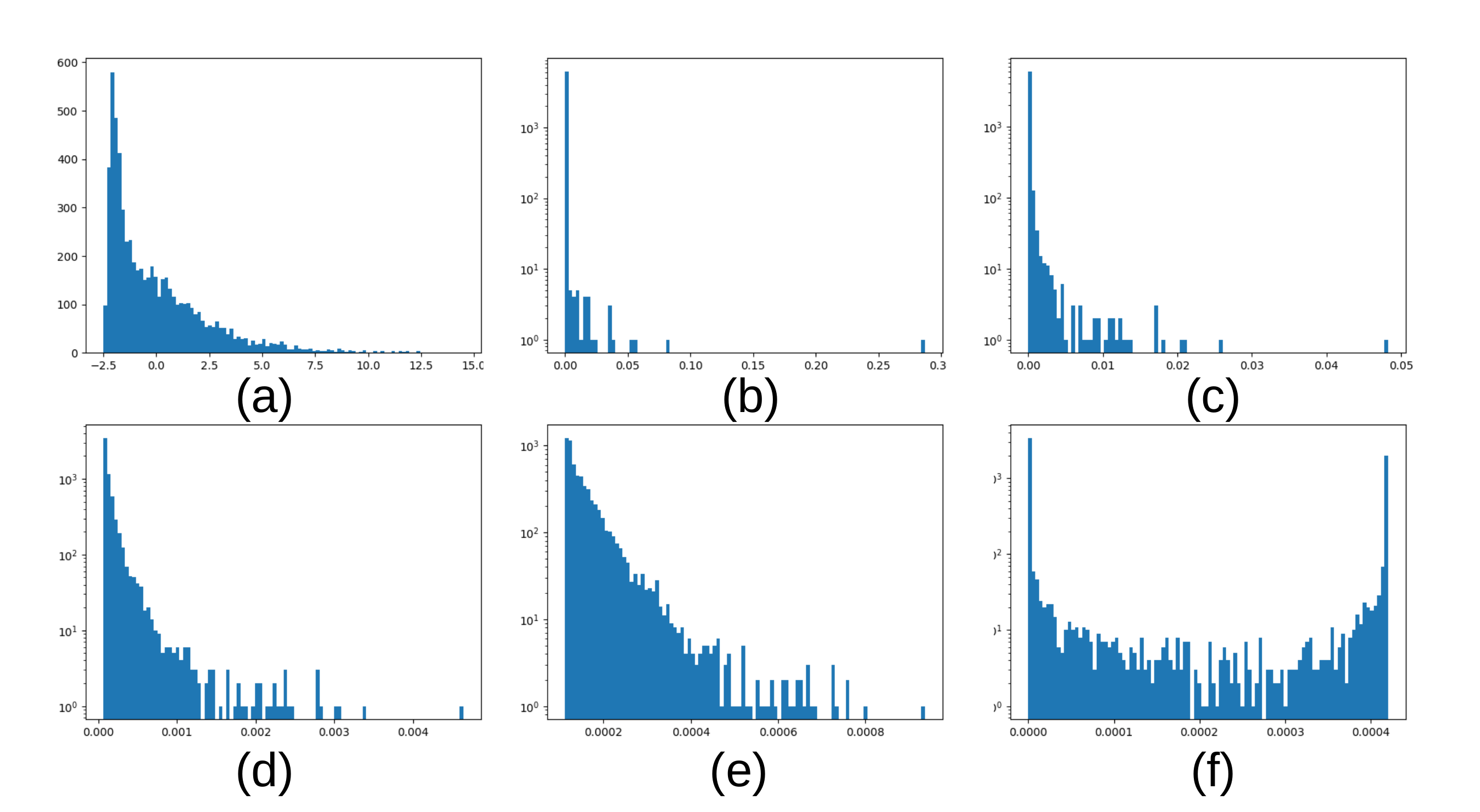}
\vspace{-0.8cm}
\caption{Histograms of (a) raw attention scores, (b) attention distribution obtained by the softmax function with temperature $T=1$, (c) $T=2$, (d) $T=4$, (e) $T=8$, and \textbf{(f) attention distribution computed using an NSF}. The Y-axis is displayed on a logarithmic scale for better visualization.}
\label{fig:softmax_sigmax_mapping}
\end{figure}

\begin{table}[ht]
\centering
\caption{Subtyping performance on BRACS, when we apply softmax with temperature scaling to the anchor model.}
\label{Tab:softmax_anchor_results}
\resizebox{\textwidth}{!}{\begin{tabular}{cccccccc}
\toprule
\rowcolor{lightgray} \multicolumn{8}{c}{BRACS} \\ \hline
\multicolumn{2}{c}{\begin{tabular}[c]{@{}c@{}}Normalized\\ Sigmoid\end{tabular}} & \multicolumn{2}{c}{\begin{tabular}[c]{@{}c@{}}Softmax\\ T=2\end{tabular}} & \multicolumn{2}{c}{\begin{tabular}[c]{@{}c@{}}Softmax\\ T=4\end{tabular}} & \multicolumn{2}{c}{\begin{tabular}[c]{@{}c@{}}Softmax\\ T=8\end{tabular}} \\ \hline
F1 score $\uparrow$ & AUC $\uparrow$ & F1 score $\uparrow$& AUC $\uparrow$ & F1 score $\uparrow$& AUC $\uparrow$ & F1 score $\uparrow$& AUC $\uparrow$ \\ \hline
$0.781\scriptscriptstyle\pm 0.042$ & $0.914\scriptscriptstyle\pm 0.014$ & $0.667\scriptscriptstyle\pm 0.049$ & $0.860\scriptscriptstyle\pm 0.027$ & $0.712\scriptscriptstyle\pm 0.029$ & $0.876\scriptscriptstyle\pm 0.012$ & $0.688\scriptscriptstyle\pm 0.037$ & $0.858\scriptscriptstyle\pm 0.031$ \\ \hline
\end{tabular}}
\end{table}

Furthermore, to clarify the differences between the NSF and softmax, we plot the histograms of the attention scores—(a) outputs from the NSF and softmax with various temperature scalings—in \Cref{fig:softmax_sigmax_mapping}. As shown, the saturation property of the NSF suppresses excessively large values.

\subsection{Entmax}

\begin{table}[ht]
\centering
\caption{ASMIL performance when replacing NSF with entmax on CAMELYON-16.}
\label{Tab:sparsemax_res}
\resizebox{\textwidth}{!}{\begin{tabular}{ccclclc}
\toprule
\rowcolor{lightgray} \multicolumn{7}{c}{CAMELYON-16} \\ \hline
\multicolumn{1}{c|}{Metric} & NSF & $\text{entmax}_{\alpha= 2}$ (sparsemax) & \multicolumn{1}{c}{$\text{entmax}_{\alpha= 1.75}$} & $\text{entmax}_{\alpha= 1.5}$ (entmax-15) & \multicolumn{1}{c}{$\text{entmax}_{\alpha= 1.25}$} & $\text{entmax}_{\alpha= 1}$ (softmax) \\ \hline
\multicolumn{1}{c|}{F1 score $\uparrow$} & $0.965\scriptscriptstyle\pm 0.020$ & $0.938\scriptscriptstyle\pm 0.031$ & $0.927\scriptscriptstyle\pm 0.034$ & $0.937\scriptscriptstyle\pm 0.014$ & $0.910\scriptscriptstyle\pm 0.026$ & $0.942\scriptscriptstyle\pm 0.0147$ \\
\multicolumn{1}{c|}{AUC $\uparrow$} & $0.985\scriptscriptstyle\pm 0.017$ & $0.964\scriptscriptstyle\pm 0.012$ & $0.959\scriptscriptstyle\pm 0.017$ & $0.960\scriptscriptstyle\pm 0.017$ & $0.925\scriptscriptstyle\pm 0.031$ & $0.963\scriptscriptstyle\pm 0.020$ \\
\multicolumn{1}{c|}{Time per epoch $\downarrow$} & 6.340s & 8.451s & 8.452s & 8.451s & 8.450s & 6.336s\\ \hline
\end{tabular}}
\end{table}
\begin{table}[ht]
\centering
\caption{Ablation study on the impact of applying the normalized sigmoid (NS) function to both the online and anchor models. \ding{51} indicates that NSF is applied to both models, while \ding{55} denotes the default setting where NSF is applied only to the anchor model. Subtyping performance is evaluated on the CAMELYON-16 and BRACS datasets using F1 score and AUC. A significant performance drop is observed on CAMELYON-16 when NSF is applied to both models.}
\label{Tab:applyNSoveronline}
\begin{tabular}{c|cc}
\toprule
\rowcolor{lightgray} Dataset   & \multicolumn{2}{c}{CAMELYON-16}                                         \\ \hline
Online NSF       & F1 score $\uparrow$& AUC $\uparrow$ \\ \hline
\ding{51} & $0.920\scriptscriptstyle\pm0.020 $ & $0.936\scriptscriptstyle\pm 0.021$ \\
\ding{55} & $0.965\scriptscriptstyle\pm 0.020$ & $0.985\scriptscriptstyle\pm 0.017$ \\ \toprule
\rowcolor{lightgray} Dataset   & \multicolumn{2}{c}{BRACS}                                               \\ \hline
Online NSF       & F1 score $\uparrow$& AUC $\uparrow$\\ \hline
\ding{51} & $0.726\scriptscriptstyle\pm 0.014$ & $0.865\scriptscriptstyle\pm 0.017$ \\
\ding{55} & $0.781\scriptscriptstyle\pm 0.042$ & $0.914\scriptscriptstyle\pm 0.014$ \\ \hline
\end{tabular}
\end{table}

Entmax is a family of mappings that convert a score vector $z\in \mathbb{R}^d$ into a probability vector $p \in \Delta^d$ by maximizing a linear score plus Tsallis-$\alpha$ entropy \cite{tsallis1988possible} $H^T_\alpha$: 
\begin{align}
    \text{entmax}_{\alpha}(z) = {argmax}_{p\in\Delta^d} z^T p + H^T_\alpha(p),
\end{align}
The solution admits a closed form
\begin{align}
    \boldsymbol{\alpha}_i = \Big[ \frac{\alpha-1}{\alpha} (z_i - \tau) \Big]_+^{\frac{1}{{\alpha-1}}}, \text{with }\sum_i \boldsymbol{\alpha}_i = 1, \label{eq:entmaxeq}
\end{align}
where $\tau$ is a threshold chosen so that the probabilities sum to one. As limiting cases, $\alpha \rightarrow 1$, yields softmax, and while $\alpha = 2$ yields sparsemax \citep{pmlr-v48-martins16}.

While entmax offers controllable sparsity, two drawbacks are pertinent to MIL-based WSI analysis: \textbf{($i$) Lack of selective flattening},  entmax is monotone in $z$ on its active support and does not explicitly equalize top-probability entries.
\textbf{($ii$) Higher computational cost}. Computing $\tau$ in \Cref{eq:entmaxeq} requires the bisection method, which adds non-trivial overhead vs. NSF’s fully closed-form normalization. 
These differences matter for MIL on WSIs, where multiple correlated tumor foci can be present: we prefer mechanisms that both discourage over-peaky attention and keep computation predictable. We replaced NSF with $entmax_{\alpha}$ inside ASMIL and swept $\alpha \in \{2,1.75,1.5,1.25,1\}$. For $\alpha=1$,  we used PyTorch softmax; for $\alpha>1$, we solved for $\tau$ via bisection. The implementation follows the reference code from DeepSPIN\footnote{\url{https://github.com/deep-spin/entmax}}. All other hyperparameters, model, and data pipeline were kept fixed. We report results on CAMELYON-16 in \Cref{Tab:sparsemax_res}. 
As seen, across $\alpha$, entmax underperforms NSF on both F1 and AUC and incurs a $33.5\%$ increase in epoch time vs. NSF.

\section{Applying Normalized Sigmoid to the Online Model}
\label{appendix:onlineNormalizedSigmoid} 

One might question the rationale behind applying the NSF to the anchor model while using the softmax function for the online model during training. To investigate this design choice, we experiment with applying the NSF to both the online and anchor models and evaluate the model's subtyping performance on the CAMELYON-16 and BRACS datasets. The results, presented in \Cref{Tab:applyNSoveronline}, reveal a F1 score drop of over 6\% on the BRACS dataset. We attribute this degradation to the inherent characteristics of the sigmoid function: when it saturates, its gradients diminish, leading to vanishing gradients in the attention mechanism and thereby impairing the learning process.

\begin{table}[]
\centering
{
\caption{Comparison between ASMIL trained with the standard softmax function in the online model (ASMIL w. Softmax) and with the mixed attention function defined in \Cref{Eq:MixSigma} (ASMIL w. Mixture). The more flexible trainable mapping does not yield improvements over the simpler softmax baseline.}
\label{Tab:MixNSFSMX}
\resizebox{\textwidth}{!}{\begin{tabular}{c|cc|cc|cc}
\toprule
\rowcolor{lightgray} Dataset & \multicolumn{2}{c|}{CAMELYON-16}                                      & \multicolumn{2}{c|}{CAMELYON-17}                                       & \multicolumn{2}{c}{BRACS}                                             \\ \midrule
Mertic                                        & F1 score                          & AUC                               & F1 score                          & AUC                                & F1 score                          & AUC                               \\ \hline
ASMIL W. SoftMax                              & $0.965\scriptscriptstyle\pm0.020$ & $0.985\scriptscriptstyle\pm0.017$ & $0.689\scriptscriptstyle\pm0.005$ & $0.898\scriptscriptstyle\pm0.010$  & $0.781\scriptscriptstyle\pm0.042$ & $0.914\scriptscriptstyle\pm0.014$ \\
ASMIL W. Mixture                              & $0.953\scriptscriptstyle\pm0.023$ & $0.972\scriptscriptstyle\pm0.030$ & $0.686\scriptscriptstyle\pm0.012$ & $0.889\scriptscriptstyle\pm0.009 $ & $0.774\scriptscriptstyle\pm0.054$ & $0.910\scriptscriptstyle\pm0.067$ \\ \hline
$\zeta$  in \Cref{Eq:MixSigma}   & \multicolumn{2}{c|}{0.9952} & \multicolumn{2}{c|}{0.9894} & \multicolumn{2}{c}{0.9963} \\ \bottomrule
\end{tabular}}}
\end{table}

{
To further investigate the potential of applying NSF in the online model, we consider the following mixed attention variant:
\begin{align}
    \alpha_i'(z) = \zeta\, \alpha_i^{\text{SMX}}(z) + (1-\zeta)\, \alpha_i^{\text{NSF}}(z), \label{Eq:MixSigma}
\end{align}
where $\zeta = \sigma(\xi)$ and $\xi$ is a trainable scalar that balances the contributions of the softmax and NSF mappings, initialized with $\xi = 0$. We evaluate this variant on CAMELYON-16, CAMELYON-17, and BRACS, and report the results in \Cref{Tab:MixNSFSMX}. The mixed mapping does not outperform the default softmax, and the learned $\zeta$ consistently converges to values close to one, indicating that the online model prefers softmax, which does not suffer from gradient-vanishing issues.
}

\section{Alternative Stabilization Methods and Why the Anchor is Preferable}
\label{app:stability_alternatives}

Let \(\boldsymbol{\alpha}_t(x)\in\Delta^{N}\) denote the attention distribution for slide \(x\) at epoch \(t\), obtained from scores \(\boldsymbol{z}_t(x)\in\mathbb{R}^{N}\).
We diagnose instability by the Jensen-Shannon divergence
\begin{align}
\mathrm{JSD}_t(x)\;=\;\mathrm{JSD}\!\bigl(\boldsymbol{\alpha}_t(x)\|\boldsymbol{\alpha}_{t-1}(x)\bigr),
\end{align}
which we empirically find remains high when training attention-based MIL with only bag-level labels. We present a natural alternative that targets this instability and explain why the anchor model is preferred.

\subsection{Alternative: Per-slide temporal ensembling of attention}
\label{app:temporal_ema}
Maintain a per slide exponential moving average (EMA) of past attentions and penalize deviation from it:
\begin{align}
\tilde{\boldsymbol{\alpha}}_t(x)\;=\;\rho\,\tilde{\boldsymbol{\alpha}}_{t-1}(x)\;+\;(1-\rho)\,\boldsymbol{\alpha}_t(x), ~\rho\in(0,1);
\qquad
\mathcal{L}_{\text{AS}}(x)\;=\mathsf{KL}\!\Bigl(\boldsymbol{\alpha}_t(x)\|\texttt{sg}\bigl(\tilde{\boldsymbol{\alpha}}_t(x)\bigr)\Bigr).
\end{align}
% \textit{Why it can help.} 
The EMA target changes slowly when \(\rho\) is close to one, which directly shrinks epoch-to-epoch drift of \(\boldsymbol{\alpha}_t\) and reduces \(\mathsf{JSD}(\boldsymbol{\alpha}_t\| \boldsymbol{\alpha}_{t-1})\).
\textit{However,} 

($i$) It has to maintain a length-\(N\) vector per slide. For \(S\) slides and average \(\bar{N}\) tiles, memory is \(O(S\bar{N})\) floats, which can be substantial for gigapixel WSIs and prevent scaling to larger datasets.
($ii$) The EMA target still uses softmax normalization, which cannot achieve selective flattening across informative tokens; see Theorem~\ref{thm:nsf_vs_smx}.

\subsection{Why ASMIL’s anchor is preferable}
\label{app:why_anchor_better}

We highlight two main reasons for using an anchor model to stabilize the attention distribution rather than relying on temporal ensembling.

\textbf{NSF provides selective flattening that softmax cannot match.}

Replacing softmax with the normalized sigmoid function (NSF) in the anchor yields \(\boldsymbol{\alpha}^{\mathrm{nsf}}(x)\), which equalizes probabilities among truly high-score tiles while suppressing low-score ones. By Theorem~\ref{thm:nsf_vs_smx}, no single softmax temperature can realize both behaviors across a broad class of score vectors. Consequently, methods that retain softmax-based targets inherit these limitations.

\textbf{Memory and implementation simplicity.}

The anchor-based approach adds only one extra forward pass and maintains an exponential moving average (EMA) of the anchor parameters during training. It does not require storing per-slide attention distributions, making the approach scalable to large WSI datasets.

Thus, an anchor model is preferable for scalable training on large MIL datasets and for preventing attention over-concentration.

\section{Why Matching the Teacher (Anchor) model's Softmax Feature Vector Cannot Stabilize the Attention distribution}
\label{appendix:MHIMNotGood}

\begin{table}[ht]
\centering
\caption{Ratio of affinely dependent feature bags in the CAMELYON-16, CAMELYON-17, and BRACS datasets; most bags are affinely dependent.}
\label{Tab:VerAffineDep}
\begin{tabular}{c|ccc}
\toprule
\rowcolor{lightgray} Dataset                                                                 & CAMELYON-16 & CAMELYON-17 & BRACS   \\ \hline
\begin{tabular}[c]{@{}c@{}}The ratio of affine \\ dependent feature bags\end{tabular} & 99.24\%     & 99.80\%     & 96.08\% \\ \hline
\end{tabular}
\end{table}

In this section, we show why matching the softmax of the bag-level feature is a suboptimal strategy for stabilizing attention distributions. To this end, we prove that recovering the attention vector \(\boldsymbol{\alpha}\) by matching
\(\operatorname{softmax}(\alpha^T X)\) is, in general, ill-posed: the map \(f:\Delta^{K}\to\Delta^{d}\), defined by \(f(\alpha)=\operatorname{softmax}(\alpha^T X)\)
with \(X\in\mathbb{R}^{K\times d}\), fails to be injective when the feature matrix \(X\) is affinely dependent.

\begin{proof}
Assume the rows \(x_1,\dots,x_K\in\mathbb{R}^d\) of \(X\) are affinely dependent. By definition there exists a nonzero vector \(\psi\in\mathbb{R}^K\) such that
\[
\sum_{i=1}^K \psi_i = 0 \quad\text{and}\quad \sum_{i=1}^K \psi_i x_i = 0.
\]
Let \(\alpha\in\Delta^K\) be any probability vector and choose \(\epsilon>0\) small enough that \(\alpha'=\alpha+\epsilon\psi\) satisfies \(\alpha'_i\ge0\) for every \(i\). Note \(\sum_i \alpha'_i = \sum_i \alpha_i + \epsilon\sum_i\psi_i = 1\), so \(\alpha'\in\Delta^K\). Since \(\sum_{i=1}^K \psi_i x_i=0\) we have
\[
(\alpha')^T X = \alpha^T X + \epsilon \psi^T X = \alpha^T X.
\]
Therefore
\[
f(\alpha')=\operatorname{softmax}((\alpha')^T X)=\operatorname{softmax}(\alpha^T X)=f(\alpha).
\]
Because \(\psi\neq 0\) and \(\epsilon\neq0\) we have \(\alpha'\neq\alpha\), hence \(f\) is not injective.
\end{proof}

Thus, matching the softmax of the bag feature cannot reliably recover or stabilize the attention distributions when the feature bag is affinely dependent. \Cref{Tab:VerAffineDep} confirms that most feature bags extracted by VIT-S \citep{kang2022benchmarking} from WSI datasets are indeed affinely dependent.

\section{Applying ASMIL to Features Extracted by a WSI Foundation Model}
\label{Appendix:BetterFeatureExtractor}
In recent years, foundation models have enabled strong open-source feature extractors that markedly improve the performance of computational-pathology systems. To assess the generalizability of our approach, we apply ASMIL to features produced by two such extractors, UNI \cite{chen2024uni} and PATHGEN-clip \cite{sun2025pathgenm}, for the subtyping task on the CAMELYON-16 and CAMELYON-17 datasets. As reported in \Cref{Tab:FMresults}, ASMIL consistently outperforms all baseline methods when used with features extracted by foundation models, yielding the best F1 and AUC.

\begin{table}[]
\centering
\caption{The F1 score and AUC of different MIL approaches on two WSI subtyping datasets.}
% We use \textbf{bold} number and asterisk (*) to denote the best and second best results, respectively.}
\label{Tab:FMresults}
\begin{tabular}{ccccc}
\toprule
\rowcolor{lightgray} \multicolumn{5}{c}{PathGen-Clip-VIT-L} \\ \hline
 \multicolumn{1}{c|}{Dataset}  & \multicolumn{2}{c|}{CAMELYON-16}               & \multicolumn{2}{c}{CAMELYON-17} \\ \hline
\multicolumn{1}{c|}{Method}   & F1 score  $\uparrow$      & \multicolumn{1}{c|}{AUC $\uparrow$}         & F1 score $\uparrow$          & AUC $\uparrow$           \\ \hline
\multicolumn{1}{c|}{Clam-SB}  & $0.941\scriptscriptstyle\pm0.014$ & \multicolumn{1}{c|}{$0.960\scriptscriptstyle\pm0.015$} & $0.622\scriptscriptstyle\pm0.031$    & $0.899\scriptscriptstyle\pm0.012$    \\
\multicolumn{1}{c|}{TransMIL} & $0.951\scriptscriptstyle\pm0.024$ & \multicolumn{1}{c|}{$0.968\scriptscriptstyle\pm0.028$} & $0.656\scriptscriptstyle\pm0.021$    & $0.892\scriptscriptstyle\pm0.014$    \\
\multicolumn{1}{c|}{DSMIL}    & $0.895\scriptscriptstyle\pm0.038$ & \multicolumn{1}{c|}{$0.949\scriptscriptstyle\pm0.017$} & $0.582\scriptscriptstyle\pm0.062$    & $0.887\scriptscriptstyle\pm0.013$    \\
\multicolumn{1}{c|}{IBMIL}    & $0.935\scriptscriptstyle\pm0.014$ & \multicolumn{1}{c|}{$0.953\scriptscriptstyle\pm0.009$} & $0.629\scriptscriptstyle\pm0.027$    & $0.884\scriptscriptstyle\pm0.016$    \\
\multicolumn{1}{c|}{MHIM-MIL} & $0.946\scriptscriptstyle\pm0.33$  & \multicolumn{1}{c|}{$0.984\scriptscriptstyle\pm0.016$} & $0.594\scriptscriptstyle\pm0.090$    & $0.912\scriptscriptstyle\pm0.009$    \\
\multicolumn{1}{c|}{ABMIL}    & $0.953\scriptscriptstyle\pm0.018$ & \multicolumn{1}{c|}{$0.972\scriptscriptstyle\pm0.010$} & $0.610\scriptscriptstyle\pm0.025$    & $0.864\scriptscriptstyle\pm0.017$    \\
\multicolumn{1}{c|}{AEM}      & $0.967\scriptscriptstyle\pm0.025$ & \multicolumn{1}{c|}{$0.988\scriptscriptstyle\pm0.013$} & $0.688\scriptscriptstyle\pm0.016$    & $0.905\scriptscriptstyle\pm0.005$    \\
\multicolumn{1}{c|}{ASMIL}    & $0.974\scriptscriptstyle\pm0.021$ & \multicolumn{1}{c|}{$0.990\scriptscriptstyle\pm0.014$} & $0.699\scriptscriptstyle\pm0.020$    & $0.929\scriptscriptstyle\pm0.016$    \\ \toprule
\rowcolor{lightgray} \multicolumn{5}{c}{UNI-VIT-L} \\ \hline
\multicolumn{1}{c|}{Method}   & F1 score $\uparrow$       & \multicolumn{1}{c|}{AUC $\uparrow$}         & F1 score $\uparrow$          & AUC $\uparrow$           \\ \hline
% \multicolumn{1}{c|}{Clam-SB}  &             & \multicolumn{1}{c|}{}            &                &                \\
% \multicolumn{1}{c|}{TransMIL} &             & \multicolumn{1}{c|}{}            &                &                \\
% \multicolumn{1}{c|}{DSMIL}    &             & \multicolumn{1}{c|}{}            &                &                \\
% \multicolumn{1}{c|}{IBMIL}    &             & \multicolumn{1}{c|}{}            &                &                \\
% \multicolumn{1}{c|}{MHIM-MIL} &             & \multicolumn{1}{c|}{}            &                &                \\
\multicolumn{1}{c|}{ABMIL}    & $0.968\scriptscriptstyle\pm0.011$ & \multicolumn{1}{c|}{$0.996\scriptscriptstyle\pm0.003$} & $0.605\scriptscriptstyle\pm0.047$    & $0.885\scriptscriptstyle\pm0.015$    \\
\multicolumn{1}{c|}{AEM}      & $0.975\scriptscriptstyle\pm0.003$ & \multicolumn{1}{c|}{$0.998\scriptscriptstyle\pm0.003$} & $0.633\scriptscriptstyle\pm0.024$    & $0.863\scriptscriptstyle\pm0.017$    \\
\multicolumn{1}{c|}{ASMIL}    & $0.980\scriptscriptstyle\pm0.004$ & \multicolumn{1}{c|}{$0.998\scriptscriptstyle\pm0.002$} &$ 0.672\scriptscriptstyle\pm0.035$    & $0.866\scriptscriptstyle\pm0.014$    \\ \hline
\end{tabular}
\end{table}

% \section{\textbf{Survival }Prediction}
% \label{appendix:Survival}

\section{Ablation study}
\label{Appendix:ablation}

\subsection{Ablation of the Coefficient \texorpdfstring{$\beta$}{beta}}
\label{Appendix:AblationBeta}
The coefficient $\beta>0$ in \Cref{eq:totalloss} balances the stabilization and classification objectives. To assess its impact on final performance, we sweep $\beta \in \{0, 0.1, 0.25, 0.5, 0.75, 1.0, 1.5, 2.0, 2.5, 4, 5\}$ on the CAMELYON-16 and BRACS datasets. Except for $\beta$, all experimental settings are identical to those in \Cref{Sec:SubtypingPerformance}. We report F1 score and AUC in \Cref{fig:AblationBeta,fig:AblationBeta_brasc}; results are averaged over five random seeds.
\begin{figure}[ht]
\centering
\includegraphics[]{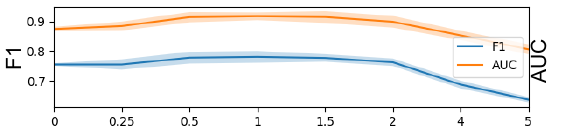}
% \vspace{-0.5cm}
\caption{Ablation study on the coefficient $\beta$, on CAMELYON-16 dataset. }
\label{fig:AblationBeta}
\end{figure}
\begin{figure}[ht]
\centering
\includegraphics[]{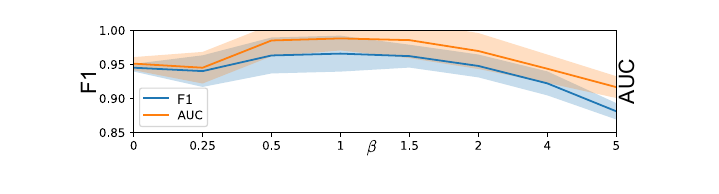}
% \vspace{-0.5cm}
\caption{Ablation study on the coefficient $\beta$, on BRACS dataset. }
\label{fig:AblationBeta_brasc}
\end{figure}
Overall, model performance is relatively insensitive to the choice of $\beta$: both F1 score and AUC plateau for $\beta \in [0.5, 1.5]$. Accordingly, we set $\beta=1$ as the default for all experiments.

\subsection{Ablation study on Number of trainable FEAT tokens}

\begin{table}[]
\centering
\caption{Ablation results on the number of tokens on different WSI datasets.}
\label{Tab:AblationOnNTokens}
\begin{tabular}{ccccc}
\toprule
\rowcolor{lightgray} \multicolumn{5}{c}{CAMELYON-16} \\ \hline
\multicolumn{1}{c|}{\# FEAT tokens} & 2      & 4      & 8      & 16     \\ \hline
\multicolumn{1}{c|}{F1 score $\uparrow$} & $0.930\scriptscriptstyle\pm0.012$ & $0.946\scriptscriptstyle\pm0.009$ & $0.965\scriptscriptstyle\pm0.012$ & $0.960\scriptscriptstyle\pm0.006$ \\
\multicolumn{1}{c|}{AUC $\uparrow$} & $0.932\scriptscriptstyle\pm0.017$ & $0.973\scriptscriptstyle\pm0.011$ & $0.985\scriptscriptstyle\pm0.013$ & $0.981\scriptscriptstyle\pm0.009$ \\ \hline
\rowcolor{lightgray} \multicolumn{5}{c}{CAMELYON-17}                                        \\ \toprule
\multicolumn{1}{c|}{F1 score $\uparrow$} & $0.556\scriptscriptstyle\pm0.012$ & $0.610\scriptscriptstyle\pm0.009$ & $0.674\scriptscriptstyle\pm0.016$ & $0.689\scriptscriptstyle\pm0.005$ \\
\multicolumn{1}{c|}{AUC $\uparrow$} & $0.784\scriptscriptstyle\pm0.019$ & $0.833\scriptscriptstyle\pm0.011$ & $0.879\scriptscriptstyle\pm0.024$ & $0.898\scriptscriptstyle\pm0.010$ \\ \toprule
\rowcolor{lightgray} \multicolumn{5}{c}{BRACS} \\ \hline
\multicolumn{1}{c|}{F1 score $\uparrow$} & $0.721\scriptscriptstyle\pm0.009$ & $0.766\scriptscriptstyle\pm0.012$ & $0.781\scriptscriptstyle\pm0.004$ & $0.782\scriptscriptstyle\pm0.004$ \\
\multicolumn{1}{c|}{AUC $\uparrow$} & $0.871\scriptscriptstyle\pm0.004$ & $0.903\scriptscriptstyle\pm0.014$ & $0.914\scriptscriptstyle\pm0.004$ & $0.912\scriptscriptstyle\pm0.026$ \\ \hline
\end{tabular}
\end{table}

In this section, we investigate how varying the number of trainable tokens influences model performance. To this end, we sweep a number of trainable tokens in the range of $[2,4,8,16]$, and report the corresponding accuracy on  CAMELYON-16, CAMELYON-17, and BRACS in \Cref{Tab:AblationOnNTokens}
In the experiment, we apply 8 trainable tokens for CAMELYON-16 and BRACS, and 16 trainable tokens on the CAMELYON-17 dataset.

\subsection{Ablation on anchor model update}

\subsubsection{Effect of Anchor Model Update Frequency}
\label{Sec:AnchorUpdateFreq}
\begin{table}[!ht]
\centering
\caption{Ablation study on anchor model update frequency, where batch-wise updates consistently outperform epoch-wise updates in both F1 score and AUC on BRACS and CAMELYON-16.} \label{tab:anchor_update}
\begin{tabular}{c|cc}
\toprule
\rowcolor{lightgray} Dataset & \multicolumn{2}{c}{BRACS}       \\ \hline
Update  & F1 score $\uparrow$          & AUC $\uparrow$           \\
Epoch   & $0.742\scriptscriptstyle\pm 0.015$    & $0.871\scriptscriptstyle\pm 0.003$    \\
Batch   & $0.781\scriptscriptstyle\pm 0.042$    & $0.914\scriptscriptstyle\pm 0.014$    \\ \toprule
\rowcolor{lightgray} Dataset & \multicolumn{2}{c}{CAMELYON-16} \\ \hline
Update  & F1 score $\uparrow$          & AUC $\uparrow$           \\
Epoch   & $0.920\scriptscriptstyle\pm 0.020$    & $0.936\scriptscriptstyle\pm 0.021$    \\
Batch   & $0.965\scriptscriptstyle\pm 0.020$    & $0.984\scriptscriptstyle\pm 0.017$    \\ \hline
\end{tabular}
\end{table}

To assess the impact of anchor update frequency, we compare epoch-wise and batch-wise update strategies on BRACS and CAMELYON-16 (\Cref{tab:anchor_update}). The results show that batch-wise updates consistently deliver superior performance. On BRACS, batch-wise updates improve the F1 score by 3.9\% and the AUC by 4.9\%. On CAMELYON-16, the improvement is even more substantial, with the F1 score increasing by 4.9\% and the AUC by 5.1\%. These gains confirm that frequent updates enable the anchor model to provide a stable and closely aligned attention reference for the online model, leading to better performance.

\subsection{Impact of the Random Drop Rate}
\label{Sec:RandomDropRate}
\begin{figure}[ht]
\centering
\includegraphics[width=\linewidth]{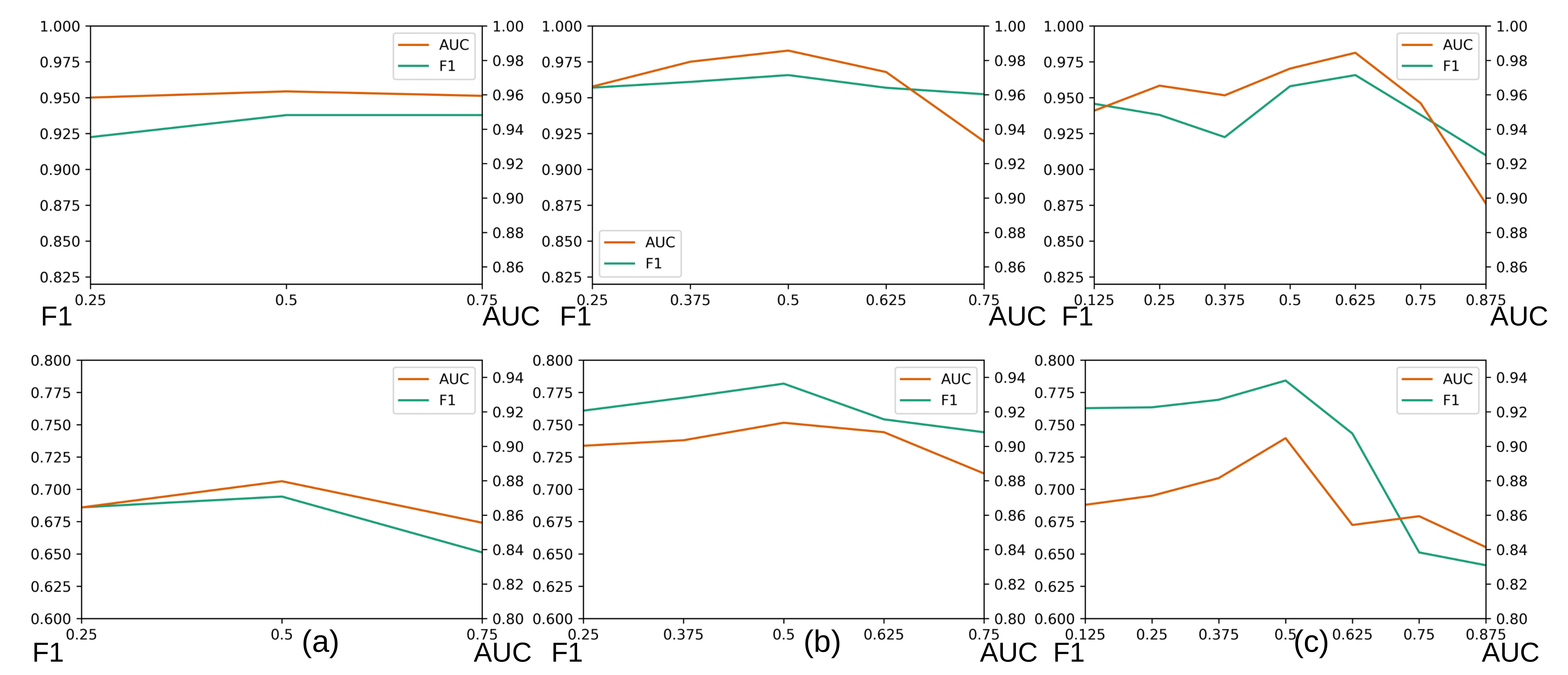}
\vspace{-0.7cm}
\caption{Ablation study of random drop probability ($B$) vs. model F1 score and AUC on CAMELYON-16 (top row) and BRACS (bottom row). Across both datasets and trainable-token settings (a) 4 tokens, (b) 8 tokens, and (c) 16 tokens, the test F1 score and AUC consistently peak around $B = 0.5$.}
\vspace{-0.2cm}
\label{fig:Ablation_drop_ratio}
\end{figure}

We evaluated the effect of random token dropping on model performance using CAMELYON-16 and BRACS, measuring both F1 and AUC across several trainable-token budgets. Results in \Cref{fig:Ablation_drop_ratio} show a consistent trend: performance rises from low $B$, peaks around $B=0.5$, then degrades for larger values. This pattern holds across datasets and capacities, indicating a stable trade-off between regularization and information loss.

Mechanistically, moderate token dropping (0.4–0.7) provides useful regularization, encouraging robustness to missing context and reducing overfitting to redundant or spurious tiles, while excessive dropping increases the chance of discarding diagnostically critical patches and thus harms recall and ranking. We therefore recommend tuning $B$ in the range of~ $0.4-0.6$. In \Cref{Appendix:MitigateOVefitting} we plot test F1 score and AUC across training epochs to demonstrate that random token dropping mitigates overfitting.

\begin{figure*}[ht]
  \centering 
  \includegraphics[width=0.48\linewidth]{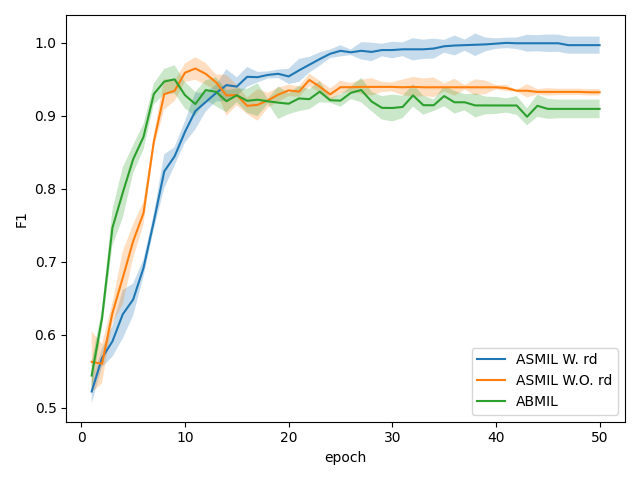}
 \includegraphics[width=0.48\linewidth]{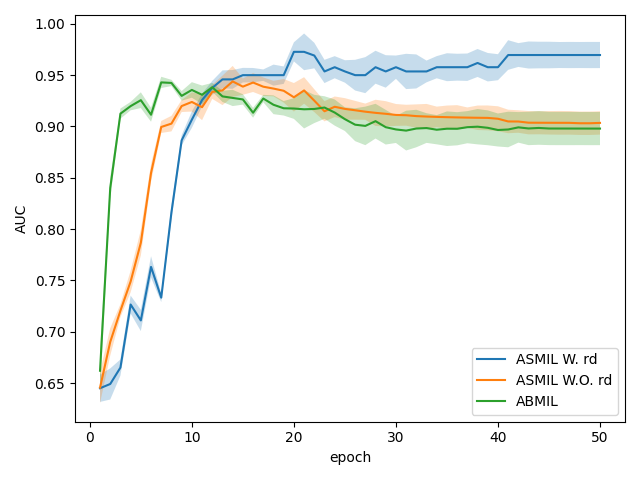}
 \vspace{-0.3cm}
  \caption{Performance comparison between ABMIL \citep{ilse2018attention}, ASMIL with random drop (ASMIL W. rd), and ASMIL without random drop (ASMIL w/o rd). Both ABMIL and ASMIL w/o rd show signs of overfitting, as their F1 score and AUC peak and then decline. In contrast, ASMIL with random drop maintains stable performance across training, demonstrating that random drop effectively mitigates overfitting.}
  \vspace{-0.5cm}
\label{fig:MitigateOVefitting}
\end{figure*}
\subsection{Random Drop Mitigates Overfitting}
\label{Appendix:MitigateOVefitting}
To verify that random drop is an efficient regularizer for attention-based MIL on WSIs, we trained three variants on \textsc{CAMELYON-16}: ($i$) ABMIL, ($ii$) ASMIL without random drop, and ($iii$) ASMIL with random drop (ours) with $B=0.5$. The figure reports validation F1 and AUC over training epochs. 

As shown in the \Cref{fig:MitigateOVefitting}, both ABMIL and ASMIL without random drop exhibit overfitting: F1 score and AUC rise early, peak, and then decline with continued training. In contrast, ASMIL with random drop maintains high and stable F1/AUC throughout later epochs, with noticeably reduced run-to-run variability (shaded regions). These trajectories empirically validate that random drop curbs the late-epoch degradation that accompanies weak supervision on CAMELYON-16. This observation aligns with our analysis that overfitting is a recurring failure mode for attention-based MIL on WSI datasets.

\begin{table}[t]
\centering
{
\caption{Statistical comparison of ASMIL with and without the anchor model.
We report the mean performance over $10$ random seeds along with p-values from
DeLong tests for AUC and permutation tests for F1.}
\label{tab:anchor_significance}
\begin{tabular}{l l c c c c}
\toprule
\rowcolor{lightgray} Dataset & Model & AUC  & F1  & $p_{\text{AUC}}$ & $p_{\text{F1}}$ \\
\midrule
\multirow{2}{*}{CAMELYON-16}
& w/o anchor & 0.942 & 0.979 & \multirow{2}{*}{0.013} & \multirow{2}{*}{0.024} \\
& w/ anchor  & 0.967 & 0.983 &                       &                     \\
\midrule
\multirow{2}{*}{CAMELYON-17}
& w/o anchor & 0.642 & 0.879 & \multirow{2}{*}{0.024} & \multirow{2}{*}{0.035}  \\
& w/ anchor   & 0.693 & 0.899 &                       &                     \\
\midrule
\multirow{2}{*}{BRACS}
& w/o anchor & 0.729 & 0.866 &  \multirow{2}{*}{0.012} & \multirow{2}{*}{0.009} \\
& w/ anchor  & 0.784 & 0.916 &                       &                     \\
\bottomrule
\end{tabular}}
\end{table}

{\subsection{Significance test on the effect of online model}
To assess whether the performance gains from the anchor model are statistically meaningful, we perform paired significance tests between ASMIL with and without the anchor over multiple 10 seeds. For AUC, we apply DeLong’s test, and for F1, we use a non-parametric permutation test. Across CAMELYON-16, CAMELYON-17, and BRACS, the anchor-augmented ASMIL consistently achieves higher AUC and F1 than its non-anchor counterpart, and these improvements are statistically significant ($p < 0.05$) for both metrics on each dataset (see Table~\ref{tab:anchor_significance}).
}

\section{Quantitative localization results and Additional Visualization}
\label{appendix:local_vis}

Predicted masks are generated as follows. For attention-based methods (CLAM \citep{lu2021data}, TransMIL \citep{shao2021transmil}, DTFD-MIL \citep{zhang2022dtfd}, DSMIL \citep{li2021dual} and CAMIL \cite{fourkioti2024camil}), we use the tile-level attention distribution. For ASMIL, the per-tile attention distribution is computed by averaging the attention distributions from all FEAT tokens to that tile. Unless otherwise noted, we rescale all per-tile scores to $[0,1]$ and threshold at $0.5$ to produce binary masks across all methods.

For tumor localization on CAMELYON-16, we follow the official challenge protocol and report the lesion-level Free-Response ROC (FROC) \citep{miller1969froc, bunch1978free, zhang2025widget2code}. Concretely, model outputs are converted to point detections; a detection is counted as a true positive if it lies within 75 µm of any ground-truth tumor region (implemented in the official script via a distance-transform “evaluation mask”), otherwise it is a false positive. We then sweep the detection score threshold to trace sensitivity versus the average number of false positives per normal WSI, and compute the standard CAMELYON-16 FROC score as the mean sensitivity at {0.25, 0.5, 1, 2, 4, 8} FP/WSI. 

Quantitative results for FROC, Dice, and specificity are reported in \Cref{tab:LocRes},  ASMIL achieves the best FROC and Dice on cancerous slides and higher specificity on normal slides, yielding fewer false positives and more contiguous lesion maps compared to baselines.

\begin{figure}[!ht]
\centering
\includegraphics[width=0.85\linewidth]{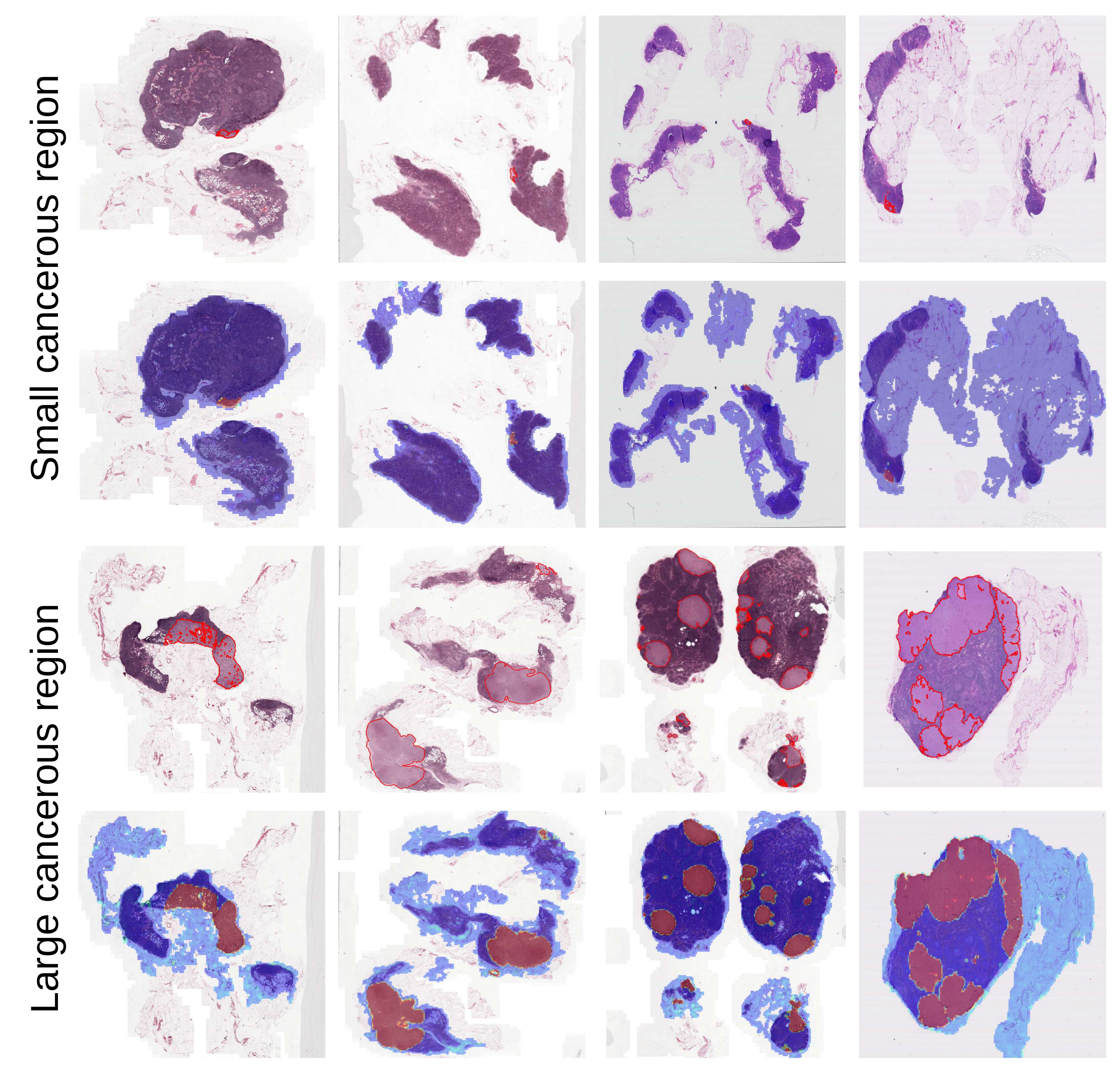}
\vspace{-0.5cm}
\caption{Additional qualitative examples of tumor regions and ASMIL attention maps. Rows 1 and 3 show the ground-truth tumor annotations (cancerous regions outlined in red), and rows 2 and 4 show the corresponding ASMIL attention maps.}
\label{fig:morelocal}
\end{figure}

\Cref{fig:morelocal} presents additional visualizations on the CAMELYON-16 dataset. It shows ASMIL attention maps for tumor slides containing both small and large cancerous regions; rows 1 and 3 provide the ground-truth annotations, and rows 2 and 4 show the corresponding attention maps.

% \begin{table}
% \centering
% \caption{Localization results on CAMELYON-16.} \label{tab:LocRes}
% \resizebox{0.4\textwidth}{!}{
% \begin{tabular}{cccc}
% \toprule
% \rowcolor{lightgray} Method   & Dice         & Specificity & FROC \\ \hline
% CLAM-SB  & $0.459\scriptscriptstyle\pm0.037$  & $0.987\scriptscriptstyle\pm0.008$ \\
% TransMIL & $0.103\scriptscriptstyle\pm0.004$  & $0.999\scriptscriptstyle\pm0.001$ \\
% DTFD-MIL & $0.525\scriptscriptstyle\pm0.033$ & $0.999\scriptscriptstyle\pm0.001$ \\
% DSMIL    & $0.259\scriptscriptstyle\pm0.083$  & $0.863\scriptscriptstyle\pm0.043$ \\
% CAMIL    & $0.515\scriptscriptstyle\pm0.083$  & $0.980\scriptscriptstyle\pm0.043$ \\
% ASMIL    & $0.586\scriptscriptstyle\pm0.045$  & $0.999\scriptscriptstyle\pm0.005$ \\ \hline
% \end{tabular}}
% \end{table}

\begin{table}
\centering
\caption{Localization results on CAMELYON-16.} \label{tab:LocRes}
\begin{tabular}{cccc}
\toprule
\rowcolor{lightgray} Method   & Dice $\uparrow$ & Specificity  $\uparrow$& FROC $\uparrow$ \\ \hline
CLAM-SB  & $0.459$  & $0.987$ &  $0.4257$ \\
TransMIL & $0.103$  & $0.999$ &  $0.0866$ \\
DTFD-MIL & $0.525$  & $0.999$ &  $0.4712$ \\
DSMIL    & $0.259$  & $0.863$ &  $0.4506$ \\
CAMIL    & $0.515$  & $0.980$ &  $0.4612$ \\
ASMIL    & $0.586$  & $0.999$ &  $0.4941$ \\ \hline
\end{tabular}
\end{table}

\section{Computational Cost}

This section reports the computational cost of ASMIL, as well as the additional cost incurred when integrating the anchor model into the baseline methods.

\subsection{Comparison of the computational cost between ASMIL and baseline methods}
\label{Appendix:CcostASMILVsOthers}
\begin{table}
\centering
\caption{Computational cost on BRACS (lower is better). We report training time and peak memory per epoch, and inference FLOPs, latency, and memory. ASMIL (ours) delivers efficient inference, cutting compute by $30.6\%$, latency by $29.2\%$, and memory by $20.3\%$ compared with TransMIL, while requiring $4\times$ less training memory than MHIM-MIL.}
\label{Tab:ASMILComputationalCost}
\begin{tabular}{c|ccccc}
\toprule
\rowcolor{lightgray} BRACS  & \multicolumn{5}{c}{Training}                   \\ \hline
Method & CLAM-SB & ABMIL & TransMIL & MHIM-MIL & ASMIL  \\ \hline
Time   & 2.26s   & 0.95s & 5.99s    & 19.4s    & 7.49s  \\
Memory & 94MB    & 90MB  & 340MB   & 2178MB   & 570MB \\ \toprule
\rowcolor{lightgray} BRACS  & \multicolumn{5}{c}{Inference}                  \\ \hline
FLOPs  & 162M    & 164M  & 781M     & 345M     & 542M   \\
Time   & 0.45s  & 0.37s & 0.74s   & 0.40s   & 0.52s \\
Memory & 69MB    & 39MB  & 246MB    & 61MB     & 196MB \\ \hline
\end{tabular}
\end{table}

\begin{table}[!ht]
\centering
\caption{Inference FLOPs, training time per epoch (Time), and memory usage (Memory) for four well-known methods, CLAM-SB, TransMIL, DSMIL, and ABMIL, with and without the anchor model. The anchor model incurs only minor computational overhead. FLOPs are measured using a fixed bag size of 2000 instances.}
\label{Tab:addtionaltimeAnchor}
\begin{tabular}{c|cc|ccc}
\toprule
\rowcolor{lightgray} BRACS & \multicolumn{2}{c|}{Training} & \multicolumn{3}{c}{Inference} \\ \hline
Method   & Time          & Memory        & FLOPs   & Time      & Memory   \\ \hline
\begin{tabular}[c]{@{}c@{}}CLAM-SB ~~~ w/o anchor\end{tabular}  & 2.26s      & 94MB       & 162M    & 0.45s   & 69MB     \\
\begin{tabular}[c]{@{}c@{}}CLAM-SB W. anchor\end{tabular}    & 2.69s      & 120MB      & 162M    & 0.45s   & 69MB     \\ \hline
\begin{tabular}[c]{@{}c@{}}TransMIL ~~~ w/o anchor\end{tabular} & 5.99s      & 340MB      & 781M    & 0.74s   & 246MB    \\
\begin{tabular}[c]{@{}c@{}}TransMIL W. anchor\end{tabular}   & 7.27s      & 443MB      & 781M    & 0.74s   & 246MB    \\ \hline
\begin{tabular}[c]{@{}c@{}}DSMIL ~~~  w/o anchor\end{tabular}    & 0.57s      & 60 MB       & 103M    & 1.09s   & 113 MB   \\
\begin{tabular}[c]{@{}c@{}}DSMIL W. anchor\end{tabular}      & 0.58s      & 145MB       & 103M    & 1.09s   & 113 MB   \\ \hline
\begin{tabular}[c]{@{}c@{}}ABMIL ~~~ w/o anchor\end{tabular}    & 0.95s      & 90MB       & 164M    & 0.37s   & 39MB     \\
\begin{tabular}[c]{@{}c@{}}ABMIL W. anchor\end{tabular}      & 1.17s      & 162MB      & 164M    & 0.37s   & 39MB     \\ \hline
\end{tabular}
\end{table}

We conducted a detailed evaluation of the computational overhead introduced by our proposed ASMIL framework, focusing on three primary metrics: floating-point operations (FLOPs), training time per epoch, and peak memory consumption. All experiments were executed under uniform hardware conditions, specifically a single NVIDIA RTX 5000 GPU coupled with an Intel Xeon W-2265 CPU and 64 GB of RAM, ensuring a fair comparison across methods.

During training, ASMIL demonstrates a competitive balance between efficiency and computational demand. On average, ASMIL requires 542M FLOPs per batch, which is lower than MHIM-MIL. The training time per epoch for ASMIL is 7.49s, substantially faster than MHIM-MIL (19.4s) and comparable to TransMIL (5.99s), while remaining higher than ABMIL and CLAM-SB. In terms of peak memory usage, ASMIL consumes 570 MB, markedly lower than MHIM-MIL (2178 MB). These results indicate that ASMIL maintains a favorable computational profile, offering a scalable alternative to more resource-intensive methods.

In inference, ASMIL continues to show strong efficiency. It requires 542M FLOPs, substantially fewer than TransMIL and comparable to MHIM-MIL. Inference time for ASMIL is 0.52s per epoch, slightly slower than CLAM-SB (0.45s) but faster than TransMIL. Peak memory usage during inference is 196 MB, markedly lower than TransMIL, highlighting ASMIL’s efficient memory footprint relative to its computational performance.
Overall, ASMIL delivers high-performance multiple-instance learning while keeping computational cost affordable.

\subsection{Additional Computational Cost Introduced by Anchor Model}
\label{Appendix:AdditionalComputationalCostIntroducedbyAnchorModel}

We conducted a detailed evaluation of the computational overhead introduced by integrating the anchor model into four widely used MIL methods, namely CLAM-SB, TransMIL, DSMIL, and ABMIL, all measured on the BRACS dataset. The results are summarized in \Cref{Tab:addtionaltimeAnchor}. 

Because no gradients are computed through the anchor model, and only the attention layer is updated during training, the computational overhead is small. As shown in \Cref{Tab:addtionaltimeAnchor}, integrating the anchor model into CLAM-SB, TransMIL, DSMIL, and ABMIL introduces only a modest increase in training time and memory usage, while the FLOPs remain unchanged. For example, training time for CLAM-SB increases from 2.26s to 2.69s and memory usage from 94 MB to 120 MB, with larger models like TransMIL showing slightly higher overhead. Importantly, during inference, the anchor model is discarded, resulting in identical FLOPs, execution time, and memory consumption compared to the baseline methods. These results demonstrate that the anchor model provides performance benefits during training with minimal computational cost and does not affect deployment efficiency, making it an effective and practical addition to existing MIL frameworks.

\begin{table}[]
\centering
{
\caption{C-index for WSI-based survival prediction using vision-only MIL models.}
\label{tab:SPR}
\begin{tabular}{cccccc}
\toprule
\rowcolor{lightgray}  Method&  BLCA         & BRCA         & GBMLGG       & LUAD         & UCEC \\ \midrule
ABMIL \textcolor{mygray}{\tiny ICML \citeyear{ilse2018attention}} & $0.5581\scriptscriptstyle\pm0.031$ & $0.5825\scriptscriptstyle\pm0.035$ & $0.7935\scriptscriptstyle\pm0.032$ & $0.6121\scriptscriptstyle\pm0.050$ & $0.6667\scriptscriptstyle\pm0.033$  \\
TransMIL \textcolor{mygray}{\tiny NeurIPS \citeyear{shao2021transmil}} & $0.5885\scriptscriptstyle\pm0.055$ & $0.6140\scriptscriptstyle\pm0.060$ & $0.7956\scriptscriptstyle\pm0.015$ & $0.5708\scriptscriptstyle\pm0.050$ & $0.6380\scriptscriptstyle\pm0.067$  \\
ILRA \textcolor{mygray}{\tiny ICLR \citeyear{xiang2023exploring}} & $0.5549\scriptscriptstyle\pm0.053$ & $0.5705\scriptscriptstyle\pm0.067$ & $0.7742\scriptscriptstyle\pm0.014$ & $0.5179\scriptscriptstyle\pm0.081$ & $0.6503\scriptscriptstyle\pm0.064$  \\
$\text{R}^2\text{T-MIL }$ \textcolor{mygray}{\tiny CVPR \citeyear{tang2024feature}} & $0.5775\scriptscriptstyle\pm0.024$ & $0.5476\scriptscriptstyle\pm0.095$ & $0.7757\scriptscriptstyle\pm0.024$ & $0.5711\scriptscriptstyle\pm0.076$ & $0.6510\scriptscriptstyle\pm0.087$  \\
DeepAttnMISL \textcolor{mygray}{\tiny MIA \citeyear{yao2020whole}}& $0.5646\scriptscriptstyle\pm0.035$ & $0.5346\scriptscriptstyle\pm0.036$ & $0.6750\scriptscriptstyle\pm0.048$ & $0.4678\scriptscriptstyle\pm0.039$ & $0.6259\scriptscriptstyle\pm0.086$  \\
Patch-GCN  \textcolor{mygray}{\tiny MICCAI \citeyear{chen2021whole}} & $0.6124\scriptscriptstyle\pm0.031$ & $0.6375\scriptscriptstyle\pm0.033$ & $0.7999\scriptscriptstyle\pm0.021$ & $0.5922\scriptscriptstyle\pm0.053$ & $0.7212\scriptscriptstyle\pm0.025$  \\
\texttt{ASMIL} (Ours)              & $0.6133\scriptscriptstyle\pm0.047$ & $0.6396\scriptscriptstyle\pm0.044$ & $0.8036\scriptscriptstyle\pm0.018$ & $0.6001\scriptscriptstyle\pm0.093$ & $0.7243\scriptscriptstyle\pm0.0488$ \\ \bottomrule
\end{tabular}}
\end{table}

\section{Survival Prediction}
\label{Appendix:SP}
{To assess whether ASMIL is also beneficial for prognosis, we extend ASMIL from slide-level classification to discrete-time overall survival prediction on histopathology WSIs. Following \citep{liu2025interpretable}, we apply an incidence-based discrete survival formulation, \textit{i.e.}, the survival times are mapped to $C$ non-overlapping time intervals, and the model outputs a discrete distribution over first-event times.}

{We follow the experimental setup of \citep{liu2025interpretable}, and evaluate on five TCGA datasets, namely BLCA, BRCA, LUAD, and UCEC. We use the concordance index (C-index) to evaluate the model's performance; specifically, it measures how often the model assigns a higher risk score to a patient who experiences the event earlier. Formally, with a little abuse of notations, let $t_i,\delta_i,\hat{R}_i$ denote the observed time, event indicator, and predicted risk for patient $i$, the C-index is defined as
\begin{align}
    \text{CI} = \frac{\sum_{i,j}\mathbf{1}[t_i<t_j]\mathbf{1}[\hat{R}_i>\hat{R}_j]\delta_i}{{\sum_{i,j}\mathbf{1}[t_i<t_j]\delta_i}},
\end{align}
where $\mathbf{1}[\cdot]$ is the indicator function. A value of $\text{CI}=0.5$ corresponds to a random ranking, and larger values indicate better risk discrimination \cite{yang2024markov, hamidi2024adversarialtrainingadaptiveknowledge, hamidi2025distributed}. Following \citet{liu2025interpretable}, we compare ASMIL against six vision-only WSI survival prediction methods, namely ABMIL \citep{ilse2018attention}, TransMIL \citep{shao2021transmil}, ILRA \cite{xiang2023exploring}, $\text{R}^2\text{T-MIL}$ \citep{tang2024feature}, DeepAttnMISL \citep{yao2020whole}, and Patch-GCN \citep{chen2021whole}, all implemented on top of the same CONCH-derived patch features \citep{Lu_2023_CVPR}.}

{\Cref{tab:SPR} reports the C-index on each TCGA dataset. ASMIL achieves the highest mean C-index among all vision-only baselines. These results indicate that stabilizing slide-level attention not only improves weakly supervised classification but also yields stronger prognostic discrimination in survival analysis.
}

\section{Evaluate ASMIL Over Non-WSI Dataset}
\label{Appendix:ASMILOtherDataset}
\begin{table}[h!]
\centering
\caption{MIL dataset statistics.}
\label{tab:mil-stats}
\begin{tabular}{l l r r r}
\toprule
\rowcolor{lightgray} Dataset & Domain & Bags (pos/neg) & Total instances & Dim./inst. \\
\midrule
MUSK1      & Drug activity               & 92  (47/45)   & 476   & 166 \\
MUSK2      & Drug activity               & 102 (39/63)   & 6598 & 166 \\
TIGER      & Images (Blobworld segments) & 200 (100/100) & 1220 & 230 \\
FOX        & Images (Blobworld segments) & 200 (100/100) & 1320 & 230 \\
ELEPHANT   & Images (Blobworld segments) & 200 (100/100) & 1391 & 230 \\
\bottomrule
\end{tabular}
\end{table}

To demonstrate ASMIL’s applicability beyond WSI, we evaluate it on five classic multiple-instance learning (MIL) benchmarks: \emph{MUSK1} \citet{musk_(version_1)_74} and \emph{MUSK2} \citet{musk_(version_2)_75}, where each bag is a molecule and instances are its low-energy 3D conformations described by 166 attributes (a bag is positive if at least one conformation is active); and the image MIL datasets \emph{TIGER}, \emph{FOX}, and \emph{ELEPHANT} \citet{NIPS2002_3e6260b8}, where each bag is a Corel image segmented into “Blobworld’’ regions (instances) with 230-D color/texture/shape features (a bag is positive if at least one segment contains the named animal). Standard size statistics are reported in \Cref{tab:mil-stats}.

\begin{table}[h!]
\centering
\caption{Results on the small MIL benchmark datasets.}
\label{Tab:MUCKresults}
\begin{tabular}{cccccc}
\toprule
\rowcolor{lightgray} Methods  & MUSK1 & MUSK2 & FOX & TIGER & ELEPHANT \\ \midrule
ABMIL \textcolor{mygray}{\tiny ICML \citeyear{ilse2018attention}}&$0.916\scriptscriptstyle\pm 0.118$&$0.928\scriptscriptstyle\pm 0.109$&  $0.952\scriptscriptstyle\pm 0.051$ & $0.953\scriptscriptstyle\pm 0.042$ &$0.969\scriptscriptstyle\pm 0.036$ \\
DSMIL \textcolor{mygray}{\tiny CVPR \citeyear{li2021dual}}& $0.959\scriptscriptstyle\pm 0.053$ & $0.952\scriptscriptstyle\pm 0.066$ & $0.939\scriptscriptstyle\pm 0.060$  & $0.951\scriptscriptstyle\pm 0.053$  & $\mathbf{0.989\scriptscriptstyle\pm 0.023}$ \\
TransMIL \textcolor{mygray}{\tiny NeurIPS \citeyear{shao2021transmil}} & $0.927\scriptscriptstyle\pm 0.093$ &  $0.877\scriptscriptstyle\pm 0.127$  & $0.944\scriptscriptstyle\pm 0.050$ & $0.963\scriptscriptstyle\pm 0.042$ & $0.979\scriptscriptstyle\pm 0.030$ \\
DEMIL \textcolor{mygray}{\tiny NeurIPS \citeyear{tang2023disambiguated}} & $0.963\scriptscriptstyle\pm 0.073$ & $0.961\scriptscriptstyle\pm 0.057$ & $0.941\scriptscriptstyle\pm 0.047$ & $0.965\scriptscriptstyle\pm 0.035$ & $0.969\scriptscriptstyle\pm 0.034$\\
RGMIL \textcolor{mygray}{\tiny Neurips\citeyear{du2023rgmil}}    & $0.968\scriptscriptstyle\pm 0.060$ & $0.963\scriptscriptstyle\pm 0.048$ & $0.954\scriptscriptstyle\pm 0.048$ & $0.949\scriptscriptstyle\pm 0.047$ & $0.965\scriptscriptstyle\pm 0.032$ \\
PSMIL  \textcolor{mygray}{\tiny ICLR\citeyear{du2025rethinking}}  & $0.968\scriptscriptstyle\pm 0.053$ & $\mathbf{0.968\scriptscriptstyle\pm 0.052}$ & $0.942\scriptscriptstyle\pm 0.054$ & $0.947\scriptscriptstyle\pm 0.047$ &$0.985\scriptscriptstyle\pm 0.030$ \\
\texttt{ASMIL} (Ours)   & $\mathbf{0.971\scriptscriptstyle\pm 0.060}$  & $0.968\scriptscriptstyle\pm 0.058$ & $\mathbf{0.961\scriptscriptstyle\pm 0.025}$ & $\mathbf{0.969\scriptscriptstyle\pm 0.037}$ & $0.985\scriptscriptstyle\pm 0.025$ \\ \hline
\end{tabular}
\end{table}

Since these datasets are relatively balanced, following \citet{du2025rethinking}, we report accuracy as the primary metric. We train for 40 epochs with the Adam optimizer \citep{kingma2014adam} and a learning rate of 0.0005.

We compare ASMIL against six MIL methods—ABMIL \citep{ilse2018attention}, DSMIL \citep{li2021dual}, TransMIL \citep{shao2021transmil}, DEMIL \citep{tang2023disambiguated}, RGMIL \citep{du2023rgmil}, and PSMIL \citep{du2025rethinking}—and report accuracies in \Cref{Tab:MUCKresults}. As shown, ASMIL outperforms all baselines on 4 of 5 datasets, demonstrating strong performance on non-WSI benchmarks.

\section{Attention Dynamics of different MIL methods on Various Datasets}
\label{Appendix:AttentionDynamics}
In this section, we illustrate that the issue of attention convergence on the WSI dataset is not unique to the ABMIL and CAMELYON-16 datasets. To this end, similar to the method we describe in \Cref{fig:ABMIL_fluctuation}, we plot the JSD of two attention distributions between two consecutive epochs. 
\clearpage
\subsection{CAMELYON-16 Dataset}
\begin{figure}[!ht]
\begin{center}
%\framebox[4.0in]{$\;$}
\includegraphics[width=\linewidth]{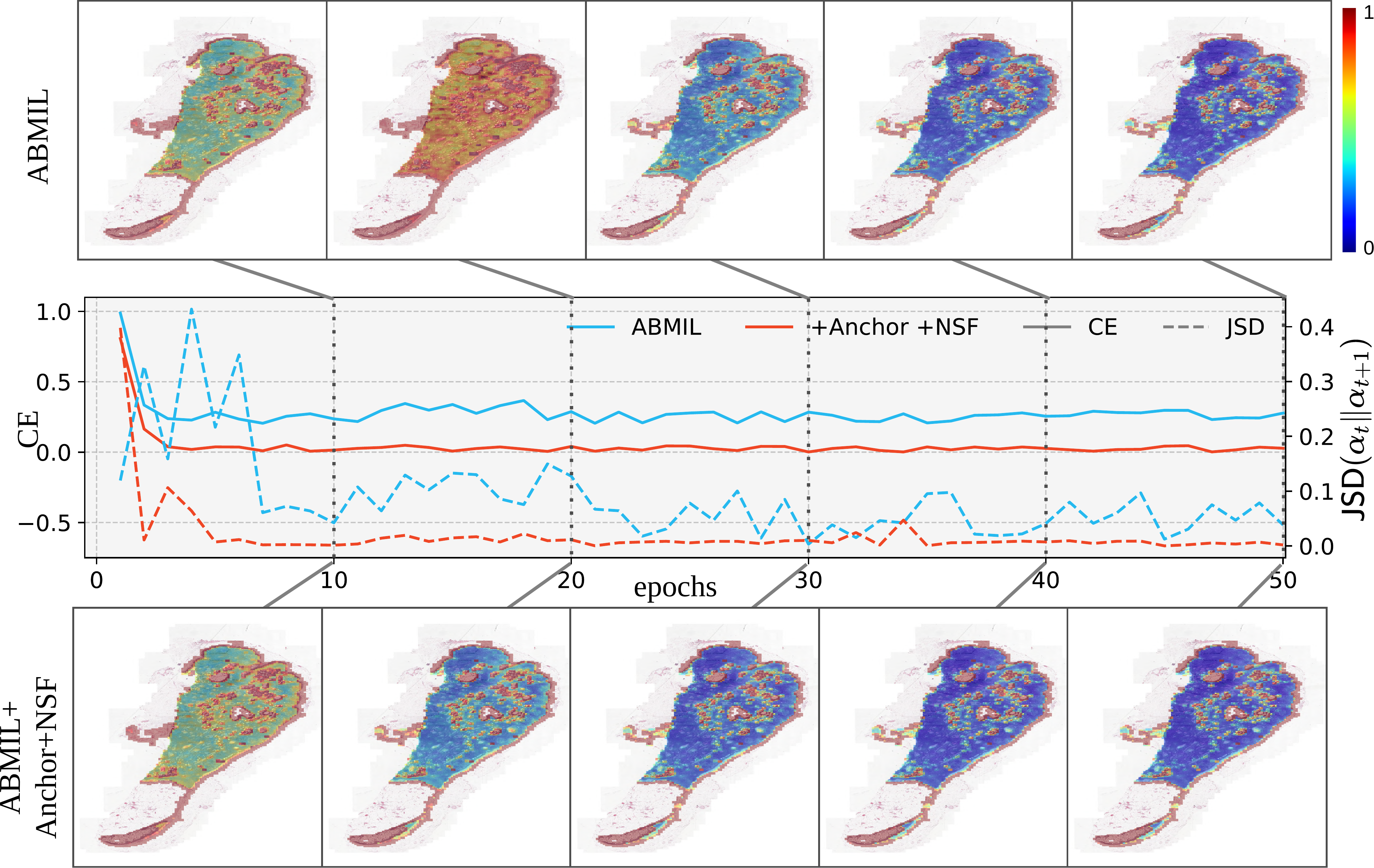}
\end{center}
\caption{Visualization of attention dynamics on a normal WSI for ABMIL vs. ABMIL + anchor + NSF.}
\end{figure}

\begin{figure}[!ht]
\begin{center}
%\framebox[4.0in]{$\;$}
\includegraphics[width=\linewidth]{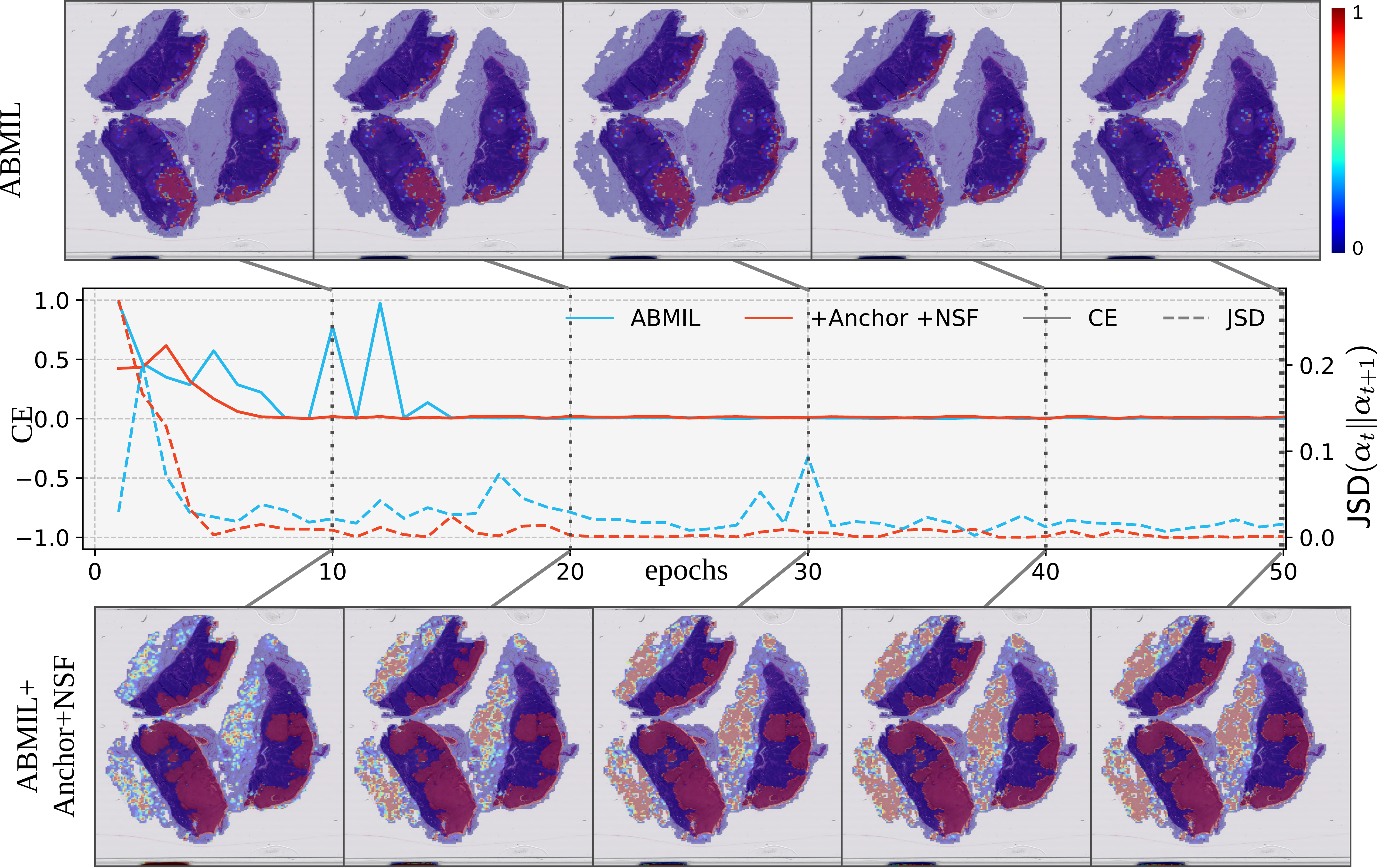}
\end{center}
\caption{Visualization of attention dynamics on a tumor WSI for ABMIL vs. ABMIL + anchor + NSF.}
\end{figure}

\begin{figure}[!ht]
\begin{center}
%\framebox[4.0in]{$\;$}
\includegraphics[width=\linewidth]{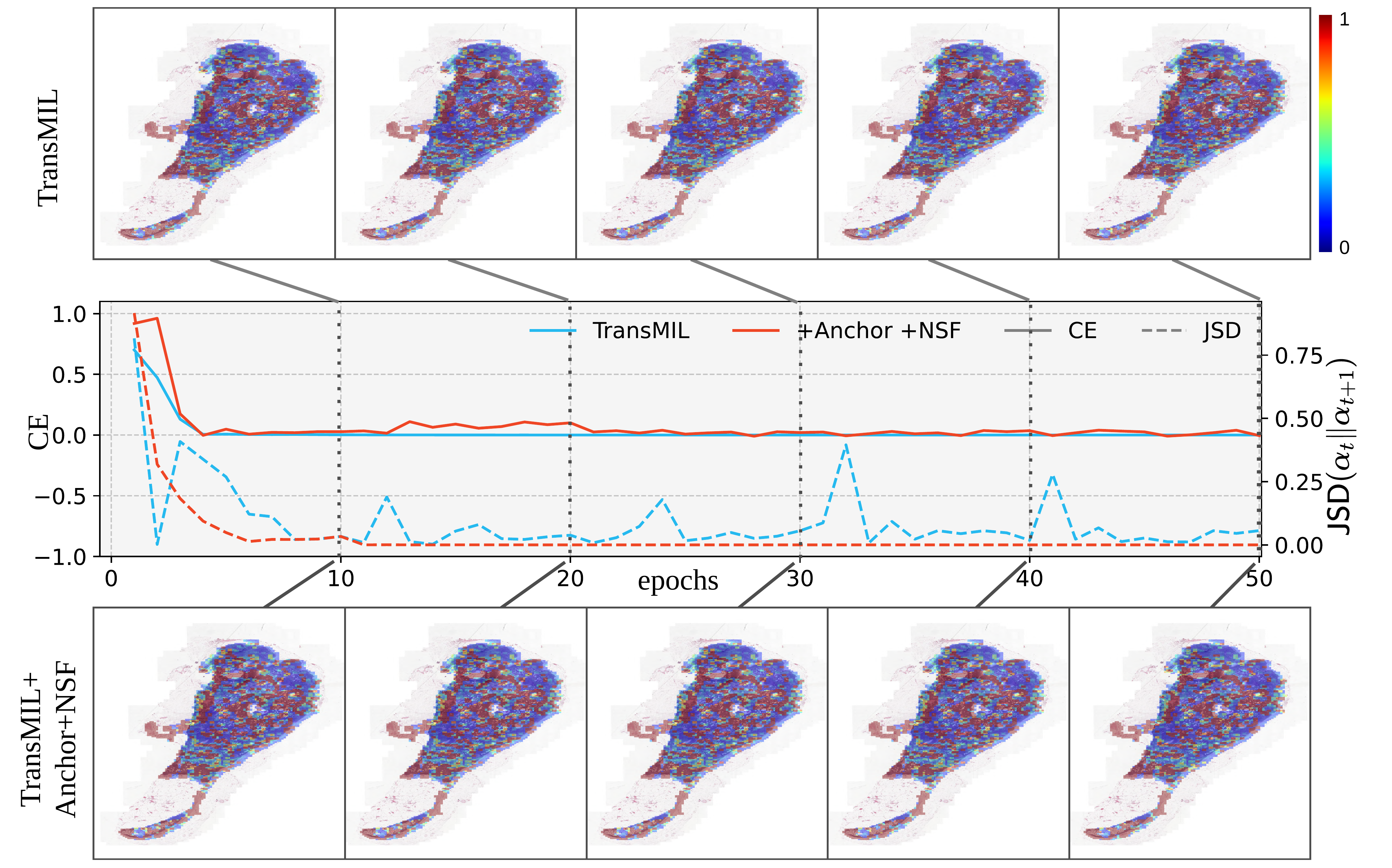}
\end{center}
\caption{ Visualization of attention dynamics on a normal WSI for TransMIL vs. TransMIL + anchor + NSF.}
\end{figure}

\begin{figure}[!ht]
\begin{center}
%\framebox[4.0in]{$\;$}
\includegraphics[width=\linewidth]{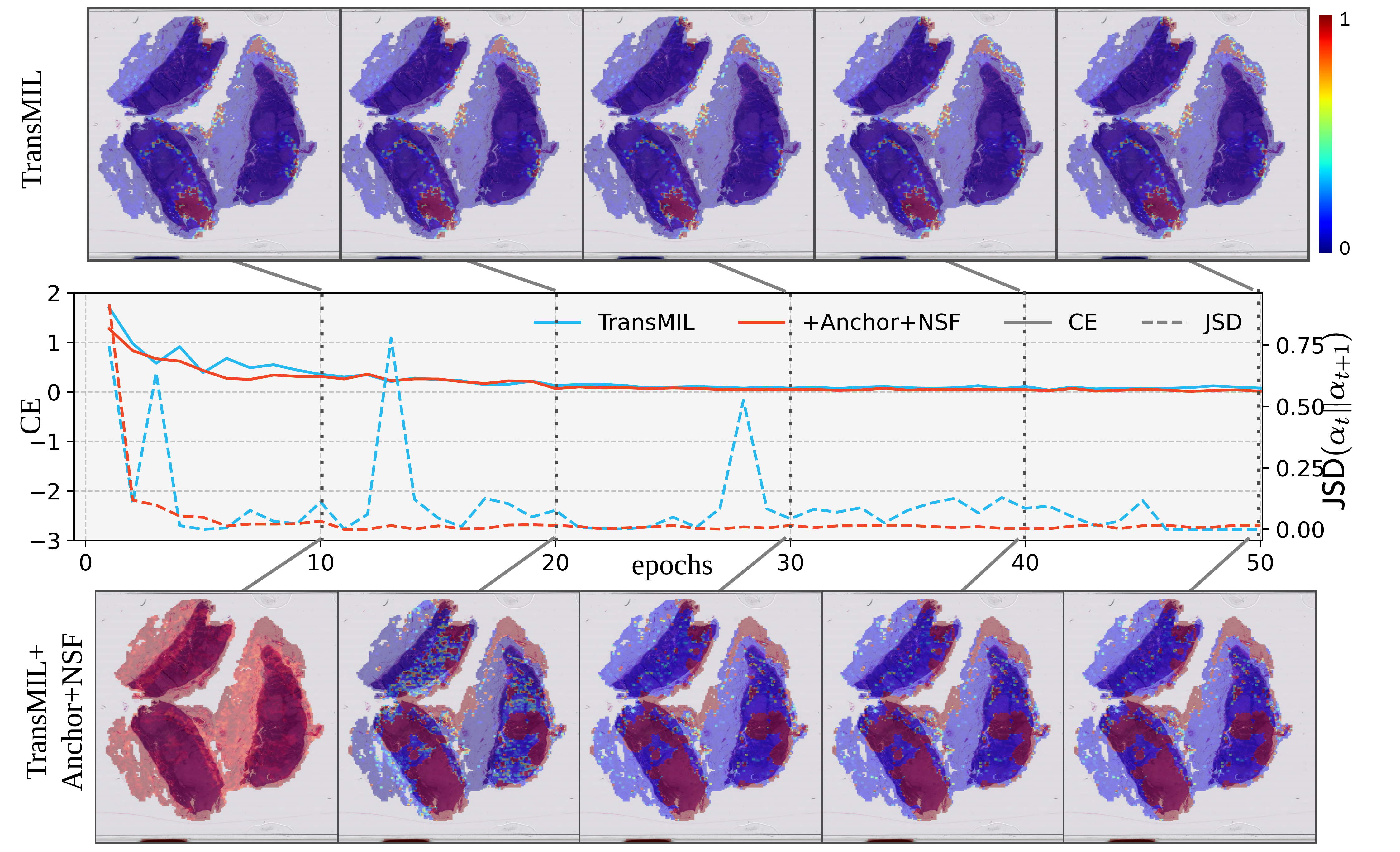}
\end{center}
\caption{ Visualization of attention dynamics on a tumor WSI for TransMIL vs. TransMIL + anchor + NSF.}
\end{figure}

\begin{figure}[!ht]
\begin{center}
%\framebox[4.0in]{$\;$}
\includegraphics[width=\linewidth]{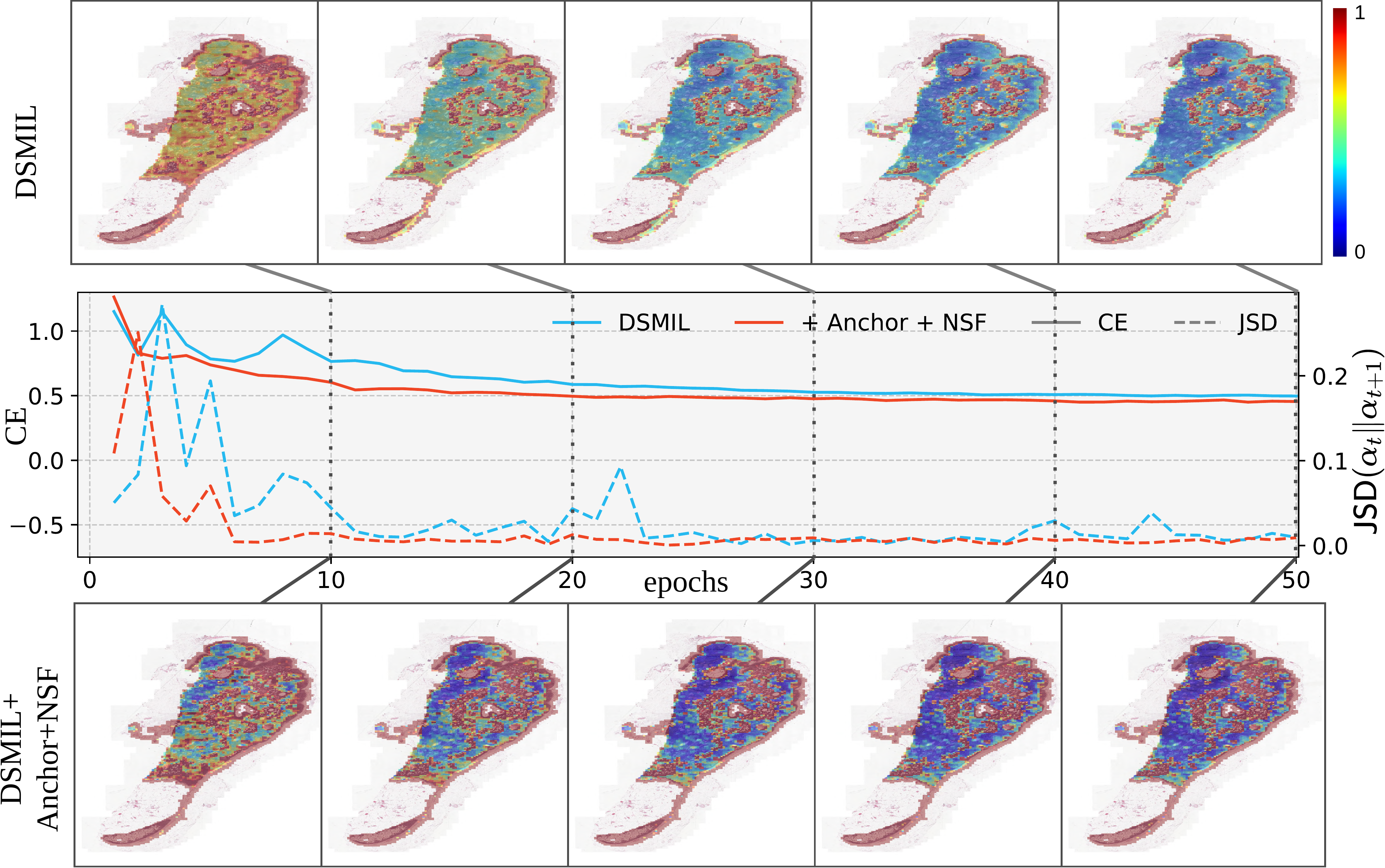}
\end{center}
\caption{ Visualization of attention dynamics on a normal WSI for DSMIL vs.  DSMIL + anchor + NSF.}
\end{figure}

\begin{figure}[!ht]
\begin{center}
%\framebox[4.0in]{$\;$}
\includegraphics[width=\linewidth]{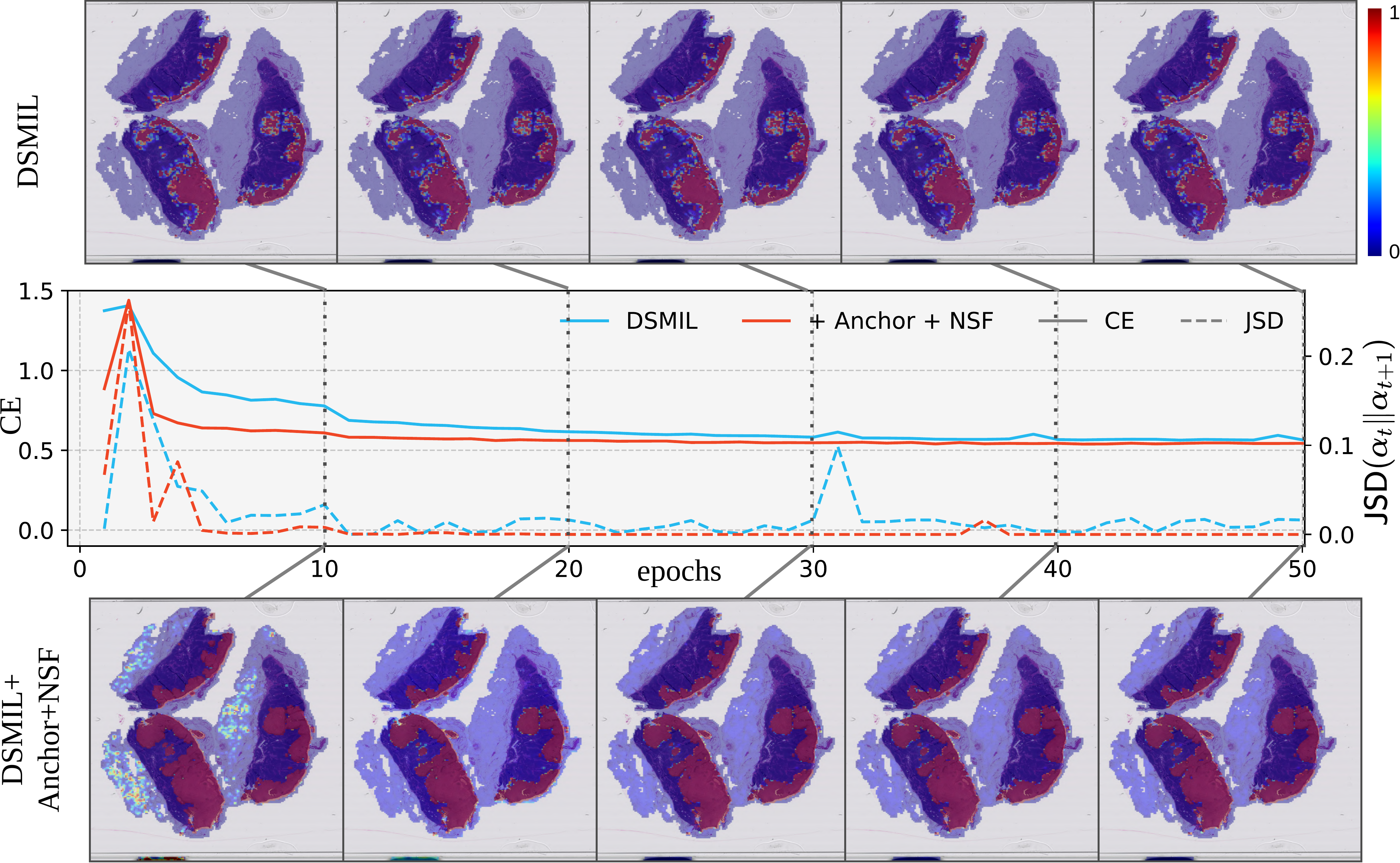}
\end{center}
\caption{ Visualization of attention dynamics on a tumor WSI for DSMIL vs. DSMIL + anchor + NSF.}
\end{figure}

\clearpage
\subsection{BRACS dataset}
\begin{figure}[!ht]
\begin{center}
%\framebox[4.0in]{$\;$}
\includegraphics[width=\linewidth]{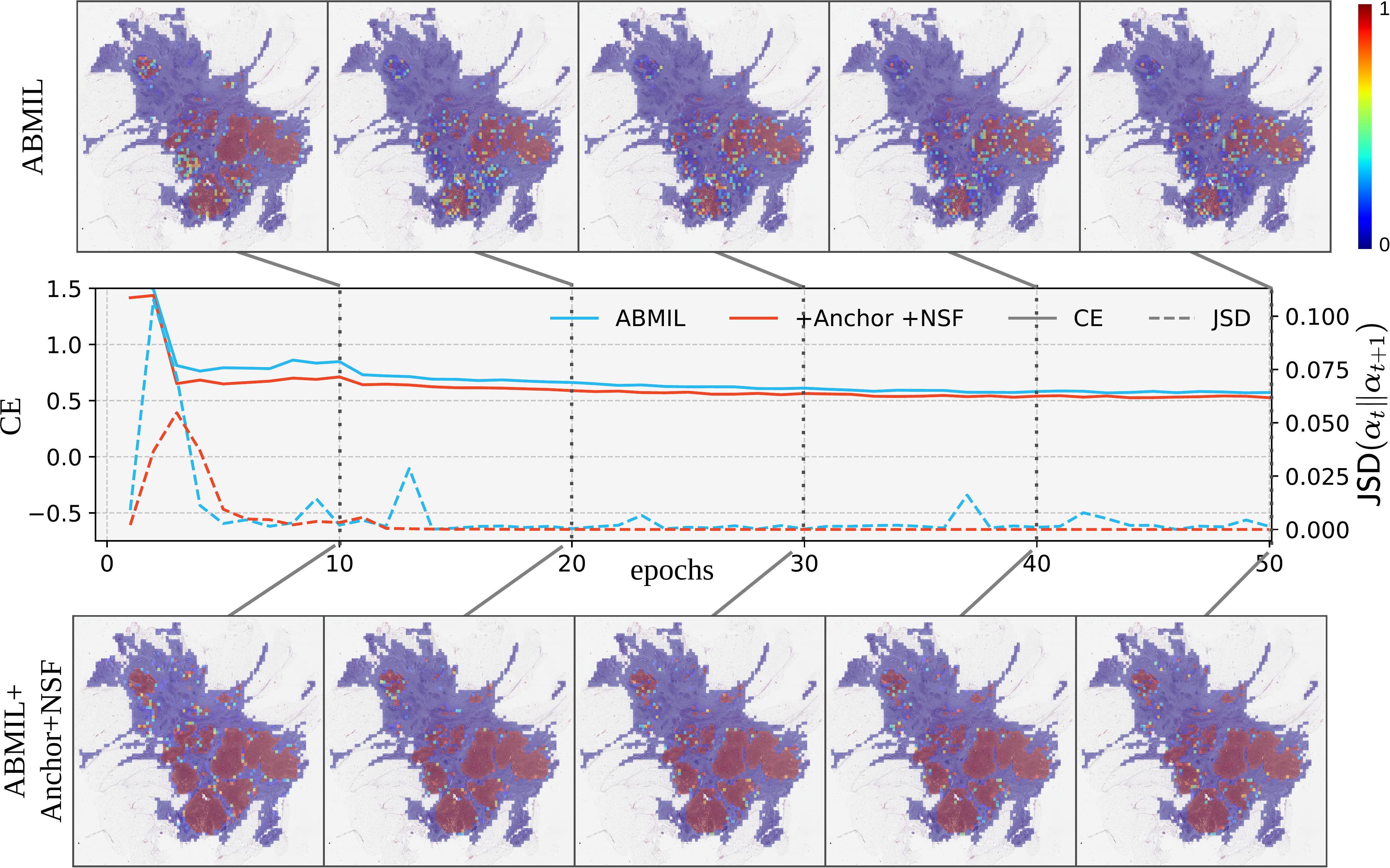}
\end{center}
\caption{ Visualization of attention dynamics on a normal WSI for ABMIL vs. ABMIL + anchor + NSF.}
\end{figure}

\begin{figure}[!ht]
\begin{center}
%\framebox[4.0in]{$\;$}
\includegraphics[width=\linewidth]{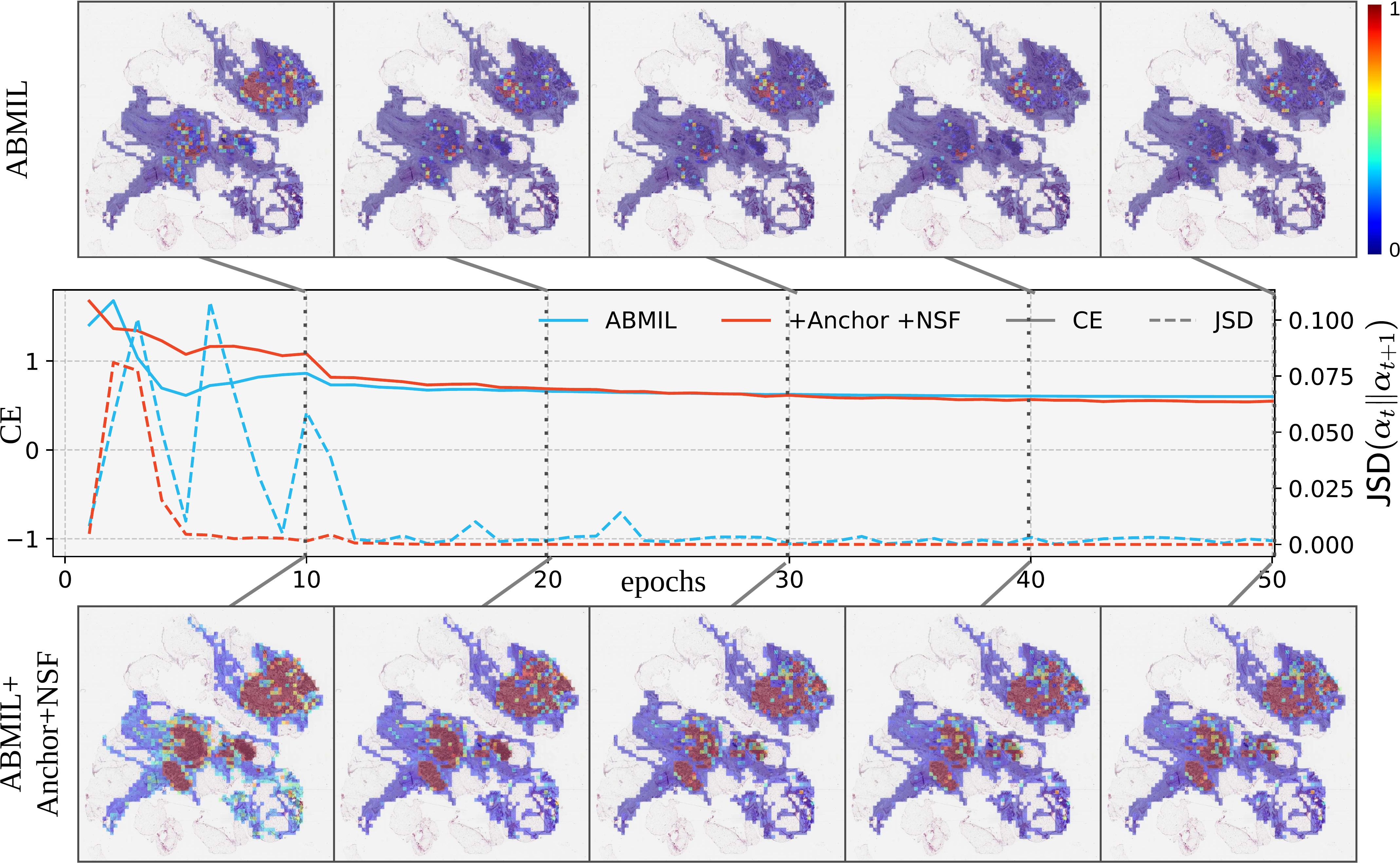}
\end{center}
\caption{ Visualization of attention dynamics on a tumor WSI for ABMIL vs. ABMIL + anchor + NSF.}
\end{figure}

\begin{figure}[!ht]
\begin{center}
%\framebox[4.0in]{$\;$}
\includegraphics[width=\linewidth]{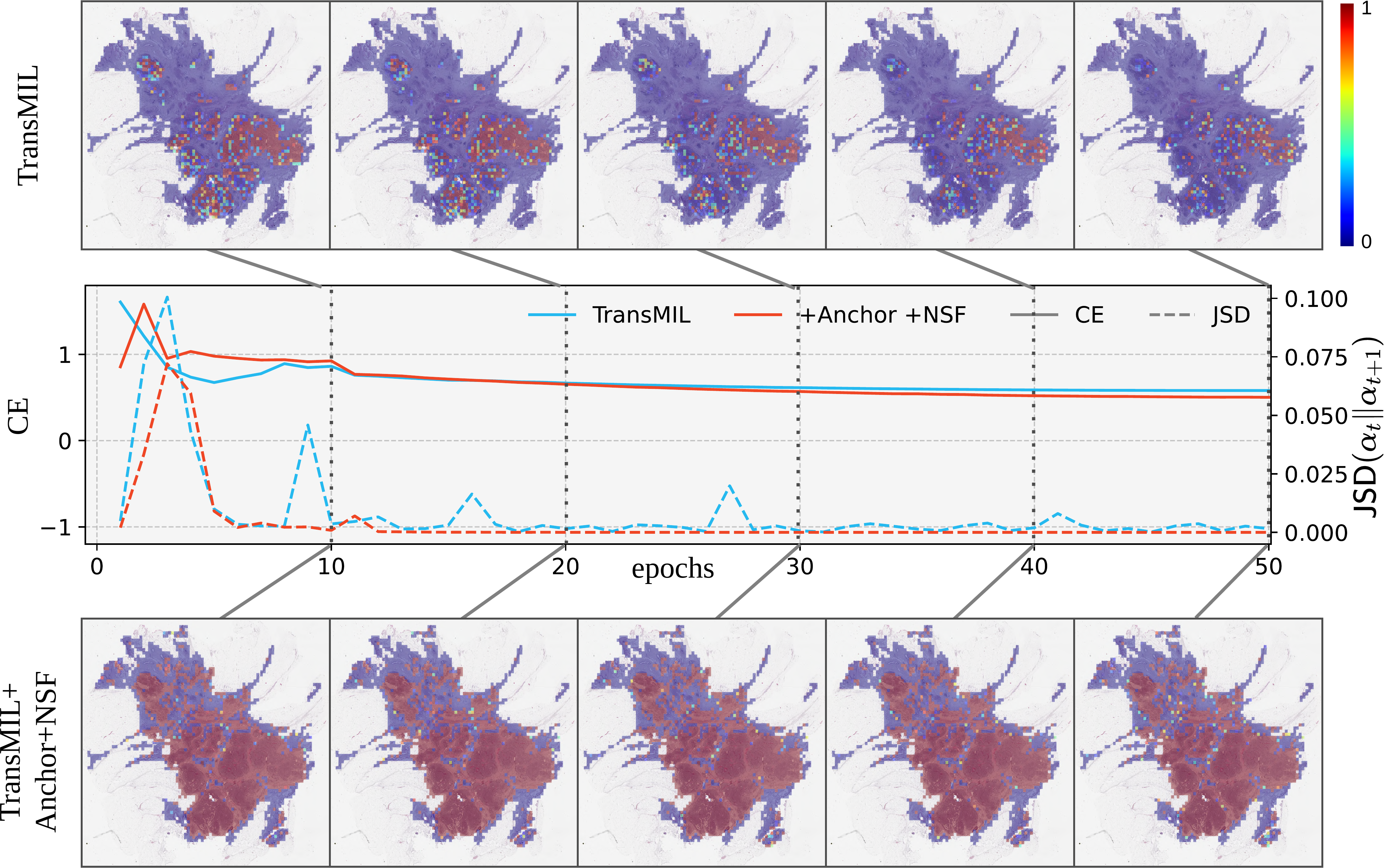}
\end{center}
\caption{ Visualization of attention dynamics on a tumor WSI for TransMIL vs. TransMIL + anchor + NSF.}
\end{figure}

\begin{figure}[!ht]
\begin{center}
%\framebox[4.0in]{$\;$}
\includegraphics[width=\linewidth]{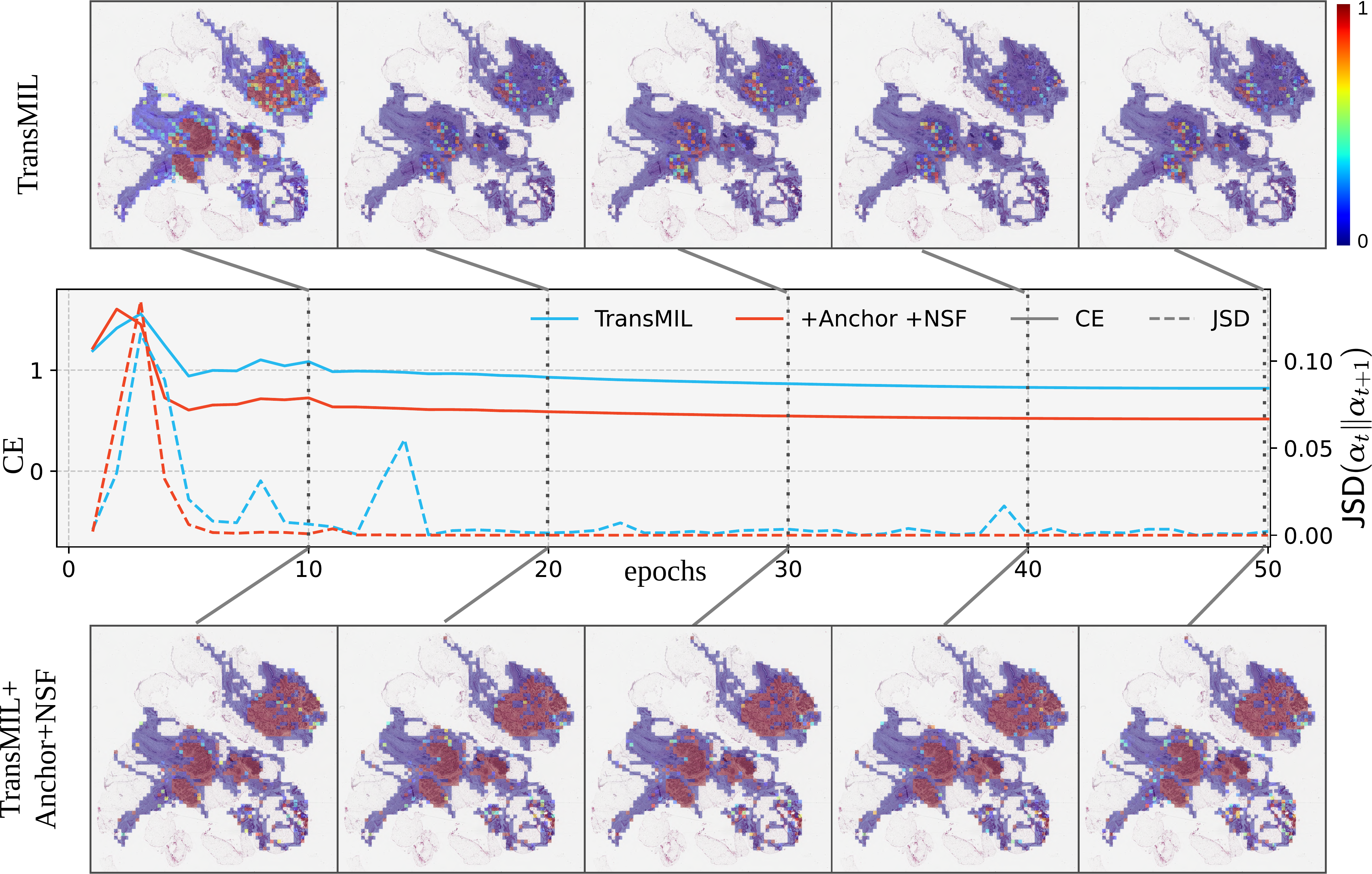}
\end{center}
\caption{ Visualization of attention dynamics on a tumor WSI for TransMIL vs. TransMIL + anchor + NSF.}
\end{figure}

\begin{figure}[!ht]
\begin{center}
%\framebox[4.0in]{$\;$}
\includegraphics[width=\linewidth]{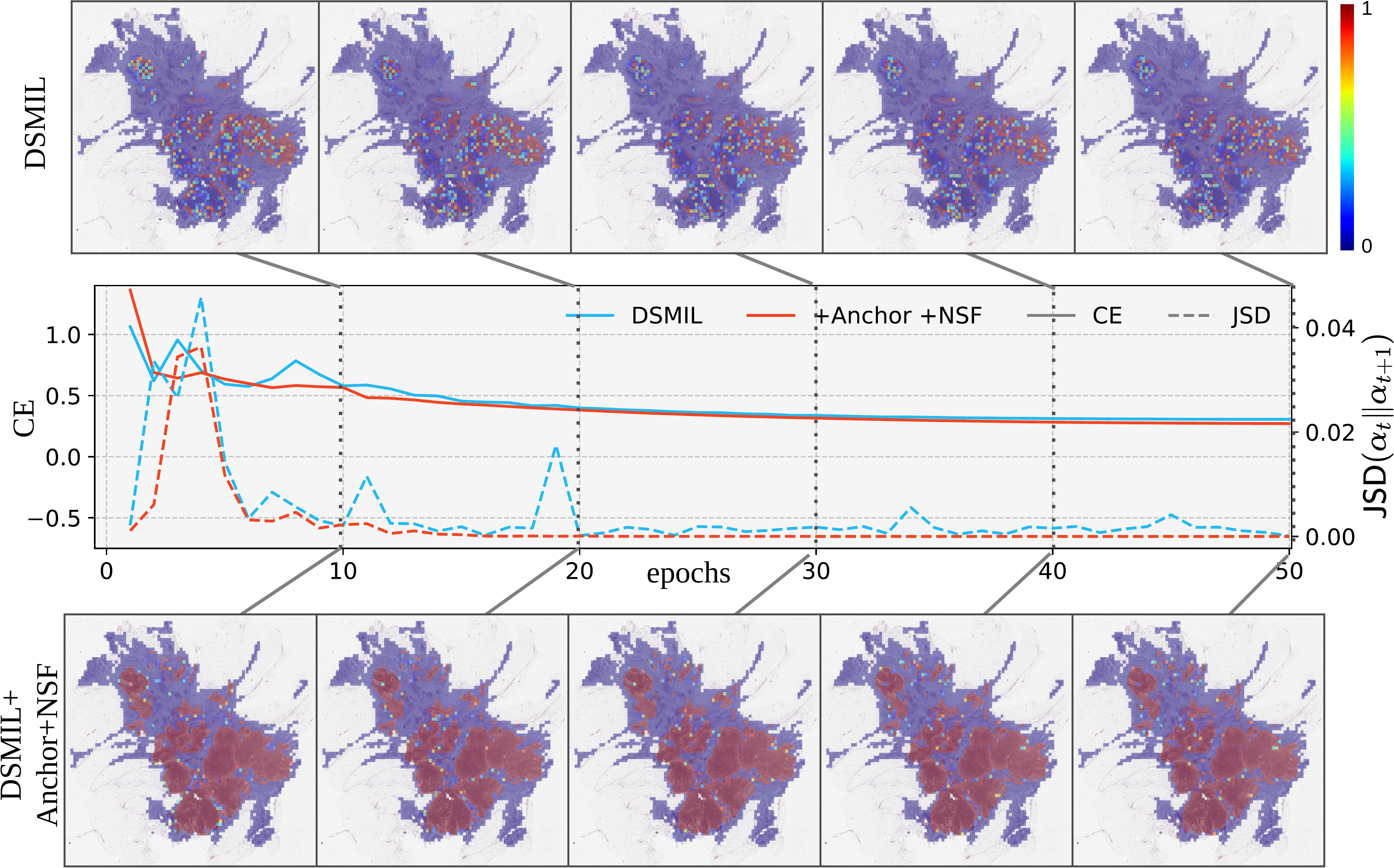}
\end{center}
\caption{ Visualization of attention dynamics on a tumor WSI for DSMIL vs. DSMIL + anchor + NSF.}
\end{figure}

\begin{figure}[!ht]
\begin{center}
%\framebox[4.0in]{$\;$}
\includegraphics[width=\linewidth]{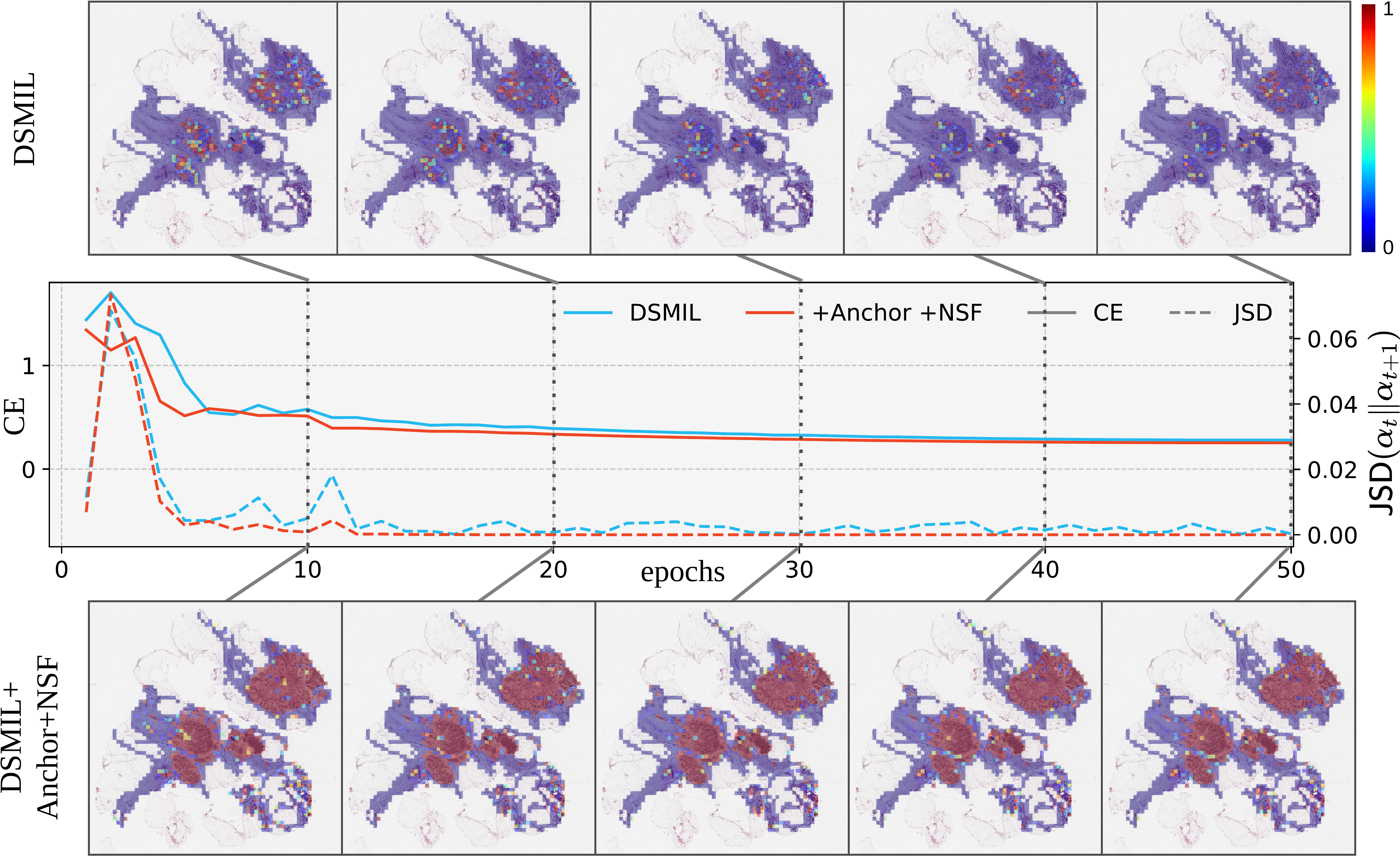}
\end{center}
\caption{ Visualization of attention dynamics on a tumor WSI for DSMIL vs. DSMIL + anchor + NSF.}
\end{figure}

Despite these advances, several avenues remain open for future investigation: 

ASMIL employs an EMA‐updated anchor model to stabilize attention dynamics, but this introduces additional computational overhead. An important direction is the development of intrinsic training strategies, such as regularization, that achieve comparable stability without auxiliary modules, thereby improving efficiency in large‐scale WSI applications.

\section{Limitations and Future Work}
\label{Appendix:LimitationsFutureWorks}
\begin{wrapfigure}[30]{r}{0.45\textwidth}
\vspace{1.4cm}
\includegraphics[width=\linewidth]{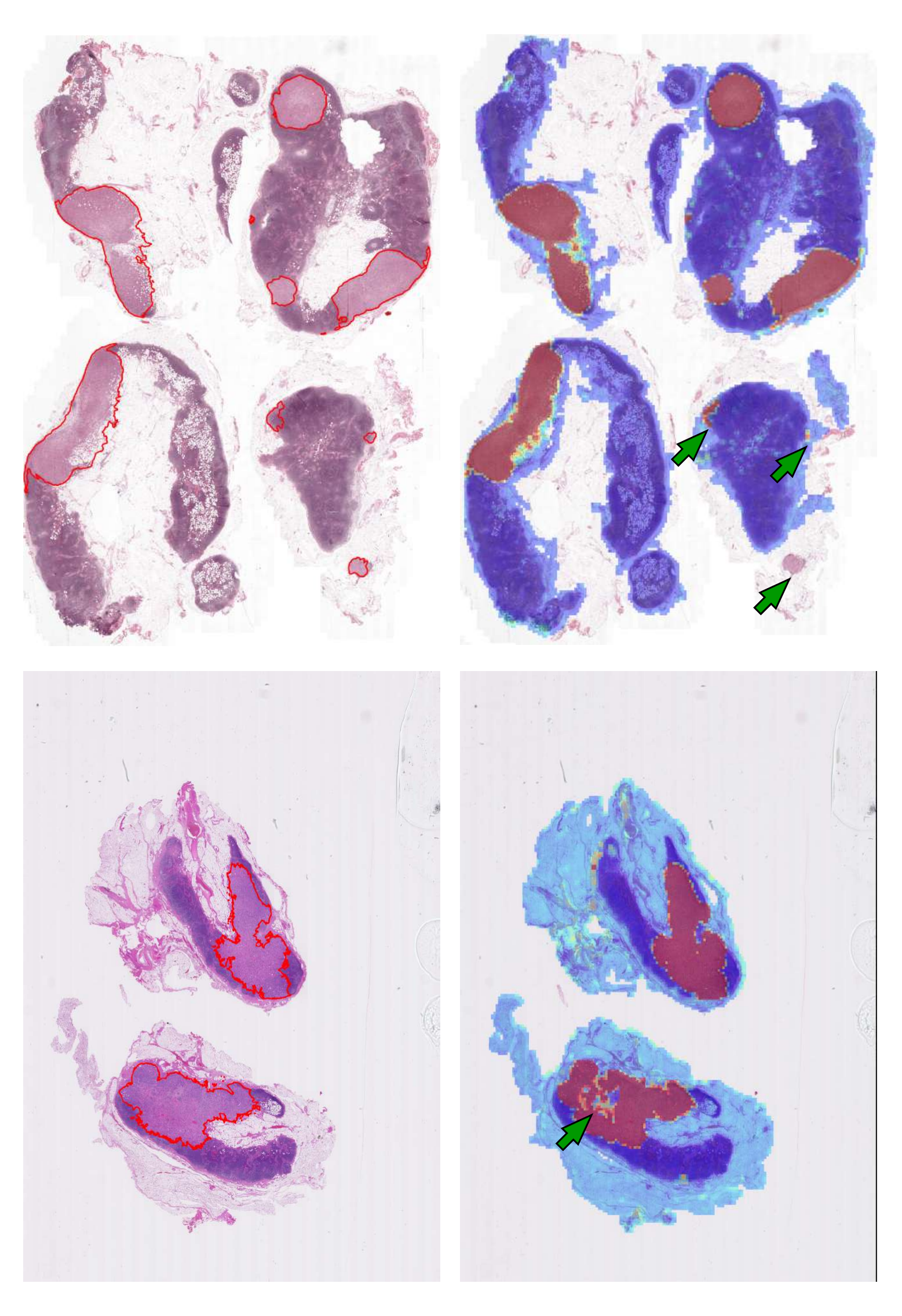}
\caption{Left: annotated WSI. Right: attention map generated by ASMIL, which fails to assign high attention to all tumor regions, as highlighted by the green arrow.}
\label{fig:fail}
\vspace{1.4cm}
\end{wrapfigure}
{Meanwhile, as more advanced regularization techniques such as MIL-dropout \citep{zhu2025how} continue to emerge, integrating them into the ASMIL framework represents a highly promising direction for future work. Such enhancements could further improve the model’s generalization ability while yielding more faithful and stable attention distributions.}

\begin{table}[]
\centering
{\caption{Rate of cancerous WSIs without missed regions on the CAMELYON-16 dataset.}
\label{Tab:NoMissing}
\begin{tabular}{ccccccc}
\toprule
\rowcolor{lightgray} Method & Clam    & TransMIL & DTFD-MIL & DSMIL   & CAMIL   & ASMIL   \\ \midrule
Rate   & 46.93\% & 3.246\%  & 50.34\%  & 48.22\% & 49.78\% & 54.63\% \\ \bottomrule
\end{tabular}}
\end{table}

Furthermore, a limitation of our approach is that ASMIL can fail by assigning low attention to tiny foci and small tumor regions (see \Cref{fig:fail}), particularly when large and small cancerous regions coexist within a single WSI. {We fix the attention threshold at 0.5 and count a cancerous WSI as successfully localized if all regions inside the tumor annotation exceed this threshold; the success rates are reported in \Cref{Tab:NoMissing}. ASMIL achieves the highest success rate, which we attribute to the NSF in the anchor model that mitigates over-concentrated attention.} This indicates room for improvement. Nevertheless, compared with published baselines, ASMIL’s attention maps consistently achieve higher Dice and FROC scores. One avenue to further enhance localization performance is to bootstrap training with a mixture of synthetic data and real WSI data. 
These directions are beyond the scope of this work and will be investigated in future research \cite{wu2024test, chiadapting}.
\section{LLM Usage Statement}
LLM used only for grammar and wording edits; no generation of ideas, methods, analyses, results, or citations. The authors reviewed all edits and accept full responsibility.

\end{document}

%% file: math_commands.tex
\usepackage{amsmath,amsfonts,bm}

\def\eqref#1{equation~\ref{#1}}
\def\1{\bm{1}}

\DeclareMathAlphabet{\mathsfit}{\encodingdefault}{\sfdefault}{m}{sl}
\SetMathAlphabet{\mathsfit}{bold}{\encodingdefault}{\sfdefault}{bx}{n}

%% file: iclr2025_conference.bib
@inproceedings{
sun2025pathgenm,
title={PathGen-1.6M: 1.6 Million Pathology Image-text Pairs Generation through Multi-agent Collaboration},
author={Yuxuan Sun and Yunlong Zhang and Yixuan Si and Chenglu Zhu and Kai Zhang and Zhongyi Shui and Jingxiong Li and Xuan Gong and XINHENG LYU and Tao Lin and Lin Yang},
booktitle={The Thirteenth International Conference on Learning Representations},
year={2025},
url={https://openreview.net/forum?id=rFpZnn11gj}
}

@ARTICLE{10530149,
  author={Qu, Linhao and Ma, Yingfan and Luo, Xiaoyuan and Guo, Qinhao and Wang, Manning and Song, Zhijian},
  journal={IEEE Transactions on Circuits and Systems for Video Technology}, 
  title={Rethinking Multiple Instance Learning for Whole Slide Image Classification: A Good Instance Classifier Is All You Need}, 
  year={2024},
  volume={34},
  number={10},
  pages={9732-9744},
  doi={10.1109/TCSVT.2024.3400876}}

@InProceedings{Lu_2023_CVPR,
    author    = {Lu, Ming Y. and Chen, Bowen and Zhang, Andrew and Williamson, Drew F. K. and Chen, Richard J. and Ding, Tong and Le, Long Phi and Chuang, Yung-Sung and Mahmood, Faisal},
    title     = {Visual Language Pretrained Multiple Instance Zero-Shot Transfer for Histopathology Images},
    booktitle = {Proceedings of the IEEE/CVF Conference on Computer Vision and Pattern Recognition (CVPR)},
    month     = {June},
    year      = {2023},
    pages     = {19764-19775}
}

@INPROCEEDINGS{7950690,
  author={Das, Kausik and Karri, Sri Phani Krishna and Guha Roy, Abhijit and Chatterjee, Jyotirmoy and Sheet, Debdoot},
  booktitle={2017 IEEE 14th International Symposium on Biomedical Imaging (ISBI 2017)}, 
  title={Classifying histopathology whole-slides using fusion of decisions from deep convolutional network on a collection of random multi-views at multi-magnification}, 
  year={2017},
  volume={},
  number={},
  pages={1024-1027},
  doi={10.1109/ISBI.2017.7950690}}

@INPROCEEDINGS{8363642,
  author={Das, Kausik and Conjeti, Sailesh and Roy, Abhijit Guha and Chatterjee, Jyotirmoy and Sheet, Debdoot},
  booktitle={2018 IEEE 15th International Symposium on Biomedical Imaging (ISBI 2018)}, 
  title={Multiple instance learning of deep convolutional neural networks for breast histopathology whole slide classification}, 
  year={2018},
  volume={},
  number={},
  pages={578-581},
  doi={10.1109/ISBI.2018.8363642}}

@article{li2019attention,
  title={An attention-based multi-resolution model for prostate whole slide imageclassification and localization},
  author={Li, Jiayun and Li, Wenyuan and Gertych, Arkadiusz and Knudsen, Beatrice S and Speier, William and Arnold, Corey W},
  journal={arXiv preprint arXiv:1905.13208},
  year={2019}
}

@article{lu2021data,
  title={Data-efficient and weakly supervised computational pathology on whole-slide images},
  author={Lu, Ming Y and Williamson, Drew FK and Chen, Tiffany Y and Chen, Richard J and Barbieri, Matteo and Mahmood, Faisal},
  journal={Nature biomedical engineering},
  volume={5},
  number={6},
  pages={555--570},
  year={2021},
  publisher={Nature Publishing Group UK London}
}

@INPROCEEDINGS{9578683,
  author={Li, Bin and Li, Yin and Eliceiri, Kevin W.},
  booktitle={2021 IEEE/CVF Conference on Computer Vision and Pattern Recognition (CVPR)}, 
  title={Dual-stream Multiple Instance Learning Network for Whole Slide Image Classification with Self-supervised Contrastive Learning}, 
  year={2021},
  volume={},
  number={},
  pages={14313-14323},
  doi={10.1109/CVPR46437.2021.01409}}

@INPROCEEDINGS{9522980,
  author={Zhang, Jingwei and Ma, Ke and Van Arnam, John and Gupta, Rajarsi and Saltz, Joel and Vakalopoulou, Maria and Samaras, Dimitris},
  booktitle={2021 IEEE/CVF Conference on Computer Vision and Pattern Recognition Workshops (CVPRW)}, 
  title={A Joint Spatial and Magnification Based Attention Framework for Large Scale Histopathology Classification}, 
  year={2021},
  volume={},
  number={},
  pages={3771-3779},
  doi={10.1109/CVPRW53098.2021.00418}}

@inproceedings{
shao2025do,
title={Do Multiple Instance Learning Models Transfer?},
author={Daniel Shao and Richard J. Chen and Andrew H. Song and Joel Runevic and Ming Y. Lu and Tong Ding and Faisal Mahmood},
booktitle={Forty-second International Conference on Machine Learning},
year={2025},
url={https://openreview.net/forum?id=hfLqdquVt3}
}

@inproceedings{lin2023interventional,
  title={Interventional bag multi-instance learning on whole-slide pathological images},
  author={Lin, Tiancheng and Yu, Zhimiao and Hu, Hongyu and Xu, Yi and Chen, Chang-Wen},
  booktitle={Proceedings of the IEEE/CVF Conference on Computer Vision and Pattern Recognition},
  pages={19830--19839},
  year={2023}
}

@inproceedings{zhang2024attention,
  title={Attention-challenging multiple instance learning for whole slide image classification},
  author={Zhang, Yunlong and Li, Honglin and Sun, Yunxuan and Zheng, Sunyi and Zhu, Chenglu and Yang, Lin},
  booktitle={European Conference on Computer Vision},
  pages={125--143},
  year={2024},
  organization={Springer}
}

@inproceedings{song2024morphological,
    title={Morphological Prototyping for Unsupervised Slide Representation Learning in Computational Pathology},
    author={Song, Andrew H and Chen, Richard J and Ding, Tong and Williamson, Drew FK and Jaume, Guillaume and Mahmood, Faisal},
    booktitle={Proceedings of the IEEE/CVF Conference on Computer Vision and Pattern Recognition},
    year={2024},
}

@inproceedings{guo2023higt,
  title={Higt: Hierarchical interaction graph-transformer for whole slide image analysis},
  author={Guo, Ziyu and Zhao, Weiqin and Wang, Shujun and Yu, Lequan},
  booktitle={International Conference on Medical Image Computing and Computer-Assisted Intervention},
  pages={755--764},
  year={2023},
  organization={Springer}
}

@ARTICLE{10849962,
  author={Tran, Manuel and Wagner, Sophia and Weichert, Wilko and Matek, Christian and Boxberg, Melanie and Peng, Tingying},
  journal={IEEE Transactions on Medical Imaging}, 
  title={Navigating Through Whole Slide Images With Hierarchy, Multi-Object, and Multi-Scale Data}, 
  year={2025},
  volume={44},
  number={5},
  pages={2002-2015},
  doi={10.1109/TMI.2025.3532728}}

@article{buzzard2024paths,
  title={PATHS: A Hierarchical Transformer for Efficient Whole Slide Image Analysis},
  author={Buzzard, Zak and Hemker, Konstantin and Simidjievski, Nikola and Jamnik, Mateja},
  journal={arXiv preprint arXiv:2411.18225},
  year={2024}
}

@inproceedings{
fourkioti2024camil,
title={{CAMIL}: Context-Aware Multiple Instance Learning for Cancer Detection and Subtyping in Whole Slide Images},
author={Olga Fourkioti and Matt De Vries and Chris Bakal},
booktitle={The Twelfth International Conference on Learning Representations},
year={2024},
url={https://openreview.net/forum?id=rzBskAEmoc}
}

@article{zhang2025AEM,
  title={AEM: Attention Entropy Maximization for Multiple Instance Learning based Whole Slide Image Classification},
  author={Zhang, Yunlong and Shui, Zhongyi and Sun, Yunxuan and Li, Honglin and Li, Jingxiong and Zhu, Chenglu and Yang, Lin},
  journal={International Conference on Medical Image Computing and Computer Assisted Intervention},
  year={2025}
}

@inproceedings{ilse2018attention,
  title={Attention-based deep multiple instance learning},
  author={Ilse, Maximilian and Tomczak, Jakub and Welling, Max},
  booktitle={International conference on machine learning},
  pages={2127--2136},
  year={2018},
  organization={PMLR}
}

@book{cover1999elements,
  title={Elements of information theory},
  author={Cover, Thomas M},
  year={1999},
  publisher={John Wiley \& Sons}
}

@article{brancati2022bracs,
  title={Bracs: A dataset for breast carcinoma subtyping in h\&e histology images},
  author={Brancati, Nadia and Anniciello, Anna Maria and Pati, Pushpak and Riccio, Daniel and Scognamiglio, Giosu{\`e} and Jaume, Guillaume and De Pietro, Giuseppe and Di Bonito, Maurizio and Foncubierta, Antonio and Botti, Gerardo and others},
  journal={Database},
  volume={2022},
  pages={baac093},
  year={2022},
  publisher={Oxford University Press UK}
}

@ARTICLE{8447230,
  author={Bándi, Péter and Geessink, Oscar and Manson, Quirine and Van Dijk, Marcory and Balkenhol, Maschenka and Hermsen, Meyke and Ehteshami Bejnordi, Babak and Lee, Byungjae and Paeng, Kyunghyun and Zhong, Aoxiao and Li, Quanzheng and Zanjani, Farhad Ghazvinian and Zinger, Svitlana and Fukuta, Keisuke and Komura, Daisuke and Ovtcharov, Vlado and Cheng, Shenghua and Zeng, Shaoqun and Thagaard, Jeppe and Dahl, Anders B. and Lin, Huangjing and Chen, Hao and Jacobsson, Ludwig and Hedlund, Martin and Çetin, Melih and Halıcı, Eren and Jackson, Hunter and Chen, Richard and Both, Fabian and Franke, Jörg and Küsters-Vandevelde, Heidi and Vreuls, Willem and Bult, Peter and van Ginneken, Bram and van der Laak, Jeroen and Litjens, Geert},
  journal={IEEE Transactions on Medical Imaging}, 
  title={From Detection of Individual Metastases to Classification of Lymph Node Status at the Patient Level: The CAMELYON17 Challenge}, 
  year={2019},
  volume={38},
  number={2},
  pages={550-560},
  doi={10.1109/TMI.2018.2867350}}

@article{CAMELYON16,
    author = {Ehteshami Bejnordi, Babak and Veta, Mitko and Johannes van Diest, Paul and van Ginneken, Bram and Karssemeijer, Nico and Litjens, Geert and van der Laak, Jeroen A. W. M. and and the CAMELYON16 Consortium},
    title = {Diagnostic Assessment of Deep Learning Algorithms for Detection of Lymph Node Metastases in Women With Breast Cancer},
    journal = {JAMA},
    volume = {318},
    number = {22},
    pages = {2199-2210},
    year = {2017},
    month = {12},
    issn = {0098-7484},
    doi = {10.1001/jama.2017.14585},
    url = {https://doi.org/10.1001/jama.2017.14585},
    eprint = {https://jamanetwork.com/journals/jama/articlepdf/2665774/jama\_ehteshami\_bejnordi\_2017\_oi\_170113.pdf},
}

@article{verghese2023computational,
  title={Computational pathology in cancer diagnosis, prognosis, and prediction--present day and prospects},
  author={Verghese, Gregory and Lennerz, Jochen K and Ruta, Danny and Ng, Wen and Thavaraj, Selvam and Siziopikou, Kalliopi P and Naidoo, Threnesan and Rane, Swapnil and Salgado, Roberto and Pinder, Sarah E and others},
  journal={The Journal of pathology},
  volume={260},
  number={5},
  pages={551--563},
  year={2023},
  publisher={Wiley Online Library}
}

@misc{evans2001method,
  author       = {James Bacus},
  title        = {Method and apparatus for acquiring and reconstructing magnified specimen images from a computer-controlled microscope },
  number       = {US20010050999A1},
  year         = {2001},
  month        = dec,
  day          = {13},
  note         = {US Patent Application},
  url          = {https://patents.google.com/patent/US20010050999A1/a}
}

@inproceedings{NIPS1990_e46de7e1,
 author = {Keeler, James and Rumelhart, David and Leow, Wee},
 booktitle = {Advances in Neural Information Processing Systems},
 editor = {R.P. Lippmann and J. Moody and D. Touretzky},
 pages = {},
 publisher = {Morgan-Kaufmann},
 title = {Integrated Segmentation and Recognition of Hand-Printed Numerals},
 url = {https://proceedings.neurips.cc/paper_files/paper/1990/file/e46de7e1bcaaced9a54f1e9d0d2f800d-Paper.pdf},
 volume = {3},
 year = {1990}
}

@InProceedings{Tang_2023_ICCV,
    author    = {Tang, Wenhao and Huang, Sheng and Zhang, Xiaoxian and Zhou, Fengtao and Zhang, Yi and Liu, Bo},
    title     = {Multiple Instance Learning Framework with Masked Hard Instance Mining for Whole Slide Image Classification},
    booktitle = {Proceedings of the IEEE/CVF International Conference on Computer Vision (ICCV)},
    month     = {October},
    year      = {2023},
    pages     = {4078-4087}
}

@inproceedings{
shao2021transmil,
title={Trans{MIL}: Transformer based Correlated Multiple Instance Learning for Whole Slide Image Classification},
author={Zhuchen Shao and Hao Bian and Yang Chen and Yifeng Wang and Jian Zhang and Xiangyang Ji and Yongbing Zhang},
booktitle={Advances in Neural Information Processing Systems},
editor={A. Beygelzimer and Y. Dauphin and P. Liang and J. Wortman Vaughan},
year={2021},
url={https://openreview.net/forum?id=LKUfuWxajHc}
}

@inproceedings{zhang2022dtfd,
  title={Dtfd-mil: Double-tier feature distillation multiple instance learning for histopathology whole slide image classification},
  author={Zhang, Hongrun and Meng, Yanda and Zhao, Yitian and Qiao, Yihong and Yang, Xiaoyun and Coupland, Sarah E and Zheng, Yalin},
  booktitle={Proceedings of the IEEE/CVF conference on computer vision and pattern recognition},
  pages={18802--18812},
  year={2022}
}

@inproceedings{xiong2021nystromformer,
  title={Nystr{\"o}mformer: A nystr{\"o}m-based algorithm for approximating self-attention},
  author={Xiong, Yunyang and Zeng, Zhanpeng and Chakraborty, Rudrasis and Tan, Mingxing and Fung, Glenn and Li, Yin and Singh, Vikas},
  booktitle={Proceedings of the AAAI conference on artificial intelligence},
  volume={35},
  pages={14138--14148},
  year={2021}
}

@article{cheng2021computational,
  title={Computational image analysis identifies histopathological image features associated with somatic mutations and patient survival in gastric adenocarcinoma},
  author={Cheng, Jun and Liu, Yuting and Huang, Wei and Hong, Wenhui and Wang, Lingling and Zhan, Xiaohui and Han, Zhi and Ni, Dong and Huang, Kun and Zhang, Jie},
  journal={Frontiers in Oncology},
  volume={11},
  pages={623382},
  year={2021},
  publisher={Frontiers Media SA}
}

@inproceedings{kang2022benchmarking,
    author    = {Kang, Mingu and Song, Heon and Park, Seonwook and Yoo, Donggeun and Pereira, Sérgio},
    title     = {Benchmarking Self-Supervised Learning on Diverse Pathology Datasets},
    booktitle = {Proceedings of the IEEE/CVF Conference on Computer Vision and Pattern Recognition (CVPR)},
    month     = {June},
    year      = {2023},
    pages     = {3344-3354}
}

@inproceedings{li2021dual,
  title={Dual-stream multiple instance learning network for whole slide image classification with self-supervised contrastive learning},
  author={Li, Bin and Li, Yin and Eliceiri, Kevin W},
  booktitle={Proceedings of the IEEE/CVF Conference on Computer Vision and Pattern Recognition},
  pages={14318--14328},
  year={2021}
}

@article{xu2014weakly,
  title={Weakly supervised histopathology cancer image segmentation and classification},
  author={Xu, Yan and Zhu, Jun-Yan and Eric, I and Chang, Chao and Lai, Maode and Tu, Zhuowen},
  journal={Medical image analysis},
  volume={18},
  number={3},
  pages={591--604},
  year={2014},
  publisher={Elsevier}
}

@article{kraus2016classifying,
  title={Classifying and segmenting microscopy images with deep multiple instance learning},
  author={Kraus, Oren Z and Ba, Jimmy Lei and Frey, Brendan J},
  journal={Bioinformatics},
  volume={32},
  number={12},
  pages={i52--i59},
  year={2016},
  publisher={Oxford University Press}
}

@article{hinton2015distilling,
  title={Distilling the knowledge in a neural network},
  author={Hinton, Geoffrey and Vinyals, Oriol and Dean, Jeff},
  journal={arXiv preprint arXiv:1503.02531},
  year={2015}
}

@article{TOURNIAIRE2023102763,
title = {MS-CLAM: Mixed supervision for the classification and localization of tumors in Whole Slide Images},
journal = {Medical Image Analysis},
volume = {85},
pages = {102763},
year = {2023},
issn = {1361-8415},
doi = {https://doi.org/10.1016/j.media.2023.102763},
url = {https://www.sciencedirect.com/science/article/pii/S1361841523000245},
author = {Paul Tourniaire and Marius Ilie and Paul Hofman and Nicholas Ayache and Hervé Delingette},
}

@InProceedings{pmlr-v48-martins16,
  title = 	 {From Softmax to Sparsemax: A Sparse Model of Attention and Multi-Label Classification},
  author = 	 {Martins, Andre and Astudillo, Ramon},
  booktitle = 	 {Proceedings of The 33rd International Conference on Machine Learning},
  pages = 	 {1614--1623},
  year = 	 {2016},
  editor = 	 {Balcan, Maria Florina and Weinberger, Kilian Q.},
  volume = 	 {48},
  series = 	 {Proceedings of Machine Learning Research},
  address = 	 {New York, New York, USA},
  month = 	 {20--22 Jun},
  publisher =    {PMLR},
  url = 	 {https://proceedings.mlr.press/v48/martins16.html},
}

@article{tsallis1988possible,
  title={Possible generalization of Boltzmann-Gibbs statistics},
  author={Tsallis, Constantino},
  journal={Journal of statistical physics},
  volume={52},
  number={1},
  pages={479--487},
  year={1988},
  publisher={Springer}
}

@article{tarvainen2017mean,
  title={Mean teachers are better role models: Weight-averaged consistency targets improve semi-supervised deep learning results},
  author={Tarvainen, Antti and Valpola, Harri},
  journal={Advances in neural information processing systems},
  volume={30},
  year={2017}
}

@inproceedings{caron2021emerging,
  title={Emerging Properties in Self-Supervised Vision Transformers},
  author={Caron, Mathilde and Touvron, Hugo and Misra, Ishan and J\'egou, Herv\'e  and Mairal, Julien and Bojanowski, Piotr and Joulin, Armand},
  booktitle={Proceedings of the International Conference on Computer Vision (ICCV)},
  year={2021}
}

@misc{oquab2023dinov2,
  title={DINOv2: Learning Robust Visual Features without Supervision},
  author={Oquab, Maxime and Darcet, Timothée and Moutakanni, Theo and Vo, Huy V. and Szafraniec, Marc and Khalidov, Vasil and Fernandez, Pierre and Haziza, Daniel and Massa, Francisco and El-Nouby, Alaaeldin and Howes, Russell and Huang, Po-Yao and Xu, Hu and Sharma, Vasu and Li, Shang-Wen and Galuba, Wojciech and Rabbat, Mike and Assran, Mido and Ballas, Nicolas and Synnaeve, Gabriel and Misra, Ishan and Jegou, Herve and Mairal, Julien and Labatut, Patrick and Joulin, Armand and Bojanowski, Piotr},
  journal={arXiv:2304.07193},
  year={2023}
}

@misc{grill2020bootstrap,
    title = {Bootstrap Your Own Latent: A New Approach to Self-Supervised Learning},
    author = {Jean-Bastien Grill and Florian Strub and Florent Altché and Corentin Tallec and Pierre H. Richemond and Elena Buchatskaya and Carl Doersch and Bernardo Avila Pires and Zhaohan Daniel Guo and Mohammad Gheshlaghi Azar and Bilal Piot and Koray Kavukcuoglu and Rémi Munos and Michal Valko},
    year = {2020},
    eprint = {2006.07733},
    archivePrefix = {arXiv},
    primaryClass = {cs.LG}
}

@article{chen2024uni,
  title={Towards a General-Purpose Foundation Model for Computational Pathology},
  author={Chen, Richard J and Ding, Tong and Lu, Ming Y and Williamson, Drew FK and Jaume, Guillaume and Chen, Bowen and Zhang, Andrew and Shao, Daniel and Song, Andrew H and Shaban, Muhammad and others},
  journal={Nature Medicine},
  publisher={Nature Publishing Group},
  year={2024}
}

@INPROCEEDINGS{resnet,
  author={He, Kaiming and Zhang, Xiangyu and Ren, Shaoqing and Sun, Jian},
  booktitle={2016 IEEE Conference on Computer Vision and Pattern Recognition (CVPR)}, 
  title={Deep Residual Learning for Image Recognition}, 
  year={2016},
  volume={},
  number={},
  pages={770-778},
  doi={10.1109/CVPR.2016.90}}

@article{ILSVRC15,
Author = {Olga Russakovsky and Jia Deng and Hao Su and Jonathan Krause and Sanjeev Satheesh and Sean Ma and Zhiheng Huang and Andrej Karpathy and Aditya Khosla and Michael Bernstein and Alexander C. Berg and Li Fei-Fei},
Title = {{ImageNet Large Scale Visual Recognition Challenge}},
Year = {2015},
journal   = {International Journal of Computer Vision (IJCV)},
doi = {10.1007/s11263-015-0816-y},
volume={115},
number={3},
pages={211-252}
}

@inproceedings{dong2025fast,
  title={Fast and Accurate Gigapixel Pathological Image Classification with Hierarchical Distillation Multi-Instance Learning},
  author={Dong, Jiuyang and Jiang, Junjun and Jiang, Kui and Li, Jiahan and Zhang, Yongbing},
  booktitle={Proceedings of the Computer Vision and Pattern Recognition Conference},
  pages={30818--30828},
  year={2025}
}

@article{maaten2008visualizing,
  title={Visualizing data using t-SNE},
  author={Maaten, Laurens van der and Hinton, Geoffrey},
  journal={Journal of machine learning research},
  volume={9},
  number={Nov},
  pages={2579--2605},
  year={2008}
}

@inproceedings{chen2020simple,
  title={A simple framework for contrastive learning of visual representations},
  author={Chen, Ting and Kornblith, Simon and Norouzi, Mohammad and Hinton, Geoffrey},
  booktitle={International conference on machine learning},
  pages={1597--1607},
  year={2020},
  organization={PmLR}
}

@inproceedings{10.5555/302528.302753,
author = {Maron, Oded and Lozano-P\'{e}rez, Tom\'{a}s},
title = {A framework for multiple-instance learning},
year = {1998},
isbn = {0262100762},
publisher = {MIT Press},
address = {Cambridge, MA, USA},
booktitle = {Proceedings of the 1997 Conference on Advances in Neural Information Processing Systems 10},
pages = {570–576},
numpages = {7},
location = {Denver, Colorado, USA},
series = {NIPS '97}
}

@article{miller1969froc,
  title={The FROC curve: A representation of the observer's performance for the method of free response},
  author={Miller, Harold},
  journal={The Journal of the Acoustical Society of America},
  volume={46},
  number={6B},
  pages={1473--1476},
  year={1969},
  publisher={Acoustical Society of America}
}

@article{bunch1978free,
  title={Free response approach to measurement and characterization of radiographic observer performance.},
  author={Bunch, PC},
  journal={AJR Am J Roentgenol},
  volume={130},
  number={2},
  pages={382},
  year={1978}
}

@inproceedings{
du2025rethinking,
title={Rethinking Multiple-Instance Learning From Feature Space to Probability Space},
author={Zhaolong Du and Shasha Mao and Xuequan Lu and Mengnan Qi and Yimeng Zhang and Jing Gu and Licheng Jiao},
booktitle={The Thirteenth International Conference on Learning Representations},
year={2025},
url={https://openreview.net/forum?id=torbeUlslS}
}

@article{kingma2014adam,
  title={Adam: A method for stochastic optimization},
  author={Kingma, Diederik P and Ba, Jimmy},
  journal={arXiv preprint arXiv:1412.6980},
  year={2014}
}

@inproceedings{NIPS2002_3e6260b8,
 author = {Andrews, Stuart and Tsochantaridis, Ioannis and Hofmann, Thomas},
 booktitle = {Advances in Neural Information Processing Systems},
 editor = {S. Becker and S. Thrun and K. Obermayer},
 pages = {},
 publisher = {MIT Press},
 title = {Support Vector Machines for Multiple-Instance Learning},
 url = {https://proceedings.neurips.cc/paper_files/paper/2002/file/3e6260b81898beacda3d16db379ed329-Paper.pdf},
 volume = {15},
 year = {2002}
}

@inproceedings{
du2023rgmil,
title={{RGMIL}: Guide Your Multiple-Instance Learning Model with Regressor},
author={Zhaolong Du and Shasha Mao and Yimeng Zhang and Shuiping Gou and Licheng Jiao and Lin Xiong},
booktitle={Thirty-seventh Conference on Neural Information Processing Systems},
year={2023},
url={https://openreview.net/forum?id=eGoE9CVRPc}
}

@inproceedings{
tang2023disambiguated,
title={Disambiguated Attention Embedding for Multi-Instance Partial-Label Learning},
author={Wei Tang and Weijia Zhang and Min-Ling Zhang},
booktitle={Thirty-seventh Conference on Neural Information Processing Systems},
year={2023},
url={https://openreview.net/forum?id=NYwbmCrrni}
}

@article{simeoni2025dinov3,
  title={Dinov3},
  author={Sim{\'e}oni, Oriane and Vo, Huy V and Seitzer, Maximilian and Baldassarre, Federico and Oquab, Maxime and Jose, Cijo and Khalidov, Vasil and Szafraniec, Marc and Yi, Seungeun and Ramamonjisoa, Micha{\"e}l and others},
  journal={arXiv preprint arXiv:2508.10104},
  year={2025}
}

@inproceedings{
zhu2025how,
title={How Effective Can Dropout Be in Multiple Instance Learning ?},
author={Wenhui Zhu and Peijie Qiu and Xiwen Chen and Zhangsihao Yang and Aristeidis Sotiras and Abolfazl Razi and Yalin Wang},
booktitle={Forty-second International Conference on Machine Learning},
year={2025},
url={https://openreview.net/forum?id=qsYHqLFCH5}
}

@misc{zhu2023pdl,
      title={PDL: Regularizing Multiple Instance Learning with Progressive Dropout Layers}, 
      author={Wenhui Zhu and Peijie Qiu and Xiwen Chen and Oana M. Dumitrascu and Yalin Wang},
      year={2023},
      eprint={2308.10112},
      archivePrefix={arXiv},
      primaryClass={cs.CV}
}

@inproceedings{
liu2025interpretable,
title={Interpretable Vision-Language Survival Analysis with Ordinal Inductive Bias for Computational Pathology},
author={Pei Liu and Luping Ji and Jiaxiang Gou and Bo Fu and Mao Ye},
booktitle={The Thirteenth International Conference on Learning Representations},
year={2025},
url={https://openreview.net/forum?id=trj2Jq8riA}
}

@inproceedings{tang2024feature,
  title={Feature re-embedding: Towards foundation model-level performance in computational pathology},
  author={Tang, Wenhao and Zhou, Fengtao and Huang, Sheng and Zhu, Xiang and Zhang, Yi and Liu, Bo},
  booktitle={Proceedings of the IEEE/CVF conference on computer vision and pattern recognition},
  pages={11343--11352},
  year={2024}
}

@article{yao2020whole,
  title={Whole slide images based cancer survival prediction using attention guided deep multiple instance learning networks},
  author={Yao, Jiawen and Zhu, Xinliang and Jonnagaddala, Jitendra and Hawkins, Nicholas and Huang, Junzhou},
  journal={Medical image analysis},
  volume={65},
  pages={101789},
  year={2020},
  publisher={Elsevier}
}

@inproceedings{chen2021whole,
  title={Whole slide images are 2d point clouds: Context-aware survival prediction using patch-based graph convolutional networks},
  author={Chen, Richard J and Lu, Ming Y and Shaban, Muhammad and Chen, Chengkuan and Chen, Tiffany Y and Williamson, Drew FK and Mahmood, Faisal},
  booktitle={International Conference on Medical Image Computing and Computer-Assisted Intervention},
  pages={339--349},
  year={2021},
  organization={Springer}
}

@inproceedings{
xiang2023exploring,
title={Exploring Low-Rank Property in Multiple Instance Learning for Whole Slide Image Classification},
author={Jinxi Xiang and Jun Zhang},
booktitle={The Eleventh International Conference on Learning Representations },
year={2023},
url={https://openreview.net/forum?id=01KmhBsEPFO}
}

@inproceedings{chiadapting,
  title={Adapting to Distribution Shift by Visual Domain Prompt Generation},
  author={Chi, Zhixiang and Gu, Li and Zhong, Tao and Liu, Huan and YU, YUANHAO and Plataniotis, Konstantinos N and Wang, Yang},
  booktitle={The Twelfth International Conference on Learning Representations},
  year={2024}
}

@inproceedings{wu2024test,
  title={Test-time domain adaptation by learning domain-aware batch normalization},
  author={Wu, Yanan and Chi, Zhixiang and Wang, Yang and Plataniotis, Konstantinos N and Feng, Songhe},
  booktitle={Proceedings of the AAAI Conference on Artificial Intelligence},
  volume={38},
  number={14},
  pages={15961--15969},
  year={2024}
}

@inproceedings{wu2023metagcd,
  title={Metagcd: Learning to continually learn in generalized category discovery},
  author={Wu, Yanan and Chi, Zhixiang and Wang, Yang and Feng, Songhe},
  booktitle={Proceedings of the IEEE/CVF International Conference on Computer Vision},
  pages={1655--1665},
  year={2023}
}

@article{zhang2025widget2code,
  title={Widget2Code: From Visual Widgets to UI Code via Multimodal LLMs},
  author={Zhang, Houston H and Zhang, Tao and Lin, Baoze and Xue, Yuanqi and Zhu, Yincheng and Liu, Huan and Gu, Li and Ye, Linfeng and Wang, Ziqiang and Zuo, Xinxin and others},
  journal={arXiv preprint arXiv:2512.19918},
  year={2025}
}

@misc{ye2024bayesconditionaldistributionestimation,
      title={Bayes Conditional Distribution Estimation for Knowledge Distillation Based on Conditional Mutual Information}, 
      author={Linfeng Ye and Shayan Mohajer Hamidi and Renhao Tan and En-Hui Yang},
      year={2024},
      eprint={2401.08732},
      archivePrefix={arXiv},
      primaryClass={cs.LG},
      url={https://arxiv.org/abs/2401.08732}, 
}

@article{yang2309conditional,
  title={Conditional mutual information constrained deep learning for classification (2023)},
  author={Yang, En-Hui and Hamidi, Shayan Mohajer and Ye, Linfeng and Tan, Renhao and Yang, Beverly},
  journal={URL https://arxiv. org/abs/2309.09123},
  volume={5}
}

@inproceedings{yang2024markov,
  title={Markov knowledge distillation: Make nasty teachers trained by self-undermining knowledge distillation fully distillable},
  author={Yang, En-hui and Ye, Linfeng},
  booktitle={European Conference on Computer Vision},
  pages={154--171},
  year={2024},
  organization={Springer}
}

@misc{hamidi2024adversarialtrainingadaptiveknowledge,
      title={Adversarial Training via Adaptive Knowledge Amalgamation of an Ensemble of Teachers}, 
      author={Shayan Mohajer Hamidi and Linfeng Ye},
      year={2024},
      eprint={2405.13324},
      archivePrefix={arXiv},
      primaryClass={cs.LG},
      url={https://arxiv.org/abs/2405.13324}, 
}

@article{
hamidi2025distributed,
title={Distributed Quasi-Newton Method for Fair and Fast Federated Learning},
author={Shayan Mohajer Hamidi and Linfeng Ye},
journal={Transactions on Machine Learning Research},
issn={2835-8856},
year={2025},
url={https://openreview.net/forum?id=KbteA50cni},
note={}
}

@INPROCEEDINGS{10619241,
  author={Yang, En-Hui and Hamidi, Shayan Mohajer and Ye, Linfeng and Tan, Renhao and Yang, Beverly},
  booktitle={2024 IEEE International Symposium on Information Theory (ISIT)}, 
  title={Conditional Mutual Information Constrained Deep Learning: Framework and Preliminary Results}, 
  year={2024},
  volume={},
  number={},
  pages={569-574},
  doi={10.1109/ISIT57864.2024.10619241}}

@ARTICLE{10900607,
  author={Yang, En-Hui and Mohajer Hamidi, Shayan and Ye, Linfeng and Tan, Renhao and Yang, Beverly},
  journal={IEEE Transactions on Neural Networks and Learning Systems}, 
  title={Conditional Mutual Information Constrained Deep Learning for Classification}, 
  year={2025},
  volume={36},
  number={8},
  pages={15436-15448},
  doi={10.1109/TNNLS.2025.3540014}}
